\documentclass{article}

\usepackage[margin=1in]{geometry}

\usepackage{etoolbox}
\newtoggle{icml}
\togglefalse{icml}
\usepackage{xcolor}  
\usepackage{preamble/color-edits}

\usepackage[utf8]{inputenc}
\iftoggle{icml}{}
{
  \usepackage[margin=1in]{geometry}
\usepackage[T1]{fontenc} 
}

\usepackage{mathtools,verbatim}
\usepackage{times,upgreek}
\usepackage{pgfplots}
\usepackage{graphicx}
\usepackage{wrapfig}
\usepackage{amsthm}
\usepackage{thm-restate}

\usepackage{eucal}
\usepackage{mathrsfs}
\usepackage{bbm}
\pgfplotsset{width=10cm,compat=1.9}

\iftoggle{icml}{
  \usepackage[numbers]{natbib}
  \bibliographystyle{abbrvnat}
}{
\usepackage{natbib}
\bibliographystyle{plainnat}
}

\usepackage[hypertexnames=false]{hyperref}
\hypersetup{
    colorlinks=true,
    linkcolor=blue,
    citecolor=blue
}
\usepackage{cleveref}

\numberwithin{equation}{section}

\newtheorem{theorem}{Theorem}
\newtheorem{lemma}{Lemma}[section]
\newtheorem{claim}{Claim}[section]
\newtheorem{assumption}{Assumption}

\newtheorem{proposition}[theorem]{Proposition}
\newtheorem{remark}{Remark}

\newtheorem{definition}[theorem]{Definition}
\newtheorem{condition}{Condition}

\Crefname{assumption}{Assumption}{Assumptions}
\Crefname{definition}{Definition}{Definitions}
\Crefname{lemma}{Lemma}{Lemmas}
\Crefname{claim}{Claim}{Claims}

\usepackage[noend]{algpseudocode}
\usepackage{algorithm}
\usepackage{tablefootnote}
\usepackage{makecell}
\usepackage{nicefrac,longtable}
\usepackage{mathrsfs}
\usepackage{booktabs}
\usepackage{microtype}
\usepackage{amsmath}
\usepackage{amsfonts}
\usepackage{amssymb}
\usepackage{textcomp}
\usepackage{inconsolata}
\usepackage{enumerate}
\usepackage{bbm}
\newcommand{\algcomment}[1]{\textcolor{blue!70!black}{\footnotesize{\texttt{\textbf{\% #1}}}}}

\DeclareMathOperator*{\argmax}{arg\,max}

\newcommand{\nc}{\newcommand}
\nc{\ra}{\rightarrow}
\nc{\llin}{\ell^{{\rm lin}}}
\nc{\MN}{\mathcal{N}}

\nc{\erma}{\zeta}
\nc{\ep}{\epsilon}
\nc{\MX}{\mathcal{X}}
\nc{\BZ}{\mathbb{Z}}
\nc{\st}{\star}
\nc{\BR}{\mathbb{R}}
\nc{\MC}{\mathcal{C}}
\nc{\MZ}{\mathcal{Z}}
\nc{\close}[3]{\MC_{{#1},{#2}}({#3})}
\nc{\bv}{\mathbf{v}}
\nc{\bz}{\mathbf{z}}
\nc{\bw}{\mathbf{w}}
\nc{\bb}{\mathbf{b}}
\nc{\MA}{\mathcal{A}}
\nc{\ME}{\mathcal{E}}
\nc{\YY}{[-1,1]}
\nc{\CLS}{\textbf{Cls}}
\nc{\REG}{\textbf{Reg}}
\nc{\condt}{\ | \ }
\nc{\MI}{\mathcal{I}}

\DeclareMathOperator{\E}{\mathbb{E}}

\newcommand{\bx}{{\mathbf{x}}}
\nc{\x}{{\mathbf{x}}}
\newcommand{\by}{{\mathbf{y}}}
\newcommand{\G}{\mathcal{G}}
\newcommand{\rr}{\mathbb{R}}
\newcommand{\pp}{\mathbb{P}}
\newcommand{\ee}{\mathbb{E}}

\newcommand{\R}{\mathcal{R}}

\newcommand{\w}{{\mathbf{w}}}

\newcommand{\norm}[1]{\left|\left| #1 \right|\right|}
\newcommand{\inprod}[2]{\left\langle #1 , #2 \right\rangle}
\newcommand{\abs}[1]{\left| #1 \right|}

\newcommand{\reals}{\mathbb{R}}

\DeclareMathOperator*{\argmin}{argmin}
\newcommand{\yhat}{\hat{y}}
\newcommand{\fhat}{\widehat{f}}
\DeclareMathOperator{\reg}{Reg}

\renewcommand{\epsilon}{\varepsilon}

\DeclareMathOperator{\esssup}{ess\,sup}
\DeclareMathOperator{\sign}{sign}

\DeclareMathOperator{\vol}{vol}

\newcommand{\filt}{\scrF}
\newcommand{\N}{\bbN}

\newcommand{\I}{\bbI}

\newcommand{\BigOh}[1]{O\left(#1\right)}

\newcommand{\algfont}[1]{\mathbf{\mathsf{#1}}}

\newcommand{\selfdot}{\algfont{self}.}

\newcommand{\eye}{\bI}
\newcommand{\iidsim}{\overset{\mathrm{i.i.d.}}{\sim}}
\newcommand{\sigdir}{\sigma_{\mathrm{dir}}}

\renewcommand{\ast}{\star}
\newcommand{\paramstar}{\mathbf{\Theta}^\star}
\newcommand{\paramhat}{\widehat{\mathbf{\Theta}}}
\newcommand{\param}{\mathbf{\Theta}}

\newcommand{\sige}{\nu}
\newcommand{\fro}{\mathrm{F}}
\newcommand{\op}{\mathrm{op}}

\newcommand{\Qij}{\cQ_{ij}}
\newcommand{\fronorm}[1]{\|#1\|_{\fro}}

\renewcommand{\R}{\reals}

\newcommand{\dn}{\textsc{DN}}

\newcommand{\Ztij}{Z_{t;ij}}
\newcommand{\barZtij}{\bar Z_{t;ij}}

\newcommand{\bdelta}{\boldsymbol{\delta}}
\newcommand{\Var}{\mathbb{V}}

\newcommand{\Varut}{\Var_{\bu_t}}
\newcommand{\Varxt}{\Var_{\bx_t}}

\newcommand{\Rijst}{\cR_{ij}}

\newcommand{\epsorac}{\epsilon_{\mathrm{orac}}}

\newcommand{\bDelta}{\boldsymbol{\Delta}}

\newcommand{\byhat}{\mathbf{\widehat{y}}}
\newcommand{\gtil}{\widetilde{g}}
\newcommand{\gstar}{g_{\star}}
\newcommand{\ermoracle}{\textsc{ErmOracle}}
\newcommand{\ogdupdate}{\textsc{Ogd}}
\newcommand{\sigij}{\Sigma_{ij}}
\newcommand{\Igij}{I_{ij}(g)}
\newcommand{\bwk}{\bw_{1:K}}
\newcommand{\elltilgamma}[1][\ghat]{\widetilde{\ell}_{\gamma,t,#1}}
\newcommand{\reorder}{\algfont{Reorder}}

\newcommand{\ghat}{\widehat{g}}
\newcommand{\gst}{g^{\star}}
\newcommand{\fst}{f^{\star}}
\newcommand{\bzhat}{\widehat{\bz}}
\newcommand{\Bnoise}{B_{\mathrm{noise}}}
\newcommand{\bxhat}{\widehat{\bx}}
\newcommand{\paramtil}{\widetilde{\param}}

\renewcommand{\I}{\mathbb{I}}

\usepackage{bm}

\newcommand{\Exp}{\E}
\renewcommand{\Pr}{\mathbb{P}}

\newcommand{\sigtilij}{\widetilde{\Sigma}_{ij}}
\newcommand{\Itilgij}{\widetilde{I}_{ij}(\ghat)}
\DeclareMathOperator{\tr}{Tr}

\newcommand{\simreg}{\mathrm{SimReg}}

\def\ddefloop#1{\ifx\ddefloop#1\else\ddef{#1}\expandafter\ddefloop\fi}
\def\ddef#1{\expandafter\def\csname bb#1\endcsname{\ensuremath{\mathbb{#1}}}}
\ddefloop ABCDEFGHIJKLMNOPQRSTUVWXYZ\ddefloop

\def\ddefloop#1{\ifx\ddefloop#1\else\ddef{#1}\expandafter\ddefloop\fi}
\def\ddef#1{\expandafter\def\csname fr#1\endcsname{\ensuremath{\mathfrak{#1}}}}
\ddefloop ABCDEFGHIJKLMNOPQRSTUVWXYZ\ddefloop

\def\ddefloop#1{\ifx\ddefloop#1\else\ddef{#1}\expandafter\ddefloop\fi}
\def\ddef#1{\expandafter\def\csname eul#1\endcsname{\ensuremath{\EuScript{#1}}}}
\ddefloop ABCDEFGHIJKLMNOPQRSTUVWXYZ\ddefloop

\def\ddefloop#1{\ifx\ddefloop#1\else\ddef{#1}\expandafter\ddefloop\fi}
\def\ddef#1{\expandafter\def\csname scr#1\endcsname{\ensuremath{\mathscr{#1}}}}
\ddefloop ABCDEFGHIJKLMNOPQRSTUVWXYZ\ddefloop

\def\ddefloop#1{\ifx\ddefloop#1\else\ddef{#1}\expandafter\ddefloop\fi}
\def\ddef#1{\expandafter\def\csname b#1\endcsname{\ensuremath{\mathbf{#1}}}}
\ddefloop ABCDEFGHIJKLMNOPQRSTUVWXYZ\ddefloop

\def\ddefloop#1{\ifx\ddefloop#1\else\ddef{#1}\expandafter\ddefloop\fi}
\def\ddef#1{\expandafter\def\csname bhat#1\endcsname{\ensuremath{\hat{\mathbf{#1}}}}}
\ddefloop ABCDEFGHIJKLMNOPQRSTUVWXYZ\ddefloop

\def\ddefloop#1{\ifx\ddefloop#1\else\ddef{#1}\expandafter\ddefloop\fi}
\def\ddef#1{\expandafter\def\csname btil#1\endcsname{\ensuremath{\tilde{\mathbf{#1}}}}}
\ddefloop ABCDEFGHIJKLMNOPQRSTUVWXYZ\ddefloop

\def\ddefloop#1{\ifx\ddefloop#1\else\ddef{#1}\expandafter\ddefloop\fi}
\def\ddef#1{\expandafter\def\csname bst#1\endcsname{\ensuremath{\mathbf{#1}^\star}}}
\ddefloop ABCDEFGHIJKLMNOPQRSTUVWXYZ\ddefloop

\def\ddefloop#1{\ifx\ddefloop#1\else\ddef{#1}\expandafter\ddefloop\fi}
\def\ddef#1{\expandafter\def\csname bst#1\endcsname{\ensuremath{\mathbf{#1}^\star}}}
\ddefloop abcdefghijklmnopqrstuvwxyz\ddefloop

\def\ddefloop#1{\ifx\ddefloop#1\else\ddef{#1}\expandafter\ddefloop\fi}
\def\ddef#1{\expandafter\def\csname b#1\endcsname{\ensuremath{\mathbf{#1}}}}
\ddefloop abcdeghijklnopqrstuvwxyz\ddefloop

\def\ddefloop#1{\ifx\ddefloop#1\else\ddef{#1}\expandafter\ddefloop\fi}
\def\ddef#1{\expandafter\def\csname barb#1\endcsname{\ensuremath{\bar{\mathbf{#1}}}}}
\ddefloop abcdefghijklmnopqrstuvwxyz\ddefloop

\def\ddef#1{\expandafter\def\csname c#1\endcsname{\ensuremath{\mathcal{#1}}}}
\ddefloop ABCDEFGHIJKLMNOPQRSTUVWXYZ\ddefloop
\def\ddef#1{\expandafter\def\csname h#1\endcsname{\ensuremath{\widehat{#1}}}}
\ddefloop ABCDEFGHIJKLMNOPQRSTUVWXYZ\ddefloop
\def\ddef#1{\expandafter\def\csname hc#1\endcsname{\ensuremath{\widehat{\mathcal{#1}}}}}
\ddefloop ABCDEFGHIJKLMNOPQRSTUVWXYZ\ddefloop
\def\ddef#1{\expandafter\def\csname t#1\endcsname{\ensuremath{\widetilde{#1}}}}
\ddefloop ABCDEFGHIJKLMNOPQRSTUVWXYZ\ddefloop
\def\ddef#1{\expandafter\def\csname tc#1\endcsname{\ensuremath{\widetilde{\mathcal{#1}}}}}
\ddefloop ABCDEFGHIJKLMNOPQRSTUVWXYZ\ddefloop
\addauthor{ms}{orange}
\addauthor{refs}{red}

\title{Smoothed Online Learning for Prediction in Piecewise Affine Systems  }

\author{Adam Block \footnote{\href{mailto:ablock@mit.edu}{ablock@mit.edu}} \\ MIT \and Max Simchowitz\footnote{\href{mailto:msimchow@csail.mit.edu}{msimchow@csail.mit.edu}} \\ MIT \and Russ Tedrake\footnote{\href{mailto:russt@mit.edu}{russt@mit.edu}} \\ MIT}

\date{}

\begin{document}
\maketitle
\begin{abstract}

The problem of piecewise affine (PWA) regression and planning is of foundational importance to the study of online learning, control, and robotics, where it provides a theoretically and empirically tractable setting to study systems undergoing sharp changes in the dynamics.  Unfortunately, due to the discontinuities that arise when crossing into different ``pieces,'' learning in general sequential settings is impossible and practical algorithms are forced to resort to heuristic approaches.  This paper builds on the recently developed \emph{smoothed online learning} framework and provides the first algorithms for prediction and simulation in PWA systems whose regret is polynomial in all relevant problem parameters under a weak smoothness assumption; moreover, our algorithms are efficient in the number of calls to an optimization oracle.  We further apply our results to the problems of one-step prediction and multi-step simulation regret in piecewise affine dynamical systems, where the learner is tasked with simulating trajectories and regret is measured in terms of the Wasserstein distance between simulated and true data.  Along the way, we develop several technical tools of more general interest.

\end{abstract}
\tableofcontents
\newpage

\newcommand{\sysid}{\textsc{SysId}}
\section{Introduction}

A central problem in the fields of online learning, control, and robotics is how to efficiently plan through piecewise-affine (PWA) systems.  Such systems are described by a finite set of disjoint regions in state-space, within which the dynamics are affine. In this paper, we consider the related problem of learning to make predictions in PWA systems when the dynamics are unknown. While recent years have seen considerable progress around our understanding of linear control systems, the vast majority of dynamics encountered in practical settings involve nonlinearities. Learning and predicting in nonlinear systems is typically far more challenging because, unlike their linear counterparts, the dynamics in a local region of a nonlinear system do not determine its global dynamics.

\iftoggle{icml}{}{
\begin{sloppypar}}
PWA systems allow for discontinuities across the separate regions (``pieces''), and are thus a simplified way of modeling rich phenomena that arise in numerous robotic manipulation and locomotion tasks \citep{mason2001mechanics,van2007introduction,marcucci2019mixed}, such as modeling dynamics involving contact.  In addition, deep neural networks with ReLU activation \citep{goodfellow2016deep} are PWA systems, providing further motiviation for their study.  Already, there is a computational challenge simply in optimizing an objective over these systems, which is the subject of much of the previous literature. Here, we take a statistical perspective, assuming that we have access to effective heuristic algorithms for this optimization task, as is common in online learning \citep{hazan2016computational,block2022smoothed,haghtalab2022oracle}.  Uniformly accurate learning of the dynamics across all pieces is typically impossible because some regions of space require onerous exploration to locate; thus, we instead consider an iterative prediction setting, where at each time $t$, the learner predicts the subsequent system state. From this, we extend to prediction of entire trajectories: over episodes $t = 1,\dots,T$, a learner suggests a policy $\pi_t$, and the learner attempts to learn the dynamics so as to minimize prediction error along the trajectory $\pi_\tau$ induces. This is motivated by iterative learning control, where our notion of error is an upper bound on the discrepancy between the planner's estimate of policy performance and the ground truth.
\iftoggle{icml}{}{
\end{sloppypar}}

Our iterative formulation of the learning problem is equivalent to the regret metric favored by the online learning community, and which has seen application to online/adaptive control and filtering problems in recent work.  Critically, regret does not require uniform identification of system parameters, which is typically impossible. The key pathology is that policies can oscillate across the boundary of two regions, accruing significant prediction error due to sharp discontinuities of the dynamics between said regions, a problem well-known to the online learning community \citep{littlestone1988learning}. Moreover, this pathology is not merely a theoretical artifact: discontinuities and non-smoothness pose significant challenges to planning and modeling contact-rich dynamics such as those encountered in robotic manipulation \citep{suh2022differentiable,suh2022bundled,posa2014direct}.

Our solution takes inspiration from, and establishes a connection between, two rather different fields. Recent work in the robotics and control communities has relied on \emph{randomized smoothing} to improve planning across discontinuous and non-smooth dynamics \citep{suh2022differentiable,suh2022bundled,lidec2022leveraging,pang2022global}. Additionally, the online learning community has studied \emph{smoothed online learning} \citep{rakhlin2011online,haghtalab2021smoothed,haghtalab2022oracle,block2022smoothed,block2022efficient,block2023sample}, which circumvents the threshold-effect pathologies described above.

 We show that, if the dynamics and the control inputs are subject to randomized smoothing, low regret becomes achievable.  We note that the randomized smoothing approach is in some sense canonical in mitigating the aforementioned pathology of policies that oscillate across the boundaries of two regions; smoothing prohibits this pathology by ensuring that the system is generally far from these boundaries.  More importantly, our proposed no-regret algorithm is \emph{efficient} in terms of the number of calls to an optimization oracle, a popular notion of computational efficiency in the online learning community \citep{hazan2016computational,kalai2005efficient,block2022smoothed,haghtalab2022oracle}.  In our setting, the optimization oracle required finds the \emph{best-fit} PWA system to a batch of given data.  Though this problem is intractable \citep{amaldi2002min}, there is a rich literature of popular heuristics \citep{garulli2012survey,paoletti2007identification,ferrari2003clustering}.  Unlike those works, we examine the \emph{statistical challenge} of generalization to novel data given such an oracle.  We remark that in many practical cases, our smoothness assumption comes for free, in the sense that the Gaussian noise used to smooth gradients in \citet{suh2022differentiable} is already sufficient for our results to hold.

We stress that we consider \emph{online learning} of PWA systems, so as to accomodate for distribution shift that arises as a controlling agent interacts with a dynamical system and updates its own behavior. In the \emph{PAC learning} setting, it is possible to derive on-distribution generalization bounds by bounding the pseudo-dimension of piecewise-affine functions \citep{balcan2021much}. Moreover, our class of functions do not satisfy the ``dispersion condition'' \citep{balcan2018dispersion}, which instead would render the class online-learnable without smoothing; further still, it remains an open problem as to how to develop \emph{oracle-efficient} online learning algorithms based on dispersion.  

\paragraph{Contributions.} A formal description of our setting is deferred to \Cref{sec:setting}. Informally, we consider the problem of prediction in a PWA system over a horizon of $T$ steps, where the regions are determined by intersections of halfspaces; we obtain prediction in PWA dynamical systems as a special case.  We aim to achieve sublinear-in-$T$ excess square-loss error (\emph{regret}) of $\BigOh{T^{1-\alpha}}$ with respect to the optimal predictor which knows the system dynamics, where $\alpha > 0$ is constant and the prefactor on the regret is at most polynomial in all of the problem parameters. Our result is derived from a general guarantee for online piecewise affine regression, which subsumes the online PWA identification setting.  We show that, when the dynamics and the control inputs are subject to randomized smoothing satisfying the $\sigdir$-directional smoothness condition introduced by \cite{block2022efficient}, our regret bound is polynomial in $\sigdir$, dimension, and other natural problem parameters.  While the exact dependence on horizon may be far from optimal, we emphasize that our algorithm is the first to achieve regret even polynomial in all parameters in the difficult setting of PWA systems.  Moreover, our work provides the first regret bound with polynomial dependence in dimension for oracle-efficient smoothed online learning in the sense of \citet{hazan2016computational} of a noncontinuous function class without a realizability assumption and allowing for process noise.
\iftoggle{icml}{}{

}
As a further application, we adapt our algorithm to make predictions, at each time $t$, of the trajectory comprising the next $H$ steps in the evolution of the dynamical system.  We bound the \emph{simulation regret}, the cumulative error in Wasserstein distance between the distribution of our proposed trajectory and that of the actual trajectory.  Assuming a Lyapunov condition, we demonstrate that our modified algorithm achieves simulation regret $\BigOh{\mathsf{poly}(H) \cdot T^{1- \alpha}}$, which allows for efficient simulation.

\paragraph{Key challenges.} As noted above, vanishing regret in our setting is impossible without directional smoothness due to discontinuity in the regressors between different regions. The key challenge is to leverage smoothness to mitigate the effect of these discontinuties.  Note that, if the learner were to observe the region in addition to the state, regression would be easy by decomposing the problem into $K$ separate affine regression instances; the difficulty is in minimicking this approach \emph{without explicit labels indicating which points are in which regions.} We remark that were we to care only about low regret without regard to efficiency, a simple discretization scheme of size exponential in dimension coupled with a standard learning from experts algorithm \citep{cesa2006prediction} can achieve low regret; realizing an oracle-efficient algorithm is significantly more challenging.

We adopt a natural, epoch-based approach by calling the optimization oracle between epochs, estimating the underlying system, and using this estimate to assign points to regions in the subsequent epoch. This introduces three new challenges, each of which requires substantial technical effort to overcome.  First, we need to control the performance of our oracle on the regression task, which we term \emph{parameter recovery} below.  Second, we must enforce consistency between regions learned in different epochs, so that we can predict the correct regions for subsequent covariates; doing this necessitates an intricate analysis of how the estimated parameters evolve between epochs.  Finally, we need to modify the output of our oracle to enforce low regret with respect to the classification problem of correctly predicting which mode we are in on as yet unseen data.

\paragraph{Our techniques.} Our algorithm proceeds in epochs, at the end of each of which we compute a best-fitting piecewise affine approximation to available data by calling the Empirical Risk Minimization (ERM) oracle. The analysis of this best-fit shows that we recover the parameters and decision boundaries of the associated regions (``pieces'') frequently visited during that epoch. By the pigeonhole principle, this ultimately covers all regions visited a non-neglible number of times $t$ over the horizon $T$. Here, we leverage a careful analysis based on two modern martingale-least squares techniques: the self-normalized concentration bound \citep{abbasi2011improved} for establishing recovery in a norm induced by within-region empirical covariance matrices, and the block-martingale small-ball method \citep{simchowitz2018learning} to lower-bound the smallest eigenvalues of  empirical covariances with high probability.  In addition, we introduce and bound a new complexity parameter, the \emph{disagreement cover}, in order to make our statements uniform over the set of possible decision boundaries.

After the estimation phase, each epoch uses this estimate to make predictions on new data as it arrives. The key challenge is classifying the region to which each newly observed data point belongs. Without smoothing, this can be particularly treacherous, as points on or near to the boundary of regions can easily be misclassified and the correct predictions are discontinous. To leverage smoothing, we propose a reduction to online convex optimization on the hinge loss with lazy updates and show, through a careful decomposition of the loss, that our classification error can be controlled in terms of the assumed directional smoothness.  Combining this bound on the classification error with the parameter recovery described in the preceding paragraph yields our sublinear regret guarantee.

We instantiate one-step prediction as a special case of the more general online linear regression problem, showing that the approach described above immediately applies to the first problem under consideration.  Finally, in order to simulate trajectories, we use the learned parameters in the epoch to predict the next $H$ steps of the evolution.  We then use the fact that our main algorithm recovers the parameters defining the PWA system to show that simulation regret is bounded, which solves the second main problem and allows for simulation in such systems.

\newcommand{\bmean}{\mathbf{m}}
\newcommand{\filty}{\filt^y}
\newcommand{\bdel}{\bm{\delta}}
\newcommand{\Delsep}{\Delta_{\mathrm{sep}}}
\newcommand{\pwa}{\textsc{PWA}}
\newcommand{\epscor}{\epsilon_{\mathrm{crp}}}
\section{Setting}\label{sec:setting}

In this section, we formally describe the problem setting of online PWA regression.  We suppose there is a ground truth classifier $g_\ast : \rr^d \to [K]$ and matrices $\left\{ \paramstar_{i} \mid i \in [K] \text{ and } \paramstar_i \in \rr^{m \times (d+1)} \right\}$ such that 
\iftoggle{icml}{
    \begin{align}\label{eq:setup}
        \by_t = \paramstar_{i_t} \cdot \barbx_t + \be_t + \bdelta_t, \quad i_t = \gstar(\barbx_t), \quad \barbx_t = [\bx_t^\top | 1]^\top, \quad \norm{\bdelta_t} \leq \epscor, \quad 
    \end{align}
    where, $\bx_t \in \rr^d$ are covariates, $\by_t \in \rr^m$ are responses,  $\be_t \in \rr^m$ are zero-mean noise vectors in $\rr^m$, $\bdelta_t$ are (small) non-stochastic corruptions with norm at most $\epscor \ll 1$, and $i_t \in [K]$ are the regression modes, which depend on the covariates $\bx_t$.
}
{
\begin{align}\label{eq:setup}
\by_t =  \paramstar_{i_t} \cdot\barbx_t + \be_t + \bdelta_t, \quad i_t = \gst(\barbx_t), \quad  \barbx_t = [\bx_t^\top | 1]^\top, \quad \norm{\bdelta_t} \leq \epscor,
\end{align}
where $\bx_t \in \rr^d$ are covariates, $\by_t \in \rr^m$ are responses,  $\be_t \in \rr^m$ are zero-mean noise vectors in $\rr^m$, $\bdelta_t$ are (small) non-stochastic corruptions with norm at most $\epscor \ll 1$, and $i_t \in [K]$ are the regression modes, which depend on the covariates $\bx_t$.} The extension to $\barbx_t$ accomodates affine regressors. We suppose that the learner has access to pairs $(\bx_t, \by_t)$ but, critically, does not know the regression modes $i_t$.  We do however suppose that $g_\star$ is an \emph{affine classifer}; that is \iftoggle{icml}{ $\gst \in \cG = \left\{ \bx \mapsto \argmax_{1 \leq i \leq K} \inprod{\bw_i}{\bx} + b_i \right\}$, where $\bw_i \in S^{d-1}$ is the unit sphere and $b_i \in [-B,B]$.}{
\begin{equation}
    \gst \in \cG = \left\{ \bx \mapsto \argmax_{1 \leq i \leq K} \inprod{\bw_i}{\bx} + b_i | \bw_i \in S^{d-1}, b_i \in [-C,C] \right\},
\end{equation}
where $S^{d-1}$ is the unit sphere.  
}
It can be shown that many natural physical systems are indeed modeled as PWA with affine boundaries \citep{marcucci2019mixed}, and the closed loop dynamics for model-predictive control (MPC) in these cases is also a PWA system \citep{borrelli2017predictive}. We will assume throughout access to an empirical risk minimization oracle, $\ermoracle$, which satisfies the following guarantee:
\begin{assumption}[$\ermoracle$ guarantee]\label{ass:oracle}
     Given data $(\barbx_{1:s}, \by_{1:s})$, $\ermoracle$ returns $\left\{ \paramhat_i |i \in [K] \right\}, \ghat$ satisfying
     \iftoggle{icml}
     {
        \begin{equation}
            \sum_{t = 1}^s \norm{\paramhat_{\ghat(\bx_t)} \barbx_t - \by_t}^2 - \inf_{g \in \cG, \, \left\{ \mathbf{\Theta}_i | i \in [K] \right\}} \sum_{t = 1}^s \norm{\mathbf{\Theta}_{g(\bx_t)} \barbx_t - \by_t}^2 \leq \epsorac.
        \end{equation}
     }
     {
\begin{equation}
    \sum_{t = 1}^s \norm{\paramhat_{\ghat(\bx_t)} \barbx_t - \by_t}^2 - \inf_{g \in \cG, \, \left\{ \mathbf{\Theta}_i | i \in [K] \right\}} \sum_{t = 1}^s \norm{\mathbf{\Theta}_{g(\bx_t)} \barbx_t - \by_t}^2 \leq \epsorac.
\end{equation}
}
\end{assumption} 
While such an oracle is not computationally efficient in the worst case \citep{amaldi2002min}, there are many popular heuristics used in practical situations \citep{ferrari2003clustering,garulli2012survey}.  The learner aims to construct an algorithm, efficient with respect to the above oracle, that at each time step $t$ produces $\byhat_t$ such that
\iftoggle{icml}
{
\begin{equation}
    \reg_T := \textstyle\sum_{t = 1}^T \|\byhat_t - \paramstar_{g_\star(\barbx_t)} \barbx_t\|^2 = o(T).  \label{eq:sublin_regret}
\end{equation}
}
{
    \begin{equation}
    \reg_T := \sum_{t = 1}^T \left\|\byhat_t - \paramstar_{g_\star(\barbx_t)} \barbx_t\right\|^2 = o(T).  \label{eq:sublin_regret}
\end{equation}
}
This notion of regret is standard in the online learning community and, as we shall see, immediately leads to bounds on prediction error.

\paragraph{Directional Smoothness.} As previously mentioned, without further restriction, sublinear regret is infeasible, as formalized in the following proposition:
\begin{proposition}\label{prop:need_smoothness}
    In the above setting, there exists an adversary with $m = d = 1$, K=2, that chooses $\paramstar$ and $\gstar$, as well as $\bx_1, \dots, \bx_T$ such that any learner experiences \iftoggle{icml}{$\ee\left[\reg_T\right] \geq \frac T2$.}{
    \begin{align*}
        \ee\left[\reg_T\right] \geq \frac T2.
    \end{align*}
    }
\end{proposition}
\Cref{prop:need_smoothness} follows from a construction of a system where there is a discontinuity across a linear boundary, across which the states $\bx_t$ oscillate. The bound is then a consequence of the classical fact that online classification with linear thresholds suffers linear regret; see \Cref{sec:prop:need_smoothness} for a formal proof, included for the sake of completeness. Crucially, a discontinuity in the dynamics necessitates an $\Omega(1)$ contribution to regret each time the decision boundary is incorrectly determined.  To avoid this pathology, we suppose that the contexts are smooth; however, because standard smoothness \citep{haghtalab2020smoothed} leads to regret that is exponential in dimension \citep{block2022efficient}, we instead consider a related condition, ``directional smoothness.''\footnote{For a comparison between directional and standard smoothness, consult \cite{block2022efficient}.}
\begin{definition}[Definition 52 from \citet{block2022efficient}] Given a scalar $\sigdir>0$, a random vector $\bx$ is called $\sigdir$-directionally smooth if for all $\bw \in S^{d-1}:=\{w \in \rr^d: \|w\| = 1\}$, $c \in \rr$, and $\delta > 0$, it holds that \iftoggle{icml}{$\pp\left(\abs{\inprod{\bw}{\bx} - c} \leq \delta\right) \leq \delta /\sigdir$.}{
    \begin{equation}
        \pp\left(\abs{\inprod{\bw}{\bx} - c} \leq \delta\right) \leq \frac{\delta}{\sigdir}.
    \end{equation}
}
\end{definition}
As an example, for any vector $\bz_t$, if $\bw_t \sim \cN(0,\sigdir^2 \eye)$, then $\bz_t + \bw_t$ is $\frac{\sigdir}{\sqrt{2\pi}}$-directionally smooth; see \Cref{sec:dist_smoothnesses} for a proof and examples of directional smoothness for other noise distributions. Crucially, directional smoothness is \emph{dimension-independent}, in contradistinction to standard smoothness, where we would only have $\Theta\left( \sigdir^{d} \right)$-smoothness (in the standard sense) in the previous example.  We remark that directional smoothness is a weak condition that holds in many natural settings.  Indeed, whenever noise is injected into the dynamics, directional smoothness will come for free.  Such noise injection can occur when there is uncertainty in position, as is common in robotic applications; in this case, we can interpret this uncertainty itself as stochastic, which was the original motivation for smoothed analysis of algorithms \citep{spielman2004nearly}.  Further, note that while directional smoothness ensures the system spends little time close to a boundary, due to the discreteness in time, there is no restriction on the number of mode switches, which can be $\Theta(T)$ in general.

We defer further discussion of the assumption to \Cref{sec:dist_smoothnesses}.  We require our smoothness condition to hold conditionally: 
\begin{assumption}\label{ass:dirsmooth}
     Let $\filt_{t}$ denote the filtration generated by $\bx_1,\dots,\bx_t$. For all times $t$, $\bx_t$ conditional on $\filt_{t-1}$ is $\sigdir$-directionally smooth.
\end{assumption}

\paragraph{Further Assumptions.} 
Next, we make two standard assumption for sequential regression, one controlling the tail of the noise and the other the magnitude of the parameters.
\begin{assumption}\label{ass:subgaussian}
    Let $\filty_t$ denote the filtration generated by $\filt_t$ and $\be_{1},\dots,\be_{t-1}$ (equivalently, by $\filt_t$ and $\by_t,\dots,\by_{t-1}$).  For all $t$, it holds that $\be_t$ is $\nu^2$-sub-Gaussian conditioned on $\filty_t$; in particular, $\be_t$ is conditionally zero mean.
\end{assumption}
\begin{assumption}\label{ass:boundedness}
    We suppose that for all $1 \leq t \leq T$, $\norm{\barbx_t} \leq B$ and, furthermore, $\fronorm{\paramstar_i} \leq R$ for all $i$.  We will further assume that $\ermoracle$ always returns $\paramhat_i$ such that $\fronorm{\paramhat_i} \leq R$. 
\end{assumption}
Finally, we assume that the true affine parameters are well-separated:
\begin{assumption}\label{ass:gap}
    There is some $\Delsep > 0$ such that for all $1 \leq i < j \leq K$, it holds that \iftoggle{icml}{$\fronorm{\paramstar_i - \paramstar_j} \geq \Delsep$.}{
    \begin{equation}
        \fronorm{\paramstar_i - \paramstar_j} \geq \Delsep.
    \end{equation}
    }
\end{assumption}
\iftoggle{icml}{
    Note that \Cref{ass:gap} is, in a sense, generic, as explained in \Cref{rmk:gap_necessary}
}{We remark that \Cref{ass:gap} is, in some sense, generic.  Indeed, in the original smoothed analysis of algorithms \citep{spielman2004nearly}, it was assumed that the parameter matrices were smoothed by Gaussian noise; if, in addition to smoothness in contexts $\bx_t$ we assume that the $\paramstar_i$ are drawn from a directionally smooth distribution, then \Cref{prop:smoothparamslargesep} below implies that with probability at least $1 - T^{-1}$, it holds that $\Delsep \gtrsim md \left( \frac{\sigdir}{K^2 T} \right)^{\frac 1{md}}$.  Furthermore, one reason why removing the gap assumption is difficult in our framework is that, computationally speaking, agnostically learning halfspaces is hard \citep{guruswami2009hardness}.  Without \Cref{ass:gap}, $\ermoracle$ cannot reliably separate modes and thus the postprocessing steps our main algorithm (\Cref{alg:master}) used to stabilize the predictions must also agnostically learn the modes; together with the previous observation on the difficulty of learning halfspaces, this suggests that if $\ermoracle$ is unable to separate modes, there is significant technical difficulty in achieving a oracle-efficient, no-regret algorithm.}  We defer the formal setting of one-step prediction and simulation regret to \Cref{app:lem:pwa_sys_one_step,sec:simregret}.


\section{Algorithm and Guarantees}
\newcommand{\ghatorder}{\ghat^{\mathrm{ord}}}

\begin{algorithm}[!t]
    \begin{algorithmic}[1]
    \State{}\textbf{Initialize } Epoch length $E$, Classifiers $\bw_{1:K}^{0} = (\mathbf{0}, \dots, \mathbf 0)$, margin parameter $\gamma > 0$, learning rate $\eta > 0$
    \For{$\tau =1,2,\dots, T/E$} \iftoggle{icml}{\,}{\qquad \qquad}(\algcomment{iterate over epochs})
    \State{}\iftoggle{icml}{$\small(\paramhat_{\tau,i})_{1 \leq i \leq K}, \, \ghat_\tau  \gets \small\ermoracle(\barbx_{1:\tau E},\by_{1: \tau E} )$}{$\left\{(\paramhat_{\tau,i})_{1 \leq i \leq K}, \, \ghat_\tau \right\} \gets \ermoracle((\barbx_{1:\tau E},\by_{1: \tau E} ))$}
    \State{}
    \iftoggle{icml}{$\ghat_\tau \gets \reorder\left(\ghat_{\tau}, (\paramhat_{\tau, i})_{1\leq i \leq K}, (\paramhat_{\tau - 1,i})_{1 \leq i \leq K} \right)$}{$\ghat_\tau \gets \reorder\left(\ghat_{\tau}, \left(\paramhat_{\tau, i}\right)_{1\leq i \leq K}, \left(\paramhat_{\tau - 1,i}\right)_{1 \leq i \leq K} \right)$} 
    \State{}$\bw_{1:K}^{\tau E} \gets \ogdupdate(\bw_{1:K}^{(\tau-1) E},\barbx_{(\tau-1)E : \tau E}; \ghat_\tau, \gamma, \eta) $ 
    \State{}Let $\tilde{g}_{\tau}(\barbx) = \argmax_{1 \leq i \leq K} \inprod{\bw_{i}^{\tau E}}{\barbx} $ denote classifier induced by $\bw_{1:K}^{\tau E}$.
    \For{$t = \tau E ,\dots, (\tau + 1) E - 1$}
    \State{} \textbf{Receive } $\barbx_t$, and \textbf{predict} $\mathbf{\yhat}_t = \paramhat_{\tau, \tilde{g}_\tau(\barbx_t)} \cdot \barbx_t$
    \State{} \textbf{Receive } $\by_t$
    \EndFor
    \EndFor
    \end{algorithmic}
      \caption{Main Algorithm}
      \label{alg:master}
\end{algorithm}

We propose \Cref{alg:master}, an oracle-efficient protocol with provably sublinear regret and polynomial dependence on all problem parameters.  \Cref{alg:master} runs in epochs: at the beginning of epoch $\tau$, the learner calls $\ermoracle$ on the past data to produce a linear classifier $\ghat_\tau$ and estimates of the regression matrices in each mode, $\paramhat_{\tau, i}$.  A major challenge is identifying the same affine regions between epochs. To this end, the learner permutes the labels of $\ghat_\tau$ to preserve consistency in labels across epochs, as explained informally in \Cref{sec:onlinelearningsketch} and in greater detail in \Cref{sec:onlinelearning}.  

The learner then runs online gradient descent (\Cref{alg:ogdupdate}) on the hinge loss to produce a modified classifier $\gtil_\tau$; finally, throughout the epoch, the learner uses $\paramhat_{\tau, i}$ and $\gtil_\tau$ to predict $\byhat_t$ before repeating the process in the next epoch.  The reason for using a secondary algorithm to transform $\ghat_\tau$ into $\gtil_\tau$ is to enforce stability of the predictor across epochs, which is necessary to bound the regret.  Again, we suppose that \Cref{ass:oracle} holds throughout.

\iftoggle{icml}{\newcommand{\maxparam}{\overline{\mathrm{par}}}}{We have the following result:}
\begin{theorem}[Regret Bound]\label{thm:informal}
    Suppose we are in the setting of online PWA regression \eqref{eq:setup} and Assumptions \ref{ass:dirsmooth}-\ref{ass:gap} hold. Then, if the parameters $E, \gamma, \eta$ are set as in \eqref{eq:parameters}, it holds that with probability at least $1-\delta$,
    \iftoggle{icml}
    {
    $\reg_T \leq \mathsf{poly}\left(\maxparam,\log(1/\delta)\right) \cdot T^{35/36} + \epsorac \cdot T + K^2 \cdot \sum_{t = 1}^T \norm{\bdelta_t}$, for $\maxparam := \max\{d,m, \sigdir^{-1},\Delsep^{-1}, B, R, K, \nu\}$,
    }
    {
    \begin{equation}
        \reg_T \leq \mathsf{poly}\left(d, \frac 1\sigdir, m, B, R, K, \nu, \frac{1}{\Delsep},\log(1/\delta)\right) \cdot T^{1 - \Omega(1)} + \epsorac \cdot T + K^2 \cdot \sum_{t = 1}^T \norm{\bdelta_t},
    \end{equation}
    }
    where the polynomial dependence is given in \eqref{eq:polydependence}.
\end{theorem}
\begin{remark}
    While \Cref{thm:informal} requires that we are correctly setting the parameters of the algorithm, aggregation algorithms allow us to tune these parameters in an online fashion (see, e.g. \cite{hazan2009efficient}).
\end{remark}
\begin{remark}
    The misspecification errors $\bdelta_t$ can be indicators corresponding to the $\barbx_t$ being in small, rarely visited regions; thus, we could run an underspecified $\ermoracle$, with $K' \ll K$ modes and recover a regret bound depending only on frequently visited modes.
\end{remark}
When $\epscor = \epsorac = 0$, we obtain exactly the polynomial, no-regret rates promised in the introduction.  Along the way to proving our regret bound, we establish that $\ermoracle$ correctly recovers the linear parameters within frequently visited modes.
\begin{theorem}[Parameter Recovery]\label{thm:parameterrecoveryinformal}
    Suppose that Assumptions \ref{ass:dirsmooth}-\ref{ass:boundedness} hold and let $\left\{ \paramhat_i | i \in [K] \right\}, \ghat$ denote the output of $\ermoracle(\barbx_{1:T}, \by_{1:T})$ and $I_{ij}(\ghat)$ denote the set of times $t$ such that $\ghat(\barbx_t) = i$ and $\gstar(\barbx_t) = j$.  Then with probability at least $1 - \delta$, for all $1 \leq i,j \leq K$ such that
    \iftoggle{icml}
    {$
        \abs{I_{ij}(\ghat)} \geq \mathrm{poly}\left(\maxparam, \log\left( \frac 1\delta \right)\right) \cdot T^{1 - \Omega(1)}$}
    {
    \begin{align*}
        \abs{I_{ij}(\ghat)} \geq \mathrm{poly}\left(K, B, \frac{1}{\sigdir}, \log\left( \frac 1\delta \right)\right) \cdot T^{1 - \Omega(1)},
    \end{align*}
    }
    it holds that \iftoggle{icml}{$\fronorm{\paramhat_i - \paramstar_j}^2$ is bounded by  $
        \frac{\sqrt{T}}{\abs{I_{ij}(\ghat)}} \cdot \mathrm{poly}\left( \maxparam, \log\left( \frac 1\delta \right) \right).$
    }{
    \begin{align*}
        \fronorm{\paramhat_i - \paramstar_j}^2 \leq \frac{\sqrt{T}}{\abs{I_{ij}(\ghat)}} \cdot \mathrm{poly}\left( K, B, \frac 1{\sigdir}, R, d, m ,\nu, \log\left( \frac 1\delta \right) \right).
    \end{align*}
    }
\end{theorem}
Above we have given the result with the assumption that $\epscor = \epsorac = 0$ for the sake of presentation; in the formal statement, the general case is stated. 
\Cref{thm:parameterrecoveryinformal} implies that directional smoothness is sufficient to guarantee parameter recovery at a parametric rate, in marked distinction to most adversarial learning regimes.  \iftoggle{icml}{}{Note that this recovery result also implies mode recovery up to a permutation, due to separation.}  Before providing more detail on the algorithm and its analysis in \Cref{sec:body_analysis}, we discuss \iftoggle{icml}{an application to $H$-step prediction.  An application to one-step prediction with learner-provided controls can be found in \Cref{app:lem:pwa_sys_one_step}}{our two main applications.}

\iftoggle{icml}{}{
\subsection{Regret for One-Step Prediction in PWA Systems}\label{app:lem:pwa_sys_one_step}
A direct application of our main result is online prediction in piecewise affine systems. Consider the following dynamical system with state $\bz_t \in \R^{d_z}$ and input $\bu_t \in \R^{d_u}$:
\begin{equation}
    \bz_{t+1} = \bA^\star_{i_t}  \bz_t + \bB^\star_{i_t} \bu_t + \bmean^\star_{i_t} + \be_t, \quad i_t = \gst(\bz_t,\bu_t) \label{eq:pwa_system1}.
\end{equation}
Substitute $\paramstar_i := [\bA_i \mid \bB_i \mid \bmean_i]$ and define the concatenations $\bx_t = [\bz_t \mid \bu_t]$ and  $\barbx_t = [\bx_t \mid 1]$. The following lemma, proven in \Cref{app:lem:pwa_sys_smoothness}, gives sufficient conditions on the system noise $\be_t$ and structure of the control inputs $\bu_t$ under which $\bx_t$ is directionally smooth.
\begin{lemma}\label{lem:pwa_sys_smoothness} Let $\cF_t$ denote the filtration generated by $\bx_t = [\bz_t \mid \bu_t]$. Suppose that $\be_{t-1} \mid \cF_{t-1}$ is $\sigdir$ smooth, and that in addition, $\bu_{t} = \bar\bK_t \bx_t + \bar{\bu}_t + \bar{\be}_t$, where $\bar \bK_t$  and $\bar{\bu}_t$ are $\cF_{t-1}$-measurable\footnote{This permits, for example, that $\bar \bK_t$ is chosen based on the previous mode $i_{t-1}$, or any estimate thereof that does not use $\bx_t$.} , and $\bar{\be}_t \mid \cF_{t-1},\be_t$ is $\sigdir$-directionally smooth. Then, $\bx_t \mid \cF_{t-1}$ is $\sigdir/\sqrt{(1+\|\bK_t\|_{\op})^2 +1}$- directionally smooth.
\end{lemma}
Directionally smooth noise distributions, such as Gaussians, are common in the study of online control \citep{dean2017sample,simchowitz2020improper}, and the smoothing condition on the input can be achieved by adding fractionally small noise, as is common in many reinforcement learning domains, such as to compute gradients in policy learning \citep{sutton1999policy} or for Model Predictive Path Integral (MPPI) Control \citep{williams2015model}.

Throughout, we keep the notation for compactly representing our parameters by letting $\paramstar_{i} = [\bstA_i \mid \bstB_i \mid \bmean^\star_i]$, the estimate at time $t$ be $\paramhat_{t,i} = [\bhatA_{t,i} \mid \bhatB_{t,i} \mid \hat\bmean_{t,i}]$, and covariates $\bx_t = (\bz_t,\bu_t)$. We let $\filt_{t}$ denote the filtration generated by $\bx_{1:t}$, and note that $\be_t$ is $\filt_{t+1}$-measurable.

\begin{assumption}[Boundedness]\label{asm:pwa_bound} The covariates and parameters, as defined above, satisfy \Cref{ass:boundedness}. 
\end{assumption}
\begin{assumption}[SubGaussianity and Smoothness so as to satisfy \Cref{lem:pwa_sys_smoothness}]\label{ass:pwa_smooth} We assume that \iftoggle{icml}{$\be_t \mid \filt_t$ is $\sigdir$-directionally smooth and $\bu_{t} = \bar\bK_t \bx_t + \bar{\bu}_t + \bar{\be}_t$, where $\bar \bK_t$  and $\bar{\bu}_t$ are $\cF_{t-1}$-measurable and $\bar{\be}_t \mid \cF_{t-1},\be_t$ is $\sigdir$-directionally smooth.}{
\begin{itemize}
    \item $\be_t \mid \filt_t$ is $\sigdir$-directionally smooth
    \item $\bu_{t} = \bar\bK_t \bx_t + \bar{\bu}_t + \bar{\be}_t$, where $\bar \bK_t$  and $\bar{\bu}_t$ are $\cF_{t-1}$-measurable and $\bar{\be}_t \mid \cF_{t-1},\be_t$ is $\sigdir$-directionally smooth.
\end{itemize}
}
Further, we assume that $\be_t \mid \filt_t$ is $\nu^2$-subGaussian.
\end{assumption}
Under these assumptions, we can apply our main result, \Cref{thm:informal}, to bound the one-step prediction error in PWA systems.  In particular, we have the following:
\begin{theorem}[One-Step Regret in PWA Systems]\label{thm:onestepregret}
    Suppose that $\bz_t, \bu_t$ evolve as in \eqref{eq:pwa_system1} with the attendant notation defined therein.  Suppose further that Assumptions \ref{ass:gap}, \ref{asm:pwa_bound}, and \ref{ass:pwa_smooth} hold and that at each time $t$, the learner predicts $\bzhat_{t+1}$ with the aim of minimizing the cumulative square loss with respect to the correct $\bz_{t+1}$.  If the learner applies \Cref{alg:master} to this setting, then with probability at least $1-  \delta$, the learner experiences regret at most
    \iftoggle{icml}
    { $\sum_{t = 0}^{T-1} \norm{\bz_{t+1} - \bzhat_{t+1}}^2 \leq T \nu^2 + \mathsf{poly}\left(\maxparam, \max_{1 \leq t \leq T} \norm{\bK_t}_{\op}, \log(1/\delta)\right) \cdot T^{1 - \Omega(1)}$}
    {
    \begin{align*}
        \sum_{t = 0}^{T-1} \norm{\bz_{t+1} - \bzhat_{t+1}}^2 \leq T \nu^2 + \mathsf{poly}\left(d, \frac {\sqrt{1 + (1 + \max_{1 \leq t \leq T} \norm{\bK_t}_{\op})^2}}\sigdir, m, B, R, K, \nu, \log\left( \frac 1\delta \right)\right) \cdot T^{1 - \Omega(1)}
    \end{align*}
    }
    where the exact polynomial dependence is given in \Cref{thm:gap}.
\end{theorem}
\begin{proof}
    The result follows from applying Lemma \ref{lem:pwa_sys_smoothness} to demonstrate directional smoothness and using Assumptions \ref{ass:pwa_smooth} and \ref{asm:pwa_bound} and smoothness to apply Theorem \ref{thm:informal}.
\end{proof}
}

\subsection{Guarantees for Simulation Regret}\label{sec:simregret}
\iftoggle{icml}{We now describe an application of our results to $H$-step prediction;}{We now describe an extension of the one-step prediction desideratum to a notion of $H$-step prediction;} more formal description is detailed in \Cref{app:formal_sim_reg}.  The  learning process occurs in \emph{episodes} $t = 1,2,\dots,T$, consisting of steps $h=1,2,\dots,H$. In each episode, an external planner provides the learner a policy $\pi_t$. For simplicity, we assume that $\pi_t = (\bu_{t,1},\dots,\bu_{t,H})$ is \emph{open loop} (state-independent) and stochastic, minimicking situations such as model predictive control; extensions are mentioned in \Cref{sec:sim_reg_ext}. With each episode, the true dynamics are given by 
\begin{equation}
\begin{aligned}
    \bz_{t,h+1} &= \bA^\star_{i_{t,h}}  \bz_t + \bB^\star_{i_{t,h}} \bu_{t,h} + \bmean^\star_{i_{t,h}} + \be_{t,h}, \quad i_{t,h} = \gst(\bz_{t,h},\bu_{t,h}) \label{eq:pwa_system2}.
    \end{aligned}
\end{equation}
We assume that the process noise $\be_{t,h}$ are independently drawn from a known distribution $\cD$, and that $\bz_{t,1}$ is sampled from an arbitrary $\sigdir$-smooth distribution. At the start of episode $t$, the learner then produces estimates $\hat\bA_t,\hat\bB_t,\hat{g}_t$ and simulates the plan $\pi_t$ using the plug in estimate of the dynamics, i.e.,  
\iftoggle{icml}
{ $\hat{\bz}_{t;h+1} = \hat{\bA}_{t,\hat{i}_{t,h}}\hat\bz_{t,h} + \hat{\bB}_{t,\hat{i}_{t,h}}\bu_{t,h} + \hat{\bmean}_{t,\hat{i}_{t,h}}\hat{\be}_{t,h},$ $
\hat{i}_{t,h} = \ghat_t(\hat{\bz}_{t,h},\hat{\bu}_{t,h})$,}
{
\begin{align*}
\hat{\bz}_{t;h+1} &= \hat{\bA}_{t,\hat{i}_{t,h}}\hat\bz_{t,h} + \hat{\bB}_{t,\hat{i}_{t,h}}\bu_{t,h} + \hat{\bmean}_{t,\hat{i}_{t,h}}\hat{\be}_{t,h}, \\
\hat{i}_{t,h} &= \ghat_t(\hat{\bz}_{t,h},\hat{\bu}_{t,h}), 
\end{align*}
}
where $\hat{\be}_{t,h} \iidsim \cD$, and $\hat{\bu}_{t,h}$ are an i.i.d. draw from the stochastic, open-loop policy $\pi_t$.
Because the simulated and true processes use a different noise sequence, we consider the following notion of error, which measures the distance betwen the two \emph{distributions} induced by the dynamics, as opposed to the regret, which measure the distance between the realizations thereof. We choose the Wasserstein metric $\cW_2$, as it upper bounds the difference in expected value of any Lipschitz reward function between the true and simulated trajectories (see \Cref{app:formal_sim_reg} for a formal definition and more explanation).
\begin{definition}[Simulation Regret]\label{defn:sim_reg} Let $\cW_2(\cdot,\cdot)$ denote the $L_2$-Wasserstein distance, define the concatenations $\bx_{t,h} = (\bz_{t,h},\bu_{t,h})$ and $\hat \bx_{t,h} = (\hat \bz_{t,h},\hat \bu_{t,h})$, and let $\bar{\filt}_t$ be the filtration generated by $\{\bx_{s,h}\}_{1 \le s\le t, 1 \le h \le H}$. We define the $T$-episode, $H$-step simulation regret as \iftoggle{icml}{$\simreg_{T,H} := \iftoggle{icml}{\textstyle}{}\sum_{t=1}^T \cW_2\left(\bx_{t,1:H},\hat{\bx}_{t,1:H} \mid \bar{\filt}_{t-1}\right)^2$.}{
\begin{align*}
\simreg_{T,H} &:= \iftoggle{icml}{\textstyle}{}\sum_{t=1}^T \cW_2\left(\bx_{t,1:H},\hat{\bx}_{t,1:H} \mid \bar{\filt}_{t-1}\right)^2.
\end{align*}}
\end{definition}
Our goal is to achieve $\simreg_{T,H} \lesssim \mathrm{poly}(H) \cdot T^{1-\Omega(1)}$, but polynomial-in-$H$ dependence may be challenging for arbitrary open-loop policies and unstable dynamics of the pieces.  Thus, in the interest of simplicity, we decouple the problems of linear stability from the challenge of error compounding due to discontinuity of the dynamics by adopting the following strong assumption. 
\begin{assumption}\label{asm:lyap_matrix}There exists a \emph{known}, positive definite Lyapunov matrix $\bP \in \R^{d_z\times d_z}$ that satisfies $(\bA^\star_i)^\top \bP (\bA^\star)_i \preceq \bP$ for all modes $i \in [K]$. 
\end{assumption}
\iftoggle{icml}{}{Extensions are described in \Cref{sec:sim_reg_ext}.}  Using a minor modification of \Cref{alg:master}, detailed in \Cref{alg:simregret}, we show that we can get vanishing simulation regret, summarized in the following result: 

\begin{theorem}\label{thm:simregretinformal}
    Suppose that we are in the setting of \eqref{eq:pwa_system2} and that Assumptions \ref{ass:dirsmooth}-\ref{ass:gap} and \ref{asm:lyap_matrix} hold.  If we run a variant of \Cref{alg:master} (see \Cref{alg:simregret} in \Cref{app:algmodifications}), then with probability at least $1 - \delta$, it holds that 
    \iftoggle{icml}
    { $\simreg_{T,H} \leq \mathsf{poly}\left(\maxparam, H\right) \cdot T^{35/36}.$
    }
    {
    \begin{align*}
        \simreg_{T,H} \leq \mathsf{poly}\left(H, d, \frac 1\sigdir, m, B, R, K, \nu, \frac{1}{\Delsep}\right) \cdot T^{1 - \Omega(1)}.
    \end{align*}
    }
\end{theorem}
The proof of \Cref{thm:simregretinformal} with the exact polynomial dependence on the parameters can be found in \Cref{app:formal_sim_reg} and rests on a lemma showing that under \Cref{asm:lyap_matrix}, simulation regret at a given time $t$ can be bounded as $\mathsf{poly}(H)$ multiplied by the maximum one-step prediction error for times $t, t+1, \dots, t + H$, which is bounded in \Cref{thm:informal}.  We provide an exact recovery guarantee in the case $H = 1$ and without requiring \Cref{asm:lyap_matrix} in \Cref{app:lem:pwa_sys_one_step}.


\section{Analysis}\label{sec:body_analysis}
In this section, we present a sketch of the proof of \Cref{thm:gap}.  There are two primary sources of regret: that which is due to poorly estimating the linear parameters within a mode and that which is due to misclassification of the mode.  We analyze each source separately \iftoggle{icml}{, beginning with the parameter recovery result of \Cref{thm:parameterrecoveryinformal}}{}.



\iftoggle{icml}{
\subsection{Proof Sketch of Theorem \ref{thm:parameterrecoveryinformal}}}{
\subsection{Parameter Recovery}
We now describe how we control the regret that arises due to error in estimating the linear parameter in each mode.  In particular, we sketch a proof of \Cref{thm:parameterrecoveryinformal}, stated above, which  says that if we have seen many examples of $\barbx_t$ such that $\gstar(\barbx_t) = i$ and $\ghat(\barbx_t) = j$, then our estimated parameter $\paramhat_j$ for this mode is close to $\paramstar_i$. The theorem was stated in terms of a recovery guarantee given $T$ observations, but we ultimately apply it only to times within a given epoch.  The full proof is given in \Cref{app:par_recovery}.
\begin{proof}[Proof sketch of Theorem \ref{thm:parameterrecoveryinformal}]
}
    We break the proof into three parts: first, we show that the regret of $\ghat$ and the $\paramhat_i$, when restricted to times $t \in I_{ij}(\ghat)$, is small; second, we relate the prediction error 
    \iftoggle{icml}
    {
    $
    \textstyle \Qij(\ghat) = \sum_{t \in I_{ij}(\ghat)} \|(\paramhat_i - \paramstar_j) \barbx_t\|^2$.
    }
    {
    \begin{align*}
    \Qij(\ghat) = \sum_{t \in I_{ij}(\ghat)} \norm{(\paramhat_i - \paramstar_j) \barbx_t}^2
    \end{align*}
    }
     to the regret on $I_{ij}(\ghat)$; third, demonstrate that $\Sigma_{ij}(\ghat) = \sum_{t \in I_{ij}(\ghat)} \barbx_t \barbx_t^T$ has minimal eigenvalue bounded below by some constant with high probability.  If all three of these claims hold, then, noting that
    \iftoggle{icml}
    {
    $
        \fronorm{\paramhat_i - \paramstar_j}^2 \leq ||\Sigma_{ij}(\ghat)^{-1}||_{\op} \cdot \sum_{t \in I_{ij}(\ghat)} ||(\paramhat_i - \paramstar_j) \barbx_t||^2, $
    }
    {
    \begin{align*}
        \fronorm{\paramhat_i - \paramstar_j}^2 \leq \norm{\Sigma_{ij}(\ghat)^{-1}}_{\op} \cdot \sum_{t \in I_{ij}(\ghat)} \norm{(\paramhat_i - \paramstar_j) \barbx_t}^2,
    \end{align*}
    }
    
    we conclude.  Each of these claims requires significant technical effort in its own right.  We introduce \emph{disagreement covers}, a stronger analogue of $\epsilon$-nets adapted to our problem and allowing us to turn statements proven for a single $g \in \cG$ into ones uniform over $\cG$.  The first step is to bound the size of disagreement covers for the class $\cG$ of interest; we then prove each of the three claims for some fixed $g \in \cG$ and lift the result to apply to the data-dependent $\ghat$.   \iftoggle{icml}{This is done in }{This step is accomplished in }\Cref{subsec:disagreement}.

    We now turn to our three claims\iftoggle{icml}{}{\footnote{Note that the presentation in this sketch is not in the order of that in the section below due to the natural structure of certain logical dependencies in the proof structure.}}.  For the first, proved in \Cref{lem:rijst}, we introduce the \emph{excess risk}
    \iftoggle{icml}
    {$\Rijst(g, \param) = \sum_{t \in \Igij} \norm{\param_i \barbx_t - \by_t}^2 - \norm{\be_t + \bdelta_t}^2$, for each pair $(i,j)$,}
    {
    on each pair $(i,j)$,
    \begin{align*}
        \Rijst(g, \param) = \sum_{t \in \Igij} \norm{\param_i \barbx_t - \by_t}^2 - \norm{\be_t + \bdelta_t}^2
    \end{align*}
    }
    and note that the cumulative empirical excess risk $\sum_{i,j} \Rijst(\ghat, \paramhat) $ of predicting using $\ghat$ and $\paramhat$ returned by the ERM oracle is bounded by $\epsorac$ because $\Rijst(g_{\star},\param_{\star}) = 0$ due to \Cref{eq:setup}.  Thus, showing that $\Rijst(\ghat, \paramhat)$ is small can be done by showing that for no $i', j'$ is $\cR_{i'j'}(\ghat, \paramhat) \ll 0$.  This can be accomplished by a concentration argument for a single $g, \param_0$, coupled with a covering argument using our notion of disagreement covers to boost the result to one uniform in $\cG$.

    For the second claim, i.e., that $\Qij(\ghat)$ is controlled by $\Rijst(\ghat, \paramhat)$, we provide a proof in \Cref{lem:qij_lowerbound_fixed}.  For a fixed $g \in \cG$, we decompose $\Rijst(g, \paramhat)$ into $\Qij(g)$ and an error term.  We then use a generalization of the self-normalized martingale inequality from \citet{abbasi2011improved} to control the error in terms of $\Qij(g)$, providing a self-bounded inequality.  Finally, we rearrange this inequality and apply a union bound on a disagreement cover to boost the result to one uniform in $\cG$ and $\param$.

    To prove the third claim, that $\sigij(\ghat)$ has singular values uniformly bounded from below, we split our argument into two parts.  First, in \Cref{subsec:covariance}, we assume the following sliced-small ball condition on the data, i.e., for some $\zeta, \rho > 0$, it holds that 
    \iftoggle{icml}
    {
    $ \pp( \inprod{\barbx_t}{\bw}^2 \geq c \cdot \zeta^2 | \filt_t, \, t \in I_{ij}(\ghat) ) \geq c \cdot \rho.$
    }
    {
    \begin{equation*}
        \pp\left( \inprod{\barbx_t}{\bw}^2 \geq c \cdot \zeta^2 | \filt_t, \, t \in I_{ij}(\ghat) \right) \geq c \cdot \rho.
    \end{equation*}
    }
    Given the above condition, we establish a high probability lower bound on $\sigij(g)$ for fixed $g$ using a self-normalized martingale inequality, using analysis that may be of independent interest.  We then again apply a union bound on a disagreement cover to lift this result to be uniform in $\cG$.  Finally, in \Cref{subsec:smallball}, we provide bounds on $\zeta, \rho$ using directional smoothness and Markov's inequality.
\iftoggle{icml}{}
{
\end{proof}
\begin{sloppypar}
While we do not expand on this point here, we remark in passing that our method for bounding the size of a disagreement cover can be applied to significantly simplify the proofs of many of the results found in \citet{block2022efficient}.
\end{sloppypar}
}

\subsection{Mode Prediction}\label{sec:onlinelearningsketch}

\begin{algorithm}[!t]
    \begin{algorithmic}[1]
    \State{}\textbf{Initialize } Data $\barbx_{1:E}$, clasifiers $\bwk^0$, margin parameter $\gamma > 0$, learning rate $\eta > 0$, classifier $\widehat{g}: \cX \to [K]$
    \For{$s = 1, 2, \dots, E$}\iftoggle{icml}{\quad \algcomment{$\widetilde{\ell}$ defined in \eqref{eq:ltilgamma}}}{}
    \State{}\label{line:update}  \textbf{Receive } $\barbx_i$ and \textbf{update } $\bwk^{s} \gets \Pi_{({\cB}^{d-1})^{\times K}}\left(\bwk^{{s}-1} - \eta \nabla \widetilde{\ell}_{\gamma, i, \widehat{g}}(\bw^{{s}-1}) \right)$ \iftoggle{icml}{}{\qquad(\algcomment{ Gradient step and project so $\norm{\w_i} {\le} 1$})
    \Statex{}\qquad\algcomment{$\widetilde{\ell}_{\gamma, i, \widehat{g}}$ defined in \Cref{eq:ltilgamma}}}
    \EndFor
    \State{} \textbf{Return } $\bwk^m$
    \end{algorithmic}
      \caption{Online Gradient Descent (Single Epoch)}
      \label{alg:ogdupdate}
\end{algorithm}

\iftoggle{icml}{In this section, we address the second source of error, mode misclassification.  The primary challenge is that directional smoothness, as opposed to independent data, is insufficiently strong to guarantee that $\ghat_\tau$ generalizes well to unseen data in epoch $\tau + 1$.  We take inspiration from online learning and  stabilize our algorithm across epochs by modifying the classifier $\ghat_\tau$.  }{
As mentioned earlier, there are two sources of regret; (1) misspecification of the parameters defining the linear functions on each mode and (2)  error in estimating the modes $i_t$.  The previous section described how we control this first source, while this section addresses the second source.   \Cref{alg:master} proceeds in epochs, which allows us to apply our parameter recovery results per-epoch, thereby ensuring  that, when we correctly classify our modes, we are incurring low regret.  The primary challenge is that directional smoothness, as opposed to independence, is insufficiently strong an assumption to guarantee that a learned classifier $\ghat_\tau$ at each epoch $\tau$ generalizes well to unseen data in the next epoch.  In order to ensure such generalization, we take inspiration from online learning and enforce  stability of our algorithm across epochs by modifying the classifier $\ghat_\tau$.

}
There are two challenges in ensuring good performance of our online classifier.  First, the dynamics described in \eqref{eq:setup} are only identifiable up to a permutation of the labels.  Thus, in order to enforce stability across epochs, we need to enforce consistency of labelling accross epochs.  This task is made more difficult by the fact that different modes may be combined or split up across epochs due to the black box nature of $\ermoracle$.  We solve this by appealing to a subroutine, $\reorder$, described in \Cref{alg:reorder} in \Cref{sec:onlinelearning}, which combines modes that are sufficiently large and have sufficiently similar parameters and then relabels the modes so that similar nominal clusters persist accross epochs.

The second challenge is enforcing stability of our online classifier.  We do this by introducing a surrogate loss $\elltilgamma$, the multi-class hinge loss with parameter $\gamma > 0$\footnote{We use the hinge loss because it is convex, Lipschitz, and its relationship to the indicator loss is particularly convenient under directional smoothness.  Other convex, Lipschitz surrogates provide similar guarantees.}.  Formally, let $\bwk = (\bw_1, \dots, \bw_K)$ for $\bw_i \in \cB^{d-1}$, and for $g \in \G$, define the $(1/\gamma)$-Lipschitz and convex surrogate loss:
\iftoggle{icml}
{
    \begin{equation}
    \elltilgamma[g](\bwk) = \max\left(0, \max_{j \neq g(\barbx_t)} \left(1 - \tfrac{\inprod{\bw_{g(\barbx_t)} - \bw_j}{\barbx_t}}{\gamma} \right)\right). \label{eq:ltilgamma}
\end{equation}
}
{
\begin{equation}
    \elltilgamma[g](\bwk) = \max\left(0, \max_{j \neq g(\barbx_t)} \left(1 - \frac{\inprod{\bw_{g(\barbx_t)} - \bw_j}{\barbx_t}}{\gamma} \right)\right). \label{eq:ltilgamma}
\end{equation}
}
In $\ogdupdate$ (see \Cref{alg:ogdupdate}), we modify the output of $\ermoracle$ to construct a classifier $\gtil_\tau$ by running lazy online gradient descent on this surrogate loss, using the reordered labels given by the output of $\ermoracle$ evaluated on the previous epoch.  By invoking $\ogdupdate$, we show that our mode classifier is  sufficiently stable to ensure low regret, as stated informally in the following result.
\begin{theorem}[Mode Prediction, Informal Statement of Theorem \ref{thm:onlinelearning}]
    Suppose that Assumptions \ref{ass:dirsmooth}-\ref{ass:gap} hold and let $\cE_{t}^{\mathrm{mode}}$ be the event that the learner, invoking \Cref{alg:master} with correctly tuned parameters, misclassifies the mode of $\barbx_t$.  Then, with probability at least $1 - \delta$, it holds that 
    \iftoggle{icml}
    {
    $ \sum_{t = 1}^T \I\{\cE_{t}^{\mathrm{mode}}\} \leq \mathrm{poly}\left( \maxparam, \log\left( \frac 1\delta \right) \right) \cdot T^{1 - \Omega(1)}.$
    }
    {
    \begin{align*}
        \sum_{t = 1}^T \I\{\cE_{t}^{\mathrm{mode}}\} \leq \mathrm{poly}\left( K, B, \frac{1}{\sigdir}, R, d, m ,\nu,\frac 1{\Delsep}, \log\left( \frac 1\delta \right) \right) \cdot T^{1 - \Omega(1)}.
    \end{align*}
    }
\end{theorem}
\begin{proof}[Proof Sketch]
    We begin by addressing the lack of identifiability of the modes.  For the purposes of analysis, for each epoch $\tau$, we introduce the function $\pi_\tau : [K] \to [K]$ such that \iftoggle{icml}{$\pi_\tau(i) = \argmin_{1 \leq j \leq K} \fronorm{\paramhat_{\tau, i} - \paramstar_j}$.}{
    \begin{align*}
        \pi_\tau(i) = \argmin_{1 \leq j \leq K} \fronorm{\paramhat_{\tau, i} - \paramstar_j}.
    \end{align*}
    }
    In words, $\pi_\tau(i)$ is the mode $j$ according to the ground truth whose parameters are closest to those estimated by $\ermoracle$.  We let $\cE_t^{\mathrm{mode}}$ denote the event of misclassifying the mode, i.e. the event that $\pi_\tau(\gtil_\tau(\barbx_t)) \neq \gstar(\barbx_t)$.  We can then decompose the misclassification as
    \iftoggle{icml}
    {$\cE_t^{\mathrm{mode}} \subset \left\{ \pi_{\tau + 1}(\gtil_\tau(\barbx_t)) \neq \gstar(\barbx_t) \right\} \cup \left\{ \pi_\tau(\gtil_\tau(\barbx_t)) \neq \pi_{\tau+1}(\gtil_\tau(\barbx_t)) \right\}.$
    }
    {
    \begin{align*}
        \cE_t^{\mathrm{mode}} \subset \left\{ \pi_{\tau + 1}(\gtil_\tau(\barbx_t)) \neq \gstar(\barbx_t) \right\} \cup \left\{ \pi_\tau(\gtil_\tau(\barbx_t)) \neq \pi_{\tau+1}(\gtil_\tau(\barbx_t)) \right\}.
    \end{align*}
    }
    We call the first event look-ahead classification error and the second permutation disagreement error.  To bound the first of these, we further decompose the look-ahead classification error event:
    \iftoggle{icml}
    {
$
        \left\{ \pi_{\tau + 1}(\gtil_\tau(\barbx_t)) \neq \gstar(\barbx_t) \right\} \subset \left\{ \pi_{\tau+1}(\ghat_{\tau+1}(\barbx_t)) \neq \gstar(\barbx_t) \right\} \cup \left\{ \gtil_\tau(\barbx_t) \neq \ghat_{\tau+1}(\barbx_t) \right\}.$
    }
    {
    \begin{align*}
        \left\{ \pi_{\tau + 1}(\gtil_\tau(\barbx_t)) \neq \gstar(\barbx_t) \right\} \subset \left\{ \pi_{\tau+1}(\ghat_{\tau+1}(\barbx_t)) \neq \gstar(\barbx_t) \right\} \cup \left\{ \gtil_\tau(\barbx_t) \neq \ghat_{\tau+1}(\barbx_t) \right\}.
    \end{align*}
    }
    Controlling the first of these events is proved as a corollary of the bound on the permutation disagreement error and we defer its discussion for the sake of simplicity.  To bound the probability of the second event, upper bound the indicator loss by $\elltilgamma[\ghat_{\tau+1}]$ and apply standard online learning techniques to show that regret of OGD on the $(1/\gamma)$-Lipschitz and convex surrogate loss is small.  We then use smoothness to show that the optimal comparator with respect to the surrogate loss does not experience large regret with respect to the indicator loss, which in turn has a comparator experiencing zero loss due to the well-specified nature of $\gstar$.  Thus, putting everything together, this provides bound on the look-ahead classification error, with full details presented in \Cref{app:lookahead}.

    Bounding the permutation disagreement error is what necessitates the gap assumption.  We show that if there are sufficiently many data points that $\ghat_\tau$ assigns to a given mode, with the threshold defined in terms of $\Delsep$ and other parameters of the problem, then with high probability the cluster becomes stable in the sense that $\paramhat_{\tau, i} \approx \paramhat_{\tau+1,i}$.  This result is proved using the fact that if $\abs{I_{i,j}(\ghat_\tau)}$ is large enough, then the parameter recovery results from the previous section tell us that $\fronorm{\paramhat_i - \paramhat_j}$ is small.  Using the triangle inequality and the fact that $\fronorm{\paramstar_{j} - \paramstar_{j'}} > \Delsep$ for $j' \neq j$ ensures that $j$ is unique in satisfying $\fronorm{\paramhat_{\tau,i} - \paramstar_j} \ll 1$.  A similar argument applies to epoch $\tau + 1$, and thus we can identify $\paramhat_{\tau, i}$ with some $\paramhat_{\tau+1, i'}$ by these matrices being sufficiently close in Frobenius norm.  \Cref{alg:reorder} takes advantage of exactly this property and permutes the labels across epochs in order to maintain consistency and control the permutation disagreement error; full details can be found in \Cref{sec:permutation_term}.  Combining this argument with the bound on the look-ahead classification error suffices to control the online classification component of the regret.
\end{proof}

\section{Discussion}
We have given an efficient online learning algorithm for prediction in piecewise affine (PWA) systems. Our results are the first foray into the study of such systems, and a number of outstanding questions remain:
\iftoggle{icml}
{ Does directional smoothness facilliate low regret for planning and control, in addition to for simulation and prediction? 
If the systems are piecewise-affine but forced to be continuous on their boundaries, is there an oracle efficient algorithm which suffers low regret without assuming directional smoothness? 
Can learning in piecewise affine systems be obtained under \emph{partial observation} of the system state?
%
%
Can the dependence of regret on horizon be improved?
}
{
\begin{itemize}
	\item Does directional smoothness facilliate low regret for planning and control, in addition to for simulation and prediction?
	\item If the systems are piecewise-affine but forced to be continuous on their boundaries, is there an oracle efficient algorithm which suffers low regret without assuming directional smoothness? Note that exponential weights on a suitable discretization suffices for an information-theoretic bound. 
	\item Can learning in piecewise affine systems be obtained under \emph{partial observation} of the system state?
	\item What guarantees are possible if the underlying system can be approximated by a PWA system with many pieces, even if it is not exactly described by one? Note that this may require algorithms which overcome the separation condition, \Cref{ass:gap}.
	\item What other class of systems, like Linear Complementarity systems, studied in \citep{jin2022learning}, are amenable to oracle-efficient online learning?
\end{itemize}
}
We hope that our work initiates investigation of the above as the field continues to expand its understanding beyond the linear regime.

\iftoggle{icml}{}{
\section*{Acknowledgements}
}
AB acknowledges support from the National Science Foundation Graduate Research Fellowship
under Grant No. 1122374 as well as support from ONR under grant N00014-20-1-2336 and DOE under grant DE-SC0022199. MS acknowledges support from Amazon.com Services LLC grant; PO\# 2D-06310236.  We also acknowledge Zak Mhammedi, Terry H.J. Suh and Tao Pang for their helpful comments.

\bibliographystyle{plainnat}
\bibliography{smoothedonlineoptimistic.bib}

\begin{thebibliography}{42}
\providecommand{\natexlab}[1]{#1}
\providecommand{\url}[1]{\texttt{#1}}
\expandafter\ifx\csname urlstyle\endcsname\relax
  \providecommand{\doi}[1]{doi: #1}\else
  \providecommand{\doi}{doi: \begingroup \urlstyle{rm}\Url}\fi

\bibitem[Abbasi-Yadkori et~al.(2011)Abbasi-Yadkori, P{\'a}l, and
  Szepesv{\'a}ri]{abbasi2011improved}
Yasin Abbasi-Yadkori, D{\'a}vid P{\'a}l, and Csaba Szepesv{\'a}ri.
\newblock Improved algorithms for linear stochastic bandits.
\newblock \emph{Advances in neural information processing systems}, 24, 2011.

\bibitem[Amaldi and Mattavelli(2002)]{amaldi2002min}
Edoardo Amaldi and Marco Mattavelli.
\newblock The min pfs problem and piecewise linear model estimation.
\newblock \emph{Discrete Applied Mathematics}, 118\penalty0 (1-2):\penalty0
  115--143, 2002.

\bibitem[Balcan et~al.(2018)Balcan, Dick, and Vitercik]{balcan2018dispersion}
Maria-Florina Balcan, Travis Dick, and Ellen Vitercik.
\newblock Dispersion for data-driven algorithm design, online learning, and
  private optimization.
\newblock In \emph{2018 IEEE 59th Annual Symposium on Foundations of Computer
  Science (FOCS)}, pages 603--614. IEEE, 2018.

\bibitem[Balcan et~al.(2021)Balcan, DeBlasio, Dick, Kingsford, Sandholm, and
  Vitercik]{balcan2021much}
Maria-Florina Balcan, Dan DeBlasio, Travis Dick, Carl Kingsford, Tuomas
  Sandholm, and Ellen Vitercik.
\newblock How much data is sufficient to learn high-performing algorithms?
  generalization guarantees for data-driven algorithm design.
\newblock In \emph{Proceedings of the 53rd Annual ACM SIGACT Symposium on
  Theory of Computing}, pages 919--932, 2021.

\bibitem[Block and Polyanskiy(2023)]{block2023sample}
Adam Block and Yury Polyanskiy.
\newblock The sample complexity of approximate rejection sampling with
  applications to smoothed online learning.
\newblock \emph{arXiv preprint arXiv:2302.04658}, 2023.

\bibitem[Block and Simchowitz(2022)]{block2022efficient}
Adam Block and Max Simchowitz.
\newblock Efficient and near-optimal smoothed online learning for generalized
  linear functions.
\newblock \emph{arXiv preprint arXiv:2205.13056}, 2022.

\bibitem[Block et~al.(2022)Block, Dagan, Golowich, and
  Rakhlin]{block2022smoothed}
Adam Block, Yuval Dagan, Noah Golowich, and Alexander Rakhlin.
\newblock Smoothed online learning is as easy as statistical learning.
\newblock \emph{arXiv preprint arXiv:2202.04690}, 2022.

\bibitem[Borrelli et~al.(2017)Borrelli, Bemporad, and
  Morari]{borrelli2017predictive}
Francesco Borrelli, Alberto Bemporad, and Manfred Morari.
\newblock \emph{Predictive control for linear and hybrid systems}.
\newblock Cambridge University Press, 2017.

\bibitem[Cesa-Bianchi and Lugosi(2006)]{cesa2006prediction}
Nicolo Cesa-Bianchi and G{\'a}bor Lugosi.
\newblock \emph{Prediction, learning, and games}.
\newblock Cambridge university press, 2006.

\bibitem[Dean et~al.(2020)Dean, Mania, Matni, Recht, and Tu]{dean2017sample}
Sarah Dean, Horia Mania, Nikolai Matni, Benjamin Recht, and Stephen Tu.
\newblock On the sample complexity of the linear quadratic regulator.
\newblock \emph{Foundations of Computational Mathematics}, 20\penalty0
  (4):\penalty0 633--679, 2020.

\bibitem[Federer(2014)]{federer2014geometric}
Herbert Federer.
\newblock \emph{Geometric measure theory}.
\newblock Springer, 2014.

\bibitem[Ferrari-Trecate et~al.(2003)Ferrari-Trecate, Muselli, Liberati, and
  Morari]{ferrari2003clustering}
Giancarlo Ferrari-Trecate, Marco Muselli, Diego Liberati, and Manfred Morari.
\newblock A clustering technique for the identification of piecewise affine
  systems.
\newblock \emph{Automatica}, 39\penalty0 (2):\penalty0 205--217, 2003.

\bibitem[Garulli et~al.(2012)Garulli, Paoletti, and Vicino]{garulli2012survey}
Andrea Garulli, Simone Paoletti, and Antonio Vicino.
\newblock A survey on switched and piecewise affine system identification.
\newblock \emph{IFAC Proceedings Volumes}, 45\penalty0 (16):\penalty0 344--355,
  2012.

\bibitem[Goodfellow et~al.(2016)Goodfellow, Bengio, and
  Courville]{goodfellow2016deep}
Ian Goodfellow, Yoshua Bengio, and Aaron Courville.
\newblock \emph{Deep learning}.
\newblock MIT press, 2016.

\bibitem[Guruswami and Raghavendra(2009)]{guruswami2009hardness}
Venkatesan Guruswami and Prasad Raghavendra.
\newblock Hardness of learning halfspaces with noise.
\newblock \emph{SIAM Journal on Computing}, 39\penalty0 (2):\penalty0 742--765,
  2009.

\bibitem[Haghtalab et~al.(2020)Haghtalab, Roughgarden, and
  Shetty]{haghtalab2020smoothed}
Nika Haghtalab, Tim Roughgarden, and Abhishek Shetty.
\newblock Smoothed analysis of online and differentially private learning.
\newblock \emph{arXiv preprint arXiv:2006.10129}, 2020.

\bibitem[Haghtalab et~al.(2021)Haghtalab, Roughgarden, and
  Shetty]{haghtalab2021smoothed}
Nika Haghtalab, Tim Roughgarden, and Abhishek Shetty.
\newblock Smoothed analysis with adaptive adversaries.
\newblock \emph{arXiv preprint arXiv:2102.08446}, 2021.

\bibitem[Haghtalab et~al.(2022)Haghtalab, Han, Shetty, and
  Yang]{haghtalab2022oracle}
Nika Haghtalab, Yanjun Han, Abhishek Shetty, and Kunhe Yang.
\newblock Oracle-efficient online learning for beyond worst-case adversaries.
\newblock \emph{arXiv preprint arXiv:2202.08549}, 2022.

\bibitem[Hazan and Koren(2016)]{hazan2016computational}
Elad Hazan and Tomer Koren.
\newblock The computational power of optimization in online learning.
\newblock In \emph{Proceedings of the forty-eighth annual ACM symposium on
  Theory of Computing}, pages 128--141, 2016.

\bibitem[Hazan and Seshadhri(2009)]{hazan2009efficient}
Elad Hazan and Comandur Seshadhri.
\newblock Efficient learning algorithms for changing environments.
\newblock In \emph{Proceedings of the 26th annual international conference on
  machine learning}, pages 393--400, 2009.

\bibitem[Hazan et~al.(2016)]{hazan2016introduction}
Elad Hazan et~al.
\newblock Introduction to online convex optimization.
\newblock \emph{Foundations and Trends{\textregistered} in Optimization},
  2\penalty0 (3-4):\penalty0 157--325, 2016.

\bibitem[Jin et~al.(2022)Jin, Aydinoglu, Halm, and Posa]{jin2022learning}
Wanxin Jin, Alp Aydinoglu, Mathew Halm, and Michael Posa.
\newblock Learning linear complementarity systems.
\newblock In \emph{Learning for Dynamics and Control Conference}, pages
  1137--1149. PMLR, 2022.

\bibitem[Jones and Morari(2009)]{jones2009approximate}
Colin~N Jones and Manfred Morari.
\newblock Approximate explicit mpc using bilevel optimization.
\newblock In \emph{2009 European control conference (ECC)}, pages 2396--2401.
  IEEE, 2009.

\bibitem[Kalai and Vempala(2005)]{kalai2005efficient}
Adam Kalai and Santosh Vempala.
\newblock Efficient algorithms for online decision problems.
\newblock \emph{Journal of Computer and System Sciences}, 71\penalty0
  (3):\penalty0 291--307, 2005.

\bibitem[Lidec et~al.(2022)Lidec, Montaut, Schmid, Laptev, and
  Carpentier]{lidec2022leveraging}
Quentin~Le Lidec, Louis Montaut, Cordelia Schmid, Ivan Laptev, and Justin
  Carpentier.
\newblock Leveraging randomized smoothing for optimal control of nonsmooth
  dynamical systems.
\newblock \emph{arXiv preprint arXiv:2203.03986}, 2022.

\bibitem[Littlestone(1988)]{littlestone1988learning}
Nick Littlestone.
\newblock Learning quickly when irrelevant attributes abound: A new
  linear-threshold algorithm.
\newblock \emph{Machine learning}, 2\penalty0 (4):\penalty0 285--318, 1988.

\bibitem[Marcucci and Tedrake(2019)]{marcucci2019mixed}
Tobia Marcucci and Russ Tedrake.
\newblock Mixed-integer formulations for optimal control of piecewise-affine
  systems.
\newblock In \emph{Proceedings of the 22nd ACM International Conference on
  Hybrid Systems: Computation and Control}, pages 230--239, 2019.

\bibitem[Mason(2001)]{mason2001mechanics}
Matthew~T Mason.
\newblock \emph{Mechanics of robotic manipulation}.
\newblock MIT press, 2001.

\bibitem[Pang et~al.(2022)Pang, Suh, Yang, and Tedrake]{pang2022global}
Tao Pang, HJ~Suh, Lujie Yang, and Russ Tedrake.
\newblock Global planning for contact-rich manipulation via local smoothing of
  quasi-dynamic contact models.
\newblock \emph{arXiv preprint arXiv:2206.10787}, 2022.

\bibitem[Paoletti et~al.(2007)Paoletti, Juloski, Ferrari-Trecate, and
  Vidal]{paoletti2007identification}
Simone Paoletti, Aleksandar~Lj Juloski, Giancarlo Ferrari-Trecate, and Ren{\'e}
  Vidal.
\newblock Identification of hybrid systems a tutorial.
\newblock \emph{European journal of control}, 13\penalty0 (2-3):\penalty0
  242--260, 2007.

\bibitem[Posa et~al.(2014)Posa, Cantu, and Tedrake]{posa2014direct}
Michael Posa, Cecilia Cantu, and Russ Tedrake.
\newblock A direct method for trajectory optimization of rigid bodies through
  contact.
\newblock \emph{The International Journal of Robotics Research}, 33\penalty0
  (1):\penalty0 69--81, 2014.

\bibitem[Rakhlin et~al.(2011)Rakhlin, Sridharan, and Tewari]{rakhlin2011online}
Alexander Rakhlin, Karthik Sridharan, and Ambuj Tewari.
\newblock Online learning: Stochastic, constrained, and smoothed adversaries.
\newblock \emph{Advances in neural information processing systems},
  24:\penalty0 1764--1772, 2011.

\bibitem[Simchowitz and Foster(2020)]{simchowitz2020naive}
Max Simchowitz and Dylan Foster.
\newblock Naive exploration is optimal for online lqr.
\newblock In \emph{International Conference on Machine Learning}, pages
  8937--8948. PMLR, 2020.

\bibitem[Simchowitz et~al.(2018)Simchowitz, Mania, Tu, Jordan, and
  Recht]{simchowitz2018learning}
Max Simchowitz, Horia Mania, Stephen Tu, Michael~I Jordan, and Benjamin Recht.
\newblock Learning without mixing: Towards a sharp analysis of linear system
  identification.
\newblock In \emph{Conference On Learning Theory}, pages 439--473. PMLR, 2018.

\bibitem[Simchowitz et~al.(2019)Simchowitz, Boczar, and
  Recht]{simchowitz2019learning}
Max Simchowitz, Ross Boczar, and Benjamin Recht.
\newblock Learning linear dynamical systems with semi-parametric least squares.
\newblock In \emph{Conference on Learning Theory}, pages 2714--2802. PMLR,
  2019.

\bibitem[Simchowitz et~al.(2020)Simchowitz, Singh, and
  Hazan]{simchowitz2020improper}
Max Simchowitz, Karan Singh, and Elad Hazan.
\newblock Improper learning for non-stochastic control.
\newblock In \emph{Conference on Learning Theory}, pages 3320--3436. PMLR,
  2020.

\bibitem[Spielman and Teng(2004)]{spielman2004nearly}
Daniel~A Spielman and Shang-Hua Teng.
\newblock Nearly-linear time algorithms for graph partitioning, graph
  sparsification, and solving linear systems.
\newblock In \emph{Proceedings of the thirty-sixth annual ACM symposium on
  Theory of computing}, pages 81--90, 2004.

\bibitem[Suh et~al.(2022{\natexlab{a}})Suh, Simchowitz, Zhang, and
  Tedrake]{suh2022differentiable}
Hyung~Ju Suh, Max Simchowitz, Kaiqing Zhang, and Russ Tedrake.
\newblock Do differentiable simulators give better policy gradients?
\newblock In \emph{International Conference on Machine Learning}, pages
  20668--20696. PMLR, 2022{\natexlab{a}}.

\bibitem[Suh et~al.(2022{\natexlab{b}})Suh, Pang, and Tedrake]{suh2022bundled}
Hyung Ju~Terry Suh, Tao Pang, and Russ Tedrake.
\newblock Bundled gradients through contact via randomized smoothing.
\newblock \emph{IEEE Robotics and Automation Letters}, 7\penalty0 (2):\penalty0
  4000--4007, 2022{\natexlab{b}}.

\bibitem[Sutton et~al.(1999)Sutton, McAllester, Singh, and
  Mansour]{sutton1999policy}
Richard~S Sutton, David McAllester, Satinder Singh, and Yishay Mansour.
\newblock Policy gradient methods for reinforcement learning with function
  approximation.
\newblock \emph{Advances in neural information processing systems}, 12, 1999.

\bibitem[Van Der~Schaft and Schumacher(2007)]{van2007introduction}
Arjan~J Van Der~Schaft and Hans Schumacher.
\newblock \emph{An introduction to hybrid dynamical systems}, volume 251.
\newblock springer, 2007.

\bibitem[Williams et~al.(2015)Williams, Aldrich, and
  Theodorou]{williams2015model}
Grady Williams, Andrew Aldrich, and Evangelos Theodorou.
\newblock Model predictive path integral control using covariance variable
  importance sampling.
\newblock \emph{arXiv preprint arXiv:1509.01149}, 2015.

\end{thebibliography}
\newpage
\appendix
\part{General Tools}
\section{Notation}\label{sec:notation}
\begin{center}
\begin{longtable}{| l | l |}
\hline
\textbf{Problem Parameter} & \textbf{Definition}  \\
\hline
$T$ & total number of time steps\\
$K$ & number of modes \\
$d$ & dimension of $\bx$ (general PWA)\\
$m$ & dimension of $y$\\
$\sigdir$ & directional smoothness constant (\Cref{ass:dirsmooth})\\
$\nu$ & subGaussian constant of noise (\Cref{ass:subgaussian})\\
$B$ & \makecell[l]{magnitude bound on $\|\barbx_t\|$ (\Cref{ass:boundedness}) }\\
$R$ & magnitude bound on $\|\paramstar\|_{\fro}$ (\Cref{ass:boundedness})\\
$\Delsep$ & separation parameter (optional for sharper rates, see \Cref{ass:gap})\\
\hline
\end{longtable}
\end{center}
\begin{center}
\begin{longtable}{| l | l |}
\hline
\textbf{Algorithm Parameters} & \textbf{Definition}  \\
\hline
$E$ & epoch length \\
$\tau$ & current epoch \\
$\eta$ & step size (Line 4 in \Cref{alg:ogdupdate})\\
$\gamma$ & hinge loss parameter\\
$A$ & cluster size threshold (Line 11 in \Cref{alg:reorder})\\
$\epscor$ & distance from realizability (see \eqref{eq:setup}) \\
\hline
\textbf{Algorithm Objects} & \textbf{Definition}  \\
\hline
\Gape{$\elltilgamma{g}$} & hinge loss on estimated labels (see \eqref{eq:ltilgamma}) \\
$\pi_\tau$ & stabilizing permutation (see \eqref{eq:pitau})\\
$\gamma$ & hinge loss parameter\\
\hline
\textbf{Analysis Parameters} & \textbf{Definition}   \\
\hline
$\delta$ & probability of error\\
$\xi$ & scale of disagreement cover discretization (\Cref{thm:parameterrecovery})\\
$\Xi$ & minimum cluster size to ensure continuity (\Cref{condition:large_clusters})\\
\hline
\end{longtable}
\end{center}


\newcommand{\unifsim}{\overset{unif}{\sim}}

\section{Lower Bounds}

\subsection{Proof of Proposition \ref{prop:need_smoothness}}\label{sec:prop:need_smoothness}
We suppose that $d = m = 1$ and consider the unit interval.  Thus, $\cG$ is just the set of thresholds on the unit interval, i.e.,
\begin{align*}
    \cG = \left\{ \bx \mapsto \I\left\{ \bx > \theta \right\}| \theta \in [0,1] \right\}.
\end{align*}
We suppose that
\begin{align*}
    \paramstar_0 = \begin{bmatrix}
        0 & | &1
    \end{bmatrix} && \paramstar_1 = \begin{bmatrix}
        0 & | & 0
    \end{bmatrix},
\end{align*}
i.e., $\by_t = \gstar(\bx_t)$.  We are thus in the setting of adversarially learning thresholds with an oblivious adversary.  It is well known that this is unlearnable, but we sketch a proof here.  With $T$ fixed, let the adversary sample $\epsilon_1, \dots, \epsilon_T$ as independent Rademacher random variables and let
\begin{align*}
    \bx_t = \frac 12 + \sum_{s = 1}^{t-1} \epsilon_s 2^{-s-1}
\end{align*}
and let $\theta^\ast = \bx_{T + 1}$.  We observe that $\by_t = -\epsilon_t$ for all $t$.  To see this, note that if $\epsilon_t = 1$, then $x_s > x_t$ for all $s > t$ and similarly if $\epsilon_t = -1$ then $x_s < x_t$; as $\theta^\star = \bx_{T + 1}$, the claim is clear.  Note that due to the independence of the $\epsilon_t$, any $\yhat_t$ chosen by the learner is independent of $\epsilon_t$ and thus is independent of $\by_t$.  Thus the expected number of mistakes the learner makes, independent of the learner's strategy, is $\frac{T}{2}$, concluding the proof.

\subsection{A Lower Bound for Identification}
Consider a setting where there are $K = 3$ modes with state and input dimension $d = 1$. 
For a parameters $\alpha,\beta > 0$, define the linear functions
\begin{align*}
g_{1}(x;\alpha,\beta) &= 0 \\
g_{3}(x;\alpha,\beta) &= x - (1+\alpha)\\
g_{2}(x;\alpha,\beta) &=  2(x - (1+\alpha)) -\beta
\end{align*}
For our lexicographic convention in the definition of the $\argmax$ operator\footnote{alternatively, we can make this unambiguous by ommiting the points $\{1+\tau,1+2\tau\}$}, we have
\begin{align*}
\argmax_{i} g_{i}(x;\alpha,\beta) = \begin{cases} 1 & x\le 1+\alpha\\
2 & x \ge 1+\alpha +\beta\\
3 & x \in (1+\alpha, 1+\alpha+\beta)
\end{cases}
\end{align*}
That is, $x\mapsto \argmax_i g_i(x;\alpha,\beta)$ defines three modes, with the third mode a segment of length $\beta$ between $1+\alpha$ and $1+\alpha + \beta$. 
We consider simple $3$-piece  PWA systems whose regions are defined by the above linear functions.
\begin{definition} Let $\mathscr{I}(\alpha,\beta,m,\cD)$ denote the problem instance with $K = 3$ pieces where where the dynamics abide by
\begin{align*}
\bx_{t+1} = \bu_t + \bmean_{i_t} + \bw_t \quad \bmean_i = \begin{cases} 0 & i = 1,2\\
m & i = 3 \end{cases}, \quad i_t = \argmax_{i} g_i(\bx_t;\alpha,\beta), \quad \bw_t \sim \cD.
\end{align*},
\end{definition}
The following proposition shows that, regardless of the noise distribution $\cD$, learning the parameter $\bm_3$ can be make arbitrarily hard. This is because as the $\beta$ parameter is made small, making locating region $3$ (which is necesary to learn $\bm_3$) arbitrarily challenging. 
\begin{proposition}\label{prop:recovery_impossible} Fix any positive integer $N \in \N$ and any arbitrary distribution $\cD$ over $\R$. Then, any algorithm which adaptively selects inputs $\bu_1,\dots,\bu_T$ returns an estimate $\hat{\bmean}_T$ of $\bmean_3$ must suffer
\begin{align*}
\sup_{\iota \in \{-1,1\}}\sup_{j \in [2N]}\Pr_{\scrI(j/N,1/N,\iota m,\cD)} [|\hat{\bmean}_T - \bmean_3| \ge m] \ge \frac{1}{2}\left(1 - \frac{T}{N}\right). 
\end{align*}
As $N$ is arbitrary, this makes a constant-accuracy estimate of $\bmean_3$ arbitrarily difficult.
\end{proposition}
\begin{remark} The lower bound of \Cref{prop:recovery_impossible} can be circumvented in two cases. First, if the dynamics are forced to be Lipschitz continuous, then deviations of the dynamics in small regions can only lead to small differences in parameter values. Second, if there is an assumption which stipulates that all linear regions have ``large volume'', then there is an upper bound on how large $N$ can be in the above construction, obviating our lower bound. It turns out that there are practical cricumstances under which PWA systems (a) are not Lipschitz, and (b) certain regions have vanishingly small volume \citep{jones2009approximate}. Still, whether stronger guarantees are possible under such conditions is an interesting direction for future work. 
\end{remark}

\begin{proof}[Proof of \Cref{prop:recovery_impossible}] Introduce the shorthand $\Pr_{j,\iota}[\cdot] := \Pr_{\scrI(j/N,1/N,\iota m)}[\cdot]$. Consider the regions $\cR_j := [1+j/4N,1+(j+1)/4N]$, and define the event $\cE_{j,t} := \{\bu_j \notin \cR_j\}$ and $\cE_j := \bigcap_{t=1}^T \cE_{j,t}$. As the learner has information about $\bmean_3$ on $\cE_j$, then if if $\iota \unifsim \{-1,1\}$, it holds that
\begin{align*}
\Exp_{ \iota \sim \{-1,1\}} \Pr_{j,\iota}[|\hat{\bmean}_T - \bmean_3| \ge m \mid\cE_j] \ge \min_{\mu \in \R} \Pr_{ \iota} [| \iota m - \mu| \ge m] = \frac{1}{2}
\end{align*}
Thus, 
\begin{align*}
\sup_{j,\iota} \Pr_{j,\iota}[|\hat{\bmean}_T - \bmean_3| \ge m] \ge \frac{1}{2} \sup_{j,\iota}\Pr_{j,\iota}[\cE_j] 
\end{align*}
Introduce a new measure $\Pr_0$ where the dynamics are given by 
\begin{align*}
\bx_{t+1} = \bu_t + \bw_t, \quad \bw_t \iidsim \cD.
\end{align*}
We observe that on $\Pr_{j,\iota}[\cE_j]  = \Pr_0[\cE_j]$, because on $\cE_j$ the learner has never visited region $i = 3$ where $\bm_{i_t} \ne 0$. Hence,
\begin{align*}
\sup_{j,\iota} \Pr_{j,\iota}[|\hat{\bmean}_T - \bmean_3| \ge m] \ge \frac{1}{2} \sup_{j}\Pr_{0}[\cE_j] &= \frac{1}{2} \left(1 - \min_{j}\Pr_0[\cE_j^c]\right) \ge \frac{1}{2}\left(1 - \min_{j}\sum_{t=1}^T \Pr_0[\cE_{j,t}^c]\right).
\end{align*} 
Observe that $\cE_{j,t}^c := \{\bu_t \in \R_j\}$. As $\cR_{2\ell} \cap \cR_{2(\ell+1)}$ are disjoint for $\ell \in \N$, it holds that $\cE_{2\ell,t}^c$ are disjoint events for $\ell \in \N$.
Upper bound bounding minimum by average on any subset, we have 
\begin{align*}
\min_{j}\Pr_{0}[\cE_{j,t}^c] \le \frac{1}{N}\sum_{\ell=1}^N \Pr_{0}[\cE_{2\ell+1,t}^c] \overset{(i)} = \frac{1}{N}\Pr_0[\bigcup_{\ell = 1}^N \cE_{2\ell+1,t}^c] \le 1/N,
\end{align*}
where $(i)$ uses disjointness of the events $\cE_{2\ell,t}$ as argued above. Thus, continuing from the second-to-last display, we have
\begin{align*}
\sup_{j,\iota} \Pr_{j,\iota}[|\hat{\bmean}_T - \bmean_3| \ge m]  \ge \frac{1}{2}\left(1 - \frac{T}{N}\right).
\end{align*}
\end{proof}


\section{Properties of Smoothness}

\subsection{Directional Smoothness of Gaussians and Uniform Distributions}\label{sec:dist_smoothnesses}
\begin{lemma}
    Let $\bw$ be distributed as a centred Gaussian with covariance $\sigma^2 \eye$ in $\rr^d$.  Then for any $\bz \in \rr^d$, if $\bw$ is independent of $\bz$, it holds that $\bx = \bz + \bw$ is $\sigdir$-directionally smooth, with $\sigdir = \sqrt{2 \pi} \sigma$.
\end{lemma}
\begin{proof}
    Fix some $\bu \in \cS^{d-1}$ and $c \in \rr$.  Then note that
    \begin{align*}
        \pp\left( \abs{\inprod{\bx}{\bu} - c} < \delta \right) &= \pp\left( \abs{\inprod{\bw}{\bu} - \left( -\inprod{\bz}{\bu} + c  \right)}  < \delta\right) \\
        &= \pp\left( \abs{\inprod{\bw}{\bu} - c'} < \delta \right) \\
        &\leq \frac{\delta}{\sqrt{2 \pi} \sigma},
    \end{align*}
    where the last inequality follows by the fact that $\inprod{\bw}{\bu}$ is distributed as a centred univarate Gaussian with variance $\sigma^2$ and the fact that such a distribution has density upper bounded by $\frac{1}{\sqrt{2 \pi} \sigma}$.  The result follows.
\end{proof}
\begin{lemma}
    Let $\bw$ be uniform on a centred Euclidean ball of radius $\sigma$.  Then for any $\bz \in \rr^d$, if $\bw$ is independent of $\bz$, it holds that $\bx = \bz + \bw$ is $\sigdir$-directionally smooth, with $\sigdir \geq \frac{\sigma}{2}$.
\end{lemma}
\begin{proof}
    Let $\bv$ denote a point sampled uniformly from the unit Euclidean ball and note that $\bw \stackrel{d}{=} \sigma \bv$.  We then have for any $\bu \in \cS^{d-1}$,
    \begin{align*}
        \pp\left( \abs{\inprod{\bw}{\bu} - c'} < \delta \right) &= \pp\left( \abs{\inprod{\bv}{\bu} - c} < \frac{\delta}{\sigma} \right).
    \end{align*}
    Let $A = \left( c - \frac{\delta}{\sigma}, c + \frac \delta \sigma \right)$ and let $\phi: \rr^d \to \rr$ be defined so that $\phi(\bv) = \inprod{\bv}{\bu}$.  We note that $D \phi = \bu^T$ and thus $\det\left( D \phi D \phi^T \right) = 1$ uniformly.  Using the co-area formula \citep{federer2014geometric}, we see that
    \begin{align*}
        \pp\left( \abs{\inprod{\bv}{\bu} - c} < \frac{\delta}{\sigma} \right) &= \int_{\phi^{-1}(A)} d \vol_d(\bv) \\
        &= \int_{\phi^{-1}(A)} \sqrt{\det\left( D \phi(\bv) D \phi(\bv)^T \right)} d \vol_d(\bv) \\
        &= \int_{A} \vol_{d-1}\left( \phi^{-1}(y) \right) d y \\
        &\leq \left( \sup_{y \in A} \vol_{d-1}\left( \phi^{-1}(y) \right) \right) \int_{c - \frac \delta \sigma}^{c + \frac \delta \sigma} d y \\
        &\leq \frac{2 \delta}{\sigma}.
    \end{align*}
    The result follows.
\end{proof}

\subsection{Directional Smoothness Equivalent to Lebesgue Density}
\begin{lemma}\label{lem:lebesgue_density} Let $\bx \in \R^d$ be a (Borel-measurable) random vector. Then, $\bx$ is $\sigdir$-smooth if and only if, for any $\bw \in \cS^{d-1}$, $\langle \bw, \bx \rangle$ admits a density $p(\cdot)$ with respect to the Lebesgue measure on $\R$ with $\esssup_{v \in \R} p(v) \le 2/\sigdir$.
\end{lemma}
\begin{proof} The ``if'' direction is immediate. For the converse, let $\mu(B)$ denote the Lebesgue measure of a Borel set $B \subset \R$. We observe that $\sigdir$-smoothness proves that for any interval $I = (a,b]) \subset \R$, $\Pr[\langle \bw, \bx \rangle \in I] \le \sigdir\mu(I)/2$ (consider $c = (a+b)/2$ and $\delta = |b-a|/2$). Since the Borel sigma-algebra is the generated by open intervals, this implies that for any Borel subset $B$ of $\R$, $\Pr[\langle \bw, \bx \rangle \in B] \le 2\mu(B)/\sigdir$, where $\mu(\cdot)$ denotes the Lebesgue measure.  Thus we see that the law of $\inprod{\bw}{\bx}$ is absolutely continuous with respect to the Lebesgue measure and, by definition of the Radon-Nikodym derivative, there exists some $p$ such that for all Borel $B$, it holds that
    \begin{align*}
        \pp\left( \inprod{\bw}{\bx} \in B \right) = \int_B p(a) d a.
    \end{align*}
    Now, let
    \begin{align*}
        B = \left\{ a | p(a) > \frac{2}{\sigdir} \right\}
    \end{align*}
    and note that
    \begin{align*}
        \pp\left( \inprod{\bw}{\bx} \in B \right) &= \int_B p(a) d a > \frac 2\sigdir \cdot \mu(B).
    \end{align*}
    Combining this with the fact that $\pp\left( \inprod{\bw}{\bx} \in B \right) \leq \frac{2 \mu(B)}{\sigdir}$, we see that $\mu(B) = 0$ and the result holds.
\end{proof}

\subsection{Concatenation Preserves Directional Smoothness}

\begin{lemma}\label{lem:concat_prop1} Let $\bx_1 \in \R^{d_1}$ and $\bx_2 \in \R^{d_2}$ be two random vectors such that $\bx_1 \mid \bx_2$  and $\bx_2 \mid \bx_1$ are both $\sigdir$-directionally smooth.   Then, the concatenated vector $\tilde{\bx} = (\bx_1,\bx_2)$ is $\frac{\sigdir}{\sqrt{2}}$-directionally smooth.
\end{lemma}
\begin{proof} Fix $\tilde{\bw} = (\bw_1,\bw_2) \in \R^{d_1+d_2}$ with $\|\tilde{\bw}\| = 1$. Set $\alpha_1 := \|\bw_1\|$ and $\alpha_2 = \|\bw_2\|$. Assume without loss of generality that $\alpha_1 \ge \alpha_2$, so then necessarily $\alpha_1 \ge 1/\sqrt{2}$.
Then,
\begin{align*}
\Pr[|\langle \tilde \bw, \tilde \bx \rangle - c | < \delta] &= \Exp_{\bw_2}\Pr[|\langle \tilde \bw, \tilde \bx\rangle - c | < \delta \mid \bw_2]\\
&= \Exp_{\bw_2}\Pr\left[\left|\langle \bw_1, \bx_1 \rangle - (c - \langle \bw_2,\bx_2\rangle) \right| < \delta \mid \bw_2\right]\\
&= \Exp_{\bw_2}\Pr\left[\left|\langle \frac{1}{\alpha_1}\bw_1, \bx_1 \rangle - \frac{1}{\alpha_1}(c - \langle \bw_2,\bx_2\rangle) \right| < \frac{\delta}{\alpha_1} \mid \bw_2\right] \le \frac{\delta}{\sigdir\alpha_1} \le \frac{\sqrt{2}\delta}{\sigdir}.
\end{align*}
The bound follows.
\end{proof}

\begin{lemma}\label{lem:concat_prop2} Let $\bz \in \R^{d_1}$ and $\bu,\bv \in \R^{d_2}$ be random vectors such that $\bz \mid \bv$ is  $\sigdir$-directionally smooth, $\bv \mid \bz$ is $\sigdir$-directionally smooth, and $\bu = \bK \bz + \bv$. Then,  the concatenated vector $\bx = (\bz,\bu)$ is $\sigdir/\sqrt{(1+\|\bK\|_{\op})^2 + 1}$-smooth.
\end{lemma}
\begin{proof} Fix $\bar{\bw} = (\bw_1,\bw_2) \in \R^{d_1+d_2}$ such that $\|\bar{\bw}\| = 1$. Then, 
\begin{align*}
\langle \bar{\bw},\bx \rangle &= \langle \bw_1, \bz \rangle + \langle \bw_2, \bK \bz \rangle + \langle \bw_2, \bv \rangle\\
&= \langle \bw_1 + \bK^\top \bw_2, \bz \rangle + \langle \bw_2, \bv \rangle. 
\end{align*}
Define $\alpha_1 := \|\bw_1 + \bK^\top \bw_2\|$, $\alpha_2 := \|\bw_2\|$, and $X_1 := \langle \bw_1 + \bK^\top \bw_2, \bz \rangle$ and $X_2 := \langle \bw_2, \bv \rangle$.  $\bX_1$ dependends only on $z$ and $\bX_2$ only on $v$.  Hence, $X_1 \mid X_2$ is $\alpha_1 \sigdir$ smooth and $X_2 \mid X_1$ is $\alpha_2 \sigdir$ smooth. It follows that
\begin{align*}
\Pr[|X_1 + X_2 - c| \le \delta] &= \Exp_{X_1}\Pr[|X_1 + X_2 - c| \le \delta \mid X_1] \le \frac{\delta}{\sigdir \alpha_2},
\end{align*}
and so my symmetry under labels $i = 1,2$, 
\begin{align*}
\Pr[|\langle \tilde \bw, \bx \rangle - c| \le \delta] = \Pr[|X_1 + X_2 - c| \le \delta]  \le \min\left\{\frac{1}{\alpha_1},\frac{1}{\alpha_2}\right\} \frac{\delta}{\sigdir} = \frac{1}{\max\{\alpha_1,\alpha_2\}} \cdot \frac{\delta}{\sigdir}.
\end{align*}
We continue by bounding 
\begin{align*}
\max\{\alpha_1,\alpha_2\} &= \max\{\|\bw_1 + \bK^\top \bw_2\|,\|\bw_2\|\} \\
&\ge \max\{\|\bw\|_1 - \|\bK\|_{\op}\|\bw_2\|, \|\bw_2\| \}\\
&= \min_{\alpha \in [0,1]}\max\{ \sqrt{1-\alpha^2} - \|\bK\|_{\op}\alpha, \,\alpha \} \tag{$\|\bw_1\|^2 + \|\bw_2\|^2 = 1$}.
\end{align*}
The above is minimized when $\sqrt{1-\alpha^2} - \|\bK\|_{\op}\alpha = \alpha$, so that $\alpha^2(1+\|\bK\|_{\op})^2 = 1-\alpha^2$, yielding $\alpha = 1/\sqrt{(1+\|\bK\|_{\op})^2 + 1}$. Hence, $\max\{\alpha_1,\alpha_2\} \ge 1/\sqrt{(1+\|\bK\|_{\op})^2 + 1}$. The bound follows.

\end{proof}

\subsection{Smoothness of Parameters Induces Separation}
In this section, we show that if the true parameters $\paramstar_i$ are taken from a smooth distribution, then with high probability \Cref{ass:gap} is satisfied with $\Delsep$ not too small.  In particular, we have the following result:
\begin{proposition}\label{prop:smoothparamslargesep}
    Suppose that $(\paramstar_1, \dots, \paramstar_K)$ are sampled from a joint distribution on the $K$-fold product of the Frobenius-norm ball of radius $R$ in $\rr^{m \times d}$.  Suppose that the distribution of $(\paramstar_1, \dots, \paramstar_K)$ is such that for all $1 \leq i  < j \leq K$, the distribution of $\paramstar_j$ conditioned on the value of $\paramstar_i$ is $\sigdir$-directionally smooth.  Then, with probability at least $1 - \delta$, \Cref{ass:gap} is satisfied with
    \begin{align*}
        \Delsep \geq \frac{m d}{4 \sqrt{\pi}} \cdot \left( \frac{\sigdir \delta}{K^2} \right)^{\frac 1{md}}.
    \end{align*}
\end{proposition}
\begin{proof}
    By a union bound, we have
    \begin{align*}
        \pp\left( \min_{1 \leq i < j \leq K} \fronorm{\paramstar_i - \paramstar_j} < \Delsep \right) &\leq K^2 \max_{1 \leq i < j \leq K} \pp\left( \fronorm{\paramstar_i  - \paramstar_j} < \Delsep \right) \\
        &= K^2 \max_{1 \leq i < j \leq K} \ee_{\paramstar_i}\left[\pp\left( \fronorm{\paramhat_i  - \paramhat_j} < \Delsep | \paramstar_i \right)  \right] \\
        &\leq K^2 \sup_{\paramstar_i \in \rr^{m \times d}} \pp\left(\fronorm{\paramstar_j - \paramstar_i} < \Delsep | \paramstar_i  \right) \\
        &\leq K^2 \cdot \frac{\vol\left( \cB_{\Delsep}^{md} \right)}{\sigdir},
    \end{align*}
    where $\cB_{\Delsep}^{md}$ denotes the Euclidean ball in $\rr^{m\times d}$ of radius $\Delsep$ and the last inequality follows by the smoothness assumption.  Note that
    \begin{align*}
        \vol\left( \cB_{\Delsep}^{md} \right) = \frac{\pi^{\frac{md}{2}}}{\Gamma\left( \frac{md}{2}+1 \right)} \cdot \Delsep^{md}
    \end{align*}
    and thus
    \begin{align*}
        \pp\left( \min_{1 \leq i < j \leq K} \fronorm{\paramstar_i - \paramstar_j} < \Delsep \right) \leq \frac{K^2 \pi^{\frac{md}{2}} \Delsep^{md} }{\sigdir \cdot \Gamma\left( \frac{md}{2} + 1 \right)}.
    \end{align*}
    Noting that $\Gamma\left( \frac{md}{2} + 1 \right)^{\frac 1{md}} \geq \frac{md}{4}$ concludes the proof.
\end{proof}
\iftoggle{icml}{
\begin{remark}\label{rmk:gap_necessary}
    Note that by the previous result, Assumption \ref{ass:gap} is in some sense generic.  Indeed, in the original smoothed analysis of algorithms \citep{spielman2004nearly}, it was assumed that the parameter matrices were smoothed by Gaussian noise; if, in addition to smoothness in contexts $\bx_t$ we assume that the $\paramstar_i$ are drawn from a directionally smooth distribution, then \Cref{prop:smoothparamslargesep} below implies that with probability at least $1 - T^{-1}$, it holds that $\Delsep \gtrsim md \left( \frac{\sigdir}{K^2 T} \right)^{\frac 1{md}}$.  Furthermore, one reason why removing the gap assumption is difficult in our framework is that, computationally speaking, agnostically learning halfspaces is hard \citep{guruswami2009hardness}.  Without \Cref{ass:gap}, $\ermoracle$ cannot reliably separate modes and thus the postprocessing steps our main algorithm (\Cref{alg:master}) used to stabilize the predictions must also agnostically learn the modes; together with the previous observation on the difficulty of learning halfspaces, this suggests that if $\ermoracle$ is unable to separate modes, there is significant technical difficulty in achieving a oracle-efficient, no-regret algorithm.
\end{remark}
}
{}

\part{Supporting Proofs}
\newcommand{\gfilt}{\mathscr{G}}

\section{Parameter Recovery}\label{app:par_recovery}
In this section, we fix $\tau$ and let $\left\{ \paramhat_i | i \in [K] \right\}, \ghat$ denote the output of $\ermoracle(\barbx_{1:\tau E}, \by_{1: \tau E})$.  For any $g \in \G$ and $i, j \in [K]$, we denote
\begin{equation}
    I_{ij}(g) = \left\{ 1 \leq t \leq \tau E | g(\barbx_t) = i \text{ and } \gstar(\barbx_t) = j \right\}.
\end{equation}
We will show the following result:
\begin{theorem}[Parameter Recovery]\label{thm:parameterrecovery}
    Suppose that Assumptions \ref{ass:dirsmooth}-\ref{ass:boundedness} hold.  Then there is a universal constant $C$ such that for any tunable parameter $\xi \in (0,1)$ (which appears in the analysis but not in the algorithm), with probability at least $1 - \delta$, it holds for all $1 \leq i,j \leq K$ satisfying
    \begin{align*}
        \abs{I_{ij}(\ghat)} \geq C K^2 T \xi + C \frac{B^8 K d}{\sigdir^8 \xi^8} \log\left(\frac{B K T}{\sigdir \xi \delta}\right),
    \end{align*}
    it holds that
    \begin{align*}
        \fronorm{\paramhat_i - \paramstar_j}^2 \leq C \frac{B^2}{\sigdir^2 \xi^2 \abs{I_{ij}(\ghat)}} \left(\epsorac + 1 + K^3 B^2 R^2 d^2 m \nu^2 \sqrt{T} \log\left(\frac{T R B m d K}{\delta}\right)\right) + 4 B R K^2 \frac{1}{\abs{\Igij}} \sum_{t \in \Igij} \norm{\bdelta_t}
    \end{align*}
\end{theorem}
We refer the reader to the notation table in \Cref{sec:notation} for a reminder about the parameters. 
We begin by defining for any $i,j \in [K]$ and $g \in \G$,
\begin{equation*}
    \sigij(g) = \sum_{\substack{1 \leq t \leq \tau E \\ t \in \Igij}} \barbx_t \barbx_t^T,
\end{equation*}
the empirical covariance matrix on those $t \in \Igij$.  We will first show that for a fixed $g$, $\sigij(g) \succeq c I$ for a sufficiently small $c$ depending on problem parameters.  We will then introduce a complexity notion we call a disagreement cover that will allow us to lift this statement to one uniform in $\cG$, which will imply that $\sigij(\ghat) \succeq c I$.  We will then use the definition of $\ermoracle$ to show that $\sum_{t} \fronorm{(\paramhat_{\ghat(\barbx_t)} - \paramstar_{\gstar(\barbx_t)})\barbx_t}^2$ is small and the theorem will follow.

For the entirety of the proof and without loss of generality, we will assume that $T / E \in \bbZ$.  Indeed, if $T$ is not a multiple of $E$ then we suffer regret at most $O(E)$ on the last episode, which we will see does not adversely affect our rates.

\subsection{Disagreement Covers}\label{subsec:disagreement}
We begin by introducing a notion of complexity we call a disagreement cover; in contrast to standard $\epsilon$-nets, we show below that the disagreement cover provides more uniform notion of coverage.  Moreover, we then show that the size of a disagreement cover of $\cG$ can be controlled under the assumption of directional smoothness.
\begin{definition}
    Let $\scrD = \left\{ \left( g_i, \scrD_i \right) | g_i \in \cG \text{ and } \scrD_i \subset \rr^d \right\}$.  We say that $\scrD$ is an \emph{$\epsilon$-disagrement cover} if the following two properties hold:
    \begin{enumerate}
        \item For every $g \in \cG$, there exists some $i$ such that $(g_i, \scrD_i) \in \scrD$ and $\left\{ \bx \in \rr^d | g_i(\bx) \neq g(\bx) \right\} \subset \scrD_i$.
        \item For all $i$ and $t$, it holds that $\pp\left(\bx_t \in \scrD_i \mid \filt_{t-1}\right) \leq \epsilon$
    \end{enumerate}
    We will denote by $\dn(\cG, \epsilon)$ (or $\dn(\epsilon)$ when $\cG$ is clear from context) the minimal size of an $\epsilon$-disagreement cover of $\cG$.
\end{definition}
We remark that a disagreement cover is stronger than the more classical notion of an $\epsilon$-net because the sets of points where multiple different functions $g$ disagree with a single element of the cover $g_i$ has to be contained in a single set.  With an $\epsilon$-net, there is nothing stopping the existence of an $i$ such that the set of points on which at least one $g$ satisfying $\pp(g_i \neq g) \leq \epsilon$ is the entire space.  The reason that this uniformity is necessary is to provide the following bound:
\begin{lemma}\label{lem:disagreementcovererror}
    Let $\scrD$ be an $\epsilon$-disagreement cover for $\cG$.  Then, with probability at least $1 - \delta$, it holds that
    \begin{align*}
        \sup_{g \in \cG} \min_{(g_i, \scrD_i) \in \scrD} \sum_{t = 1}^T \I\left[g(\bx_t) \neq g_i(\bx_t)\right] \leq 2 T \epsilon + 6 \log\left(\frac{\abs{\scrD}}{\delta}\right).
    \end{align*}
\end{lemma}
\begin{proof}
    Note that by the definition of a disagreement cover, it holds that
    \begin{align*}
        \sup_{g \in \cG} \min_{(g_i, \scrD_i) \in \scrD} \sum_{t = 1}^T \I\left[g(\bx_t) \neq g_i(\bx_t)\right] \leq \max_{(g_i, \scrD_i) \in \scrD} \sum_{t = 1}^T \I\left[\bx_t \in \scrD_i\right].
    \end{align*}
    Note that for any fixed $i$, it holds that $\pp\left(\bx_t \in \scrD_i | \filt_{t-1}\right) \leq \epsilon$, also by construction.  Applying a Chernoff bound (\Cref{lem:chernoff}), we see that
    \begin{align*}
        \pp\left(\sum_{t = 1}^T \I\left[\bx_t \in \scrD_i\right] \geq 2 T\epsilon + 6 \log\left(\frac 1\delta\right)\right) &\leq \delta.
    \end{align*}
    Taking a union bound over $\scrD_i \in \scrD$ concludes the proof.
\end{proof}
For the sake of completeness, we state and prove the standard Chernoff Bound with dependent data used in the previous argument:
\begin{lemma}[Chernoff Bound] \label{lem:chernoff}
    Suppose that $X_1, \dots, X_T$ are random variables such that $X_t \in \left\{ 0,1 \right\}$ for all $1 \leq t \leq T$.  Suppose that there exist $p_t$ such that $\pp\left( X_t = 1 | \filt_{t-1} \right) \leq p_t$ almost surely, where $\filt_{t-1}$ is the $\sigma$-algebra generated by $X_1, \dots, X_{t-1}$.  Then
    \begin{align*}
        \pp\left( \sum_{t = 1}^T X_t > 2 \sum_{t = 1}^T p_t + \frac 12 \log\left( \frac 1\delta \right) \right) \leq \delta
    \end{align*}
\end{lemma}
\begin{proof}
    We use the standard Laplace transform trick:
    \begin{align*}
        \pp\left( \sum_{t = 1}^T X_t > 2 \sum_{t = 1}^T p_t + u \right) &= \pp\left( e^{\lambda \sum_{t=  1}^T X_t} > e^{2 \lambda u + 2 \lambda \sum_{t = 1}^T p_t} \right) \\
        &\leq e^{- 2 \lambda u - 2 \lambda \sum_{t = 1}^T p_t} \cdot \ee\left[\prod_{t=  1}^T e^{\lambda X_t} \right] \\
        &= e^{- 2 \lambda u - 2 \lambda \sum_{t = 1}^T p_t} \cdot \ee\left[ \prod_{t = 1}^T \ee\left[ e^{\lambda X_t} | \filt_{t-1} \right] \right] \\
        &\leq  e^{- 2 \lambda u - 2 \lambda \sum_{t = 1}^T p_t} \cdot \ee\left[ \prod_{t =1}^T e^\lambda \pp\left( X_t = 1 | \filt_{t-1} \right) + (1 - \pp\left( X_t = 1 | \filt_{t-1} \right)) \right] \\
        &\leq e^{- 2 \lambda u - 2 \lambda \sum_{t = 1}^T p_t} \cdot \prod_{t = 1}^T e^{\left( e^\lambda - 1 \right) p_t}.
    \end{align*}
    Thus we see that
    \begin{align*}
        \pp\left( \sum_{t = 1}^T X_t > 2 \sum_{t = 1}^T p_t + u \right) &\leq e^{- 2 \lambda u + \left( e^\lambda - 1 - 2 \lambda \right) \sum_{t = 1}^T p_t}.
    \end{align*}
    Setting $\lambda = 1$ and noting that $\sum_{t = 1}^T p_t > 0$ tells us that
    \begin{align*}
        \pp\left( \sum_{t = 1}^T X_t > 2 \sum_{t = 1}^T p_t + u \right) \leq e^{- 2 u}
    \end{align*}
    and the result follows.
\end{proof}

Returning to the main thread, we see that \Cref{lem:disagreementcovererror} allows us to uniformly bound the approximation error of considering a disagreement cover.  Before we can apply the result, however, we need to show that this complexity notion is small for the relevant class, $\cG$.  We have the following result:
\begin{lemma}\label{lem:disagreementcoversmall}
    Let $\cG$ be the set of classifiers considered above.  Then it holds that
    \begin{align*}
        \log\left(\dn(\cG, \epsilon)\right) \leq K (d + 1) \log\left(\frac{3 B K}{\epsilon}\right)
    \end{align*}
\end{lemma}
We prove the result in two parts.  For the first part, we show that any function class that is constructed by aggregating $K$ simpler classes has a disagreement cover whose size is controlled by that of the $K$ classes:
\begin{lemma}\label{lem:disagreement1}
	Let $\G_1, \dots, \G_K$ be function classes mapping $\cX \to \cY$.  Let $h : \cY^{\times K} \to \rr$ be some aggregating function and define 
	\begin{equation*}
		\G = \left\{\bx \mapsto h(g_1(\bx), \dots, g_K(\bx))\mid g_1 \in \G_1, \dots, g_K \in \G_K\right\}
	\end{equation*}
	Let $\dn(\G_i, \epsilon)$ be the minimal size of an $\epsilon$-disagreement cover of $\G_i$.  Then 
	\begin{equation*}
		\dn(\G, \epsilon) \leq \prod_{i = 1}^K \dn\left(\G_i, \frac \epsilon K\right)
	\end{equation*}
\end{lemma}
\begin{proof}
	Suppose that $\scrD_i$ are $\left(\frac \epsilon K\right)$-disagreement covers for $\G_i$ and for any $g_i \in \G_i$ denote by $(\pi(g_i), \scrD_i(g_i))$ the pair of functions and disagreement sets satisfying the definition of a disagreement cover.  Then we claim that
	\begin{equation*}
		\scrD = \left\{\left(h(g_1, \dots, g_K), \bigcup_i \scrD_i  \right) | (g_i, \scrD_i) \in \scrD_i \right\}
	\end{equation*}
	is an $\epsilon$-disagreement cover for $\G$.  Note that $\abs{\scrD}$ is clearly bounded by the desired quantity so this claim suffices to prove the result.

	To prove the claim, we first note that a union bound ensures that $\Pr(\bx_t \in \bigcup_i \scrD_i \mid \filt_{t-1}) \le \epsilon$.  We further note that if $g = h(g_1, \dots, g_K)$ then
	\begin{align*}
		\left\{\bx \mid g(\bx) \neq h(\pi(g_1(\bx)), \dots, \pi(g_K(\bx)))\right\} &\subset \bigcup_{1 \leq i \leq K} \left\{\bx \mid g_i(\bx) \neq \pi(g_i(\bx)) \right\} \subset \bigcup_{1 \leq i \leq K} \scrD_i(g_i).
	\end{align*}
	The result follows.
\end{proof}
We now prove that the class of linear threshold functions has bounded disagreement cover:
\begin{lemma}\label{lem:lineardisagreement}
    Let $\cH : \rr^d \to \left\{ \pm1 \right\}$ be the class of affine thresholds given by $\bx \mapsto \sign(\inprod{\bw}{\bx} + b)$ for some $\bw \in B^{d}$ and some $b \in [-B,B]$.  Then
    \begin{align*}
        \log \dn(\cH, \epsilon) \leq (d + 1) \log\left(\frac{3B}{\epsilon}\right).
    \end{align*}
\end{lemma}
\begin{proof}
    We first let $\cN := \{(\bw_i,b_i)\}$ denote an $\epsilon$-net on $\cB^d \times [-B, B]$. Associate each $(\bw_i,b_i)$ to its corresponding classifier $h_i(\bx) \sign(\inprod{\bw_i}{\bx} + b_i)$. We claim that for each $h_i$, there is a region $\scrD_i \subset \R^d$ such that $\{(h_i, \scrD_i)\}$ is an $\epsilon$-disagreement cover.  To see this, consider some $h$ with parameters $(\bw,b)$, and let $h_i$ with parameters $(\bw_i,b_i)$ ensure $\norm{\bw_i - \bw} + \abs{b_i - b} \leq \sigdir \epsilon / B$.  Consider any $\bx$ such that
    \begin{align*}
        \sign\left(\inprod{\bw_i}{\bx} + b_i\right) \neq \sign\left(\inprod{\bw}{\bx} + b\right)
    \end{align*}
    By the continuity of affine functions, there is some $\lambda \in (0, 1)$ such that if $\bw_\lambda = (1 - \lambda) \bw_i + \lambda \bw$ and $b_\lambda = (1 - \lambda) b_i + \lambda b$ then
    \begin{align*}
        \inprod{\bw_\lambda}{\bx} + b_\lambda  = 0
    \end{align*}
    Note however that
    \begin{align*}
        \abs{\inprod{\bw_\lambda}{\bx} + b_\lambda - \inprod{\bw_i}{\bx} - b_i} &\leq \lambda \abs{\inprod{\bw - \bw_i}{\bx}} + \lambda \abs{b_i - b} \leq B \left(\frac{\sigdir \epsilon}{B}\right) \leq \sigdir \epsilon
    \end{align*}
    by the definition of our $\epsilon$-net.  Thus, let
    \begin{equation*}
        \scrD_i = \left\{\bx \mid  \abs{\inprod{\bw_i}{\bx} + b} \leq \frac{\sigdir \epsilon}{B} \right\}
    \end{equation*}
    and note that the above computation tells us that if $\bx \not\in \scrD_i$ then $h_i$ must agree with $h = (\bw, b)$ for all $h$ that are mapped to $h_i$ by the projection onto the $\epsilon$-net.  Thus for all such $h$, it holds that
    \begin{align*}
        \pp\left(h(\bx_t) \neq h_i(\bx_t) | \filt_{t-1}\right) \leq \pp\left(\abs{\inprod{\bw_i}{\bx_t} + b} \leq \frac{\sigdir \epsilon}{B} | \filt_{t-1}\right) \leq \epsilon.
    \end{align*}
    Thus the claim holds and we have constructed an $\epsilon$-disagreement cover.  By noting that there are at most $\left(\frac 3\epsilon\right)^d \cdot \frac{B}{\epsilon}$ members of this cover by a volume argument we conclude the proof.
\end{proof}
By combining \Cref{lem:disagreement1} and \Cref{lem:lineardisagreement}, we prove \Cref{lem:disagreementcoversmall}.  In the next section, we will apply \Cref{lem:disagreementcoversmall} and \Cref{lem:disagreementcovererror} to lower bound $\sigij(\ghat)$.

\subsection{Lower Bounding the Covariance}\label{subsec:covariance}

We continue by lower bounding $\sigij(\ghat)$.  Before we begin, we introduce some notation.  For any $g \in \cG$, we will let
\begin{align*}
    \Ztij(g) = \I\left[g(\barbx_t) = i \text{ and } \gstar(\barbx_t) = j\right] && \barZtij(g) = \ee\left[\Ztij(g) | \filt_{t-1}\right]
\end{align*}
or, in words, $\Ztij(g)$ is the indicator of the event that a classifier predicts label $i$ when $\gstar$ predicts $j$ and $\barZtij(g)$ is probability of this event conditioned on the $\bx$-history.  We will work under the small-ball assumption that there exist constants $c_0, c_1$ as well as $\zeta_t, \rho_t$ such that for any $\bw \in \cS^{d}$, it holds for all $g \in \cG$ that
\begin{equation}\label{eq:smallball}
    \pp\left(\inprod{\barbx_t}{\bw}^2 \geq c_0 \zeta_t^2 \mid \filt_t, \Ztij(g) = 1\right) \geq c_1 \rho_t
\end{equation}
We will show that \eqref{eq:smallball} holds and defer control of the values of $\zeta_t, \rho_t$ to \Cref{subsec:smallball}, but, for the sake of clarity, we take these constants as given for now.  We proceed to show that for a single function $g$, that for all $i,j$ such that $\Igij$ is big, it holds that $\sigij(g)$ is also large.  We will then apply our results in the previous section to lift this statement to one uniform in $\cG$.  We have the following result:
\begin{lemma}\label{lem:uniform_smallball}
    Let $\rho = \min_{t} \rho_t$ and let $\zeta = \min_t \zeta_t$, where $\rho_t, \zeta_t$ are from \eqref{eq:smallball}.  Then with probability at least $1 - \delta$, for all $1 \leq i, j \leq K$ and all $g \in \cG$ such that
    \begin{align*}
        \abs{\Igij} \geq \max\left(\frac{2}{\rho^2}\left(\log\left(\frac{2T}{\delta}\right) + \frac d2 \log\left(C \frac{B \zeta^2}{\rho}\right) + \log\left(K^2 \dn(\epsilon) \right) \right), C\frac{B^2}{\zeta^2} \left(2 T \epsilon + 6 \log\left(\frac{\dn(\epsilon)}{\delta}\right)\right) \right)
    \end{align*}
    it holds that
    \begin{equation*}
        \sigij(g) \succeq \frac{c_0 \zeta^2}{8} \abs{\Igij}
    \end{equation*}
\end{lemma}
We will prove \Cref{lem:uniform_smallball} by first fixing $g \in \cG$ and showing the statement for the fixed $g$ and then using the results of \Cref{subsec:smallball} to make the statements uniform in $\cG$.  To prove the statement for a fixed $g$, we require the following self-normalized martingale inequality:
\begin{lemma}\label{lem:conditional_smallball}
	Let $\filt_t$ be a filtration with $A_t \in \filt_t$ and $B_t \in \filt_{t-1}$ for all $t$.  Let
	\begin{align*}
		p_{A_t} = \pp\left(A_t \mid \filt_{t-1}, \, B_t\right) && p_{B_t} = \pp\left(B_t \mid \filt_{t-1}\right)
	\end{align*}
	and suppose that $p_{A_t} \geq \rho$ for all $t$.  Then, with probability at least $1 - \delta$,
	\begin{equation*}
		\sum_{t = 1}^T \bbI[A_t] \bbI[B_t] \geq \rho \sum_{t = 1}^T \bbI[B_t] - \frac 12 \sqrt{2 \left(\sum_{t= 1}^T \bbI[B_t]\right) \log\left(\frac{2T}{\delta}\right)}
	\end{equation*}
	In particular with probability at least $1 - \delta$, if
	\begin{equation*}
		\sum_{t = 1}^T \bbI[B_t] \geq \frac{2}{\rho^2} \log\left(\frac{2 T}{\delta}\right)
	\end{equation*}
	then
	\begin{equation*}
		\sum_{t = 1}^T \bbI[A_t] \bbI[B_t] \geq \frac{\rho}{2} \sum_{t = 1}^T \bbI[B_t]
	\end{equation*}
\end{lemma}
In order to prove \Cref{lem:conditional_smallball}, we will require the following general result:
\begin{lemma}[Theorem 1 from \cite{abbasi2011improved}]\label{lem:self_normalized_scalar} Let $(u_t)$ be predictable with respect to a filtration $(\gfilt_t)$, and let $(e_t)$ be such that $e_t \mid \gfilt_t$ is $\sige^2$-subGaussian. Then, for any fixed parameter $\lambda> 0$, with probability $1 - \delta$,
	\begin{align*}
	\left(\sum_{t=1}^T u_t e_t\right)^2 \le 2\sige^2\left(\lambda + \sum_{t=1}^T u_t^2\right)\log \frac{(1 + \lambda^{-1}\sum_{t=1}^T u_t^2)^{1/2}}{\delta}
	\end{align*}
\end{lemma}
We now present the proof of \Cref{lem:conditional_smallball}:
\begin{proof}[Proof of Lemma \ref{lem:conditional_smallball}]
	We apply \Cref{lem:self_normalized_scalar} with $u_t = \bbI[B_t]$, $e_t = \bbI[A_t] - p_{A_t}$, and $\gfilt_t = \filt_t$.  Noting that the latter is $\frac 18$-subGaussian because it is conditionally mean zero and bounded in absolute value by $1$, we have with probability at least $1 - \delta$,
	\begin{align*}
		\abs{\sum_{t = 1}^T (\bbI[A_t] - p_{A_t}) \bbI[B_t]} \leq \frac 12 \sqrt{\left(\lambda + \sum_{t = 1}^T \bbI[B_t]\right)\log\left(\frac{\sqrt{1 + \frac{\sum_{t =1}^T \bbI[B_t]}{\lambda}}}{\delta}\right)}
	\end{align*}
	Taking a union bound over $\lambda \in [T]$ and noting that $\sum_{t = 1}^T \bbI[B_t] \in [T]$ almost surely, we recover the result.
\end{proof}
We now proceed to apply \Cref{lem:conditional_smallball} to prove a version of \Cref{lem:uniform_smallball} with the function $g \in \cG$ fixed:
\begin{lemma}\label{lem:singlefunction_smallball}
	Suppose that \eqref{eq:smallball} holds and that $\rho = \min_t \rho_t$ as well as $\zeta = \min_t \zeta_t$.  For fixed $1 \leq i, j \leq K$ and $g \in \cG$, it holds with probability at least $1 - \delta$ that if
\begin{align*}
	\abs{\Igij} \geq \frac{2}{\rho^2} \log\left(\frac{2 T}{\delta}\right) + \frac{d}{\rho^2} \log\left(C\frac{B \zeta^2}{\rho}\right)
\end{align*}
then
\begin{align*}
     \sigij(g) \succeq \frac{c_0 \zeta^2}{4} \abs{\Igij}
\end{align*}
\end{lemma}
\begin{proof}
	Note that for any fixed unit vector $\bu \in \cS^{d-1}$, the following holds:
	\begin{align*}
		\pp\left(\bu^T \sigij(g) \bu \leq \frac{c_0 \rho \zeta^2 \abs{\Igij}}{2}\right) \leq \pp\left(\sum_{t = 1}^T \I\left[\bu^T \bx_t \bx_t^T \bu \leq c_0 \zeta^2\right] \Ztij \leq \frac{\rho \abs{\Igij}}{2}\right)
	\end{align*}
    Now let $A_t$ denote the event that $\left(\inprod{\bu}{\bx_t}\right)^2 \geq c_0 \zeta^2$ and let $B_t$ denote the event that $\Ztij = 1$.  Noting that $\pp(A_t | B_t) \geq \rho$ by \eqref{eq:smallball}, we may apply \Cref{lem:conditional_smallball} and note that $\abs{\Igij}$ is just the sum of the $B_t$ to show that if $\Igij$ satisfies the assumed lower bound, then
    \begin{equation*}
        \pp\left(\sum_{t = 1}^T \I\left[\bu^T \bx_t \bx_t^T \bu \leq c_0 \zeta^2\right] \Ztij \leq \frac{\rho \abs{\Igij}}{2}\right) \leq 1 - \delta \exp\left(- \frac d2 \log\left(C\frac{R \zeta^2}\rho\right)\right).
    \end{equation*}
    Taking a union bound over an appropriately sized $\epsilon$-net on $\cS^{d-1}$ to approximate $\bu$, and applying \citet[Lemma 45]{block2022efficient} concludes the proof.
\end{proof}
We are now ready to prove the main result in this section, i.e., a lower bound on $\sigij(\ghat)$:
\begin{proof}[Proof of Lemma \ref{lem:uniform_smallball}]
    Fix $\scrD = \left\{ (g_i, \scrD_i) \right\}$ an $\epsilon$-disagreement cover of $\cG$ of size $\dn(\epsilon)$.  Taking a union bound over all $g_i$ in $\scrD$ and then applying \Cref{lem:singlefunction_smallball} tells us that with probability at least $1 - \frac{\delta}{2}$, it holds for all $1 \leq i,j \leq K$ and all $g_k \in \scrD$ such that
    \begin{align*}
        \abs{\Igij} \geq \frac{4}{\rho^2} \log\left(\frac{2 T K \dn(\epsilon)}{\delta}\right) + \frac{d}{\rho^2} \log\left(C \frac{B \zeta^2}{\rho}\right)
    \end{align*}
    we have
    \begin{align}
        \sigij(g_k) \succeq \frac{c_0 \zeta^2}{4} \abs{\Igij}.
    \end{align}
    Applying \Cref{lem:disagreementcovererror}, we see that with probability at least $1 - \frac{\delta}{2}$, 
    \begin{equation*}
        \sup_{g \in \cG} \min_{g_k \in \scrD}\sum_{t = 1}^T \I\left[ g(\bx_t) \neq g_k(\bx_t) \right] \leq 2 T \epsilon + 12 \log\left(\frac{\dn(\epsilon)}{\delta}\right).
    \end{equation*}
    Noting that by assumption, $\norm{\bx_t \bx_t^T} \leq B^2$, we see that for any $g \in \cG$, there is some $g_k \in \scrD$ such that
    \begin{align*}
        \sigij(g) \succeq \sigij(g_k) - B^2 \left(4 T \epsilon + 6 \log\left(\frac{\dn(\epsilon)}{\delta}\right)\right).
    \end{align*}
    Thus, applying the above lower bound on $\sigij(g_k)$ concludes the proof.
\end{proof}
\Cref{lem:uniform_smallball} has provided a lower bound on the empirical covariance matrices under the assumption that \eqref{eq:smallball} holds; in the next section we show that this assumption is valid.
\subsection{Small Ball Inequality}\label{subsec:smallball}
Introduce the shorthand $\Pr_t\left(\cdot \right) := \Pr\left(\cdot \mid \filt_{t-1}\right)$ and $\Exp_t\left[\cdot\right] := \Exp\left[\cdot \mid \filt_{t-1}\right]$. Recall from the previous section that we assumed that there are constants $c_0, c_1$ and $\zeta_t, \rho_t$ such that for any $\barbw \in \cS^{d}$, and any $g \in \cG$, it holds that
\begin{align}
    \pp_t\left(\inprod{\barbx_t}{\barbw}^2 \geq c_0 \zeta_t^2 \mid \Ztij(g) = 1 \right) \geq c_1 \rho_t.
\end{align}
In this section, we will show that smoothness and an assumption on $\Ztij(g)$ suffice to guarantee that this holds.  We will show the following result:
\begin{lemma}\label{lem:smallball_final}
    Suppose that Assumptions \ref{ass:boundedness} and \ref{ass:dirsmooth} hold.  Fix a $g \in \cG$ and $1 \leq i,j \leq K$ and suppose that
    \begin{align*}
        \pp_t\left(\Ztij(g)\right) \geq \xi.
    \end{align*}
    For all $\barbw \in \cS^d$, it holds that, for universal constants $c_0$ and $c_1$,
    \begin{align*}
        \pp_t\left(\inprod{\barbx_t}{\barbw}^2 \geq c_0\frac{\sigdir^2 \xi^2}{B^2} \mid  \Ztij(g) = 1\right) \geq c_1\frac{\sigdir^4 \xi^4 }{B^4}.
    \end{align*}
\end{lemma}

We begin by defining some notation.  First, we let $\Varxt$ denote variance with respect to $\bx_t$, conditioned on the $\bx$-history and $\Ztij(g)$. That is, for some function $f$ of $\bx_t$, we let
\begin{align}
    \Varxt\left[f\right] = \ee_t\left[f(\bx_t)^2 \mid  \Ztij(g)\right] - \ee_t\left[f(\bx_t) \mid  \Ztij(g)\right]^2.
\end{align}
For a fixed $\barbw \in \cS^{d}$, denote by $\bw$ the first $d$ coordinates. For fixed $g \in \cG$, and $1 \leq i,j \leq K$, we will denote:
\begin{align*}
    \rho_t(g, \barbw) &= \left(\frac{\Varxt\left[\inprod{\bx_t}{\bw} \right]}{\norm{\bw}^2 \vee \left(\ee_t\left[\inprod{\bx_t}{\bw}^2 \mid \Ztij(g) = 1\right]\right)}\right)^2 \cdot \frac{1}{32\left(1 +\ee_t\left[\inprod{\bx_t}{\bw}^4 \mid \Ztij(g) = 1\right]\right)} \\
    \zeta_t^2(g, \barbw) &= \frac{\Varxt\left[\inprod{\bx_t}{\bw} \right]}{\norm{\bw}^2 \vee \left(\ee_t\left[\inprod{\bx_t}{\bw}^2 \mid \Ztij(g) = 1\right]\right)}.
\end{align*}
We will now show the following result:
\begin{lemma}\label{lem:smallballzetarho}
    It holds for any fixed $1 \leq i,j \leq K$ and $g \in \cG$ that
    \begin{align*}
        \pp_t\left( \inprod{\barbx_t}{\barbw}^2 \geq c_0 \zeta_t^2(g, \barbw)  \Ztij(g) = 1\right) \geq c_1 \rho_t(g, \barbw)
    \end{align*}
\end{lemma}
After proving \Cref{lem:smallballzetarho}, we will provide lower bounds on $\zeta_t$ and $\rho_t$ under the assumptions of \Cref{lem:smallball_final}, which will allow us to prove the main result.  To begin the proof of \Cref{lem:smallballzetarho}, we will begin by showing that $\inprod{\barbx_t}{\barbw}^2$ is lower bounded in expectation in the following result:
\begin{lemma}\label{lem:smallball_expectation_lowerbound}
    Suppose that Assumptions ref{ass:dirsmooth} and \ref{ass:boundedness} hold.  Then for any fixed $1 \leq i,j \leq K$ and $g \in \cG$, it holds for all $\bw \in \cS^d$ that
    \begin{align*}
        \ee_t\left[ \inprod{\barbx_t}{\barbw}^2 \mid \Ztij(g) \right] \geq \zeta_t^2(g, \barbw)
    \end{align*}
\end{lemma}
\begin{proof}
    We prove two distinct lower bounds and then combine them.  Fix $i,j,g$, and $\barbw$.  Recall that $\bw$ is $\barbw$ without its last coordinate and let $w_0$ denote the last coordinate of $\barbw$.  We compute:
    \begin{align*}
        \ee_t\left[\inprod{\barbx_t}{\barbw}^2 \mid \Ztij(g) = 1\right] &= \ee_t\left[\left(w_0 + \inprod{\bx_t}{\bw_t}\right)^2 \mid \Ztij(g) = 1\right] \\
        &\geq \Varxt\left[w_0 + \inprod{\bx_t}{\bw}\right]
    \end{align*}
    where the equality comes from noting that the last coordinate of $\barbx_t$ is 1 by construction, the inequality is trivial.  This is our first lower bound.  For our second lower bound, we compute:
    \begin{align*}
        \ee_t\left[\inprod{\barbx_t}{\barbw}^2 \mid  \Ztij(g) = 1\right] &= \ee_t\left[\left(w_0  + \inprod{\bx_t}{\bw}\right)^2 \mid  \Ztij(g) = 1\right] \\
        &\geq \frac{w_0^2}{2} - \ee\left[\inprod{\bx_t}{\bw}^2 \mid \Ztij(g) = 1\right]
    \end{align*}
    where the inequality follows from applying the numerical inequality $2(x + y)^2 \geq x^2 - 2 y^2$.  This is our second lwoer bound.  To ease notation, denote
    \begin{align*}
        \alpha = \frac{\ee_t\left[\inprod{\bx_t}{\bw}^2| \Ztij(g) = 1\right]}{\Varxt\left[\inprod{\bx_t}{\bw} \right]} \geq 1.
    \end{align*}
    Combining the two previous lower bounds, we see that
    \begin{align*}
        \ee_t\left[\inprod{\barbx_t}{\bw}^2 \mid \Ztij(g) = 1\right] &\geq \max\left(\Varxt\left[\inprod{\bx_t}{\bw} \right], \frac{w_0^2}{2} - 2 \alpha \Varxt\left[\inprod{\bx_t}{\bw}  \right]\right) \\
        &\geq \max_{0 \leq \lambda \leq 1} \left\{ (1 - \lambda )\Varxt\left[\inprod{\bx_t}{\bw}\right] + \lambda \frac{w_0^2}{2} - 2 \lambda \alpha \Varxt\left[\inprod{\bx_t}{\bw} \right]   \right\} \\
        &= \max_{0 \leq \lambda \leq 1}\left\{ \left(1 - \lambda - 2 \alpha \lambda\right) \Varut\left[\inprod{\bu_t}{\bw}\right] + \lambda \frac{w_0^2}{2}  \right\} \\
        &\geq \max_{0 \leq \alpha \leq 1}\left\{ (1 - 3 \alpha \lambda) \Varxt\left[\inprod{\bx_t}{\bw}\right] + \lambda \frac{w_0^2}{2} \right\} \\
        &\geq \left(1 - \frac{3 \alpha}{6 \alpha}\right) \Varxt\left[\inprod{\bx_t}{\bw}\right] + \frac{1}{12 \alpha} w_0^2 \\
        &= \frac 12 \norm{\bw}^2 \Varxt\left[\inprod{\bx_t}{\frac{\bw}{\norm{\bw}}}\right] + \frac{w_0^2}{2} \cdot \frac{1}{6 \alpha} \\
        &\geq \frac{1}{2}\left(\norm{\bw}^2 + w_0^2\right) \min\left(\Varxt\left[\inprod{\bx_t}{\frac{\bw}{\norm{\bw}}}\right], \frac{1}{6\alpha} \right) \\
        &= \frac{1}{2} \min\left(\Varxt\left[\inprod{\bx_t}{\frac{\bw}{\norm{\bw}}}\right], \frac{\Varxt\left[\inprod{\bx_t}{\bw}\right]}{6 \ee_t\left[\inprod{\bx_t}{\bw}^2 \mid \Ztij(g) = 1\right]} \right),
    \end{align*}
    where the third inequality follows because $\alpha \geq 1$ and the last equality follows because $1 = \norm{\barbw}^2 = \norm{\bw}^2 + w_0^2$.  The result follows.
\end{proof}
We now are prepared to use the Paley-Zygmund inequality to prove the lower bound depending on $\rho_t$ and $\zeta_t$:
\begin{proof}[Proof of Lemma \ref{lem:smallballzetarho}]
    We apply the Paley-Zygmund inequality and note that
    \begin{align*}
        \pp_t\left(\inprod{\barbx_t}{\barbw}^2 \geq \frac{\ee_t\left[\inprod{\barbx_t}{\bw}^2 \mid \Ztij(g) = 1 \right]}{2} \mid \Ztij(g) = 1 \right) &\geq \frac{ \ee_t\left[\inprod{\barbx_t}{\bw}^2 \mid \Ztij(g) = 1 \right]^2  }{4 \ee_t\left[\inprod{\barbx_t}{\barbw}^4  \mid \Ztij(g) = 1\right]} \\
        &\geq \frac{\zeta_t(g, \barbw)^4}{4 \ee_t\left[\inprod{\barbx_t}{\barbw}^4 \mid \Ztij(g) = 1\right]}
    \end{align*}
    where the second inequality follows from \Cref{lem:smallball_expectation_lowerbound}.  Applying the bound in \Cref{lem:smallball_quartic_upperbound}, proved in a computation below for clarity, it holds that the last line above is lower bounded by:
    \begin{align*}
         \left(\frac{\Varxt\left[\inprod{\bx_t}{\bw} \right]}{\norm{\bw}^2 \vee \left(\ee_t\left[\inprod{\bx_t}{\bw}^2 \mid \Ztij(g) = 1\right]\right)}\right)^2 \cdot \frac{1}{32\left(1 +\ee_t\left[\inprod{\bx_t}{\bw}^4 \mid  \Ztij(g) = 1\right]\right)}.
    \end{align*}
    The result follows by the definition of $\rho_t(g, \barbw)$.
\end{proof}
We defer the more technical computation to the end of this section.  We now have our desired small ball result, modulo the fact that we need to lower bound $\zeta_t$ and $\rho_t$.  To do this, we have the following key result that lower bounds the conditional variance:
\begin{claim}\label{lem:reversemarkov}
    Suppose that Assumptions \ref{ass:dirsmooth} and \ref{ass:boundedness} holds and that
    \begin{align*}
        \pp\left(\Ztij(g) \mid \filt_{t-1}\right) \geq \xi.
    \end{align*}
    Then it holds for any $\bw \in \cB^d$ that
    \begin{align*}
        \Varxt\left[\inprod{\bx_t}{\bw}\right] \geq \frac {\sigdir^2 \xi^2 \norm{\bw}^2}{4}.
    \end{align*}
    In particular, it holds that
    \begin{align*}
        \rho_t(g, \bw) \geq \frac{\sigdir^4 \xi^4 }{128 B^4}, \quad \zeta_t(g, \bw)^2 \geq \frac{\sigdir^2 \xi^2}{4 B^2}.
    \end{align*}
\end{claim}
\begin{proof}
    We begin by noting that by Markov's inequality, for any $\lambda > 0$, it holds that
    \begin{align*}
        \pp_t\left(\abs{\inprod{\bx_t}{\bw} - \ee\left[\inprod{\bx_t}{\bw} \mid \Ztij(g) = 1\right]} > \lambda \mid \Ztij(g) = 1\right) \leq \frac{\Varxt\left[\inprod{\bx_t}{\bw}\right]}{\lambda^2}.
    \end{align*}
    Rearranging, we see that
    \begin{align*}
        \Varxt\left[\inprod{\bx_t}{\bw}\right] \geq \sup_{\lambda > 0} \lambda^2 \cdot \pp_t\left(\abs{\inprod{\bx_t}{\bw} - \ee_t\left[\inprod{\bx_t}{\bw} \mid \Ztij(g) = 1\right]} > \lambda \mid \Ztij(g) = 1\right).
    \end{align*}
    Now we compute by Bayes' theorem,
    \begin{multline*}
        \pp_t\left(\abs{\inprod{\bx_t}{\bw} - \ee\left[\inprod{\bx_t}{\bw} \mid \Ztij(g) = 1\right]} \leq \lambda \mid \Ztij(g) = 1\right) \\
        \leq \frac{\pp_t\left(\abs{\inprod{\bx_t}{\bw} - \ee_t\left[\inprod{\bx_t}{\bw} \mid \Ztij(g) = 1\right]} \leq \lambda \mid \right)}{\xi}
    \end{multline*}
    and note that
    \begin{align*}
        &\pp_t\left(\abs{\inprod{\bx_t}{\bw} - \ee_t\left[\inprod{\bx_t}{\bw} ~ \mid  \Ztij(g) = 1\right]} \leq \lambda \right) \\
        &\quad\leq \sup_{\bw, c} \pp_t\left(\abs{\inprod{\bx_t}{\bw} - c} \leq \lambda \right) \\
        &\quad\leq \sup_{\bw, c} \pp_t\left(\abs{\inprod{\bx_t}{\frac{\bw}{\norm{\bw}}} - c} \leq \frac{\lambda}{\norm{\bw}}\right) \leq \frac{\lambda}{\norm{\bw} \sigdir},
    \end{align*}
    by \Cref{ass:dirsmooth}.  Thus we have
    \begin{align*}
        \Varxt\left[\inprod{\bx_t}{\bw}\right] \geq \sup_{\lambda > 0} \lambda^2 \left(1 - \frac{\lambda}{\xi \sigdir \norm{\bw}}\right).
    \end{align*}
    Setting $\lambda = \frac 12 \cdot \xi \sigdir \norm{\bw}$ concludes the proof of the first statement.  The second statements follow by \Cref{ass:boundedness}.
\end{proof}
Combining the lower bound in \Cref{lem:reversemarkov} with the small ball estimate of \Cref{lem:smallballzetarho} concludes the proof of \Cref{lem:smallball_final}.  We now prove the technical results required in previous proofs above.
\begin{claim}\label{lem:smallball_quartic_upperbound}
    Suppose we are in the situation of \Cref{lem:smallball_expectation_lowerbound}.  Then it holds that
    \begin{align*}
        \ee_t\left[\inprod{\barbx_t}{\bw}^4 \mid  \Ztij(g) = 1\right] \leq 8 + 8 \ee_t\left[\inprod{\bx_t}{\bw}^4 \mid \Ztij(g) = 1\right]
    \end{align*}
\end{claim}
\begin{proof}
    Continuing to use notation from the proof of \Cref{lem:smallball_expectation_lowerbound}, we compute:
    \begin{align*}
        \ee\left[\inprod{\barbx_t}{\barbw}^4 | \filt_{t-1}, \Ztij(g) = 1\right] &= \ee\left[\left(w_0 + \inprod{\bx_t}{\bw}\right)^4 \mid \filt_{t-1}, \Ztij(g) = 1\right] \\
        &\leq 8 w_0^4 + 8 \ee\left[\inprod{\bx_t}{\bw}^4 \mid \filt_{t-1}, \Ztij(g) = 1\right].
    \end{align*}
    The result follows by noting that $w_0^2 \leq w_0^2 + \norm{\bw}^2 = \norm{\barbw}^2 = 1$.
\end{proof}

With the covariance matrices lower bounded, we proceed to upper bound the regret.

\subsection{Upper Bounding the Empirical Error}\label{subsec:empiricalerrorupper}
In this section, we will show that for $1 \leq i,j \leq K$ such that $\abs{I_{ij}(\ghat)}$ is sufficiently large, the fact that $\paramhat$ is formed by minimizing the empirical risk will force
\begin{equation*}
    \Qij(\ghat) = \sum_{t \in I_{ij}(\ghat)} \norm{\left(\paramhat_i - \paramstar_j\right) \barbx_t}^2
\end{equation*}
to be small.  In the end, we will combine this bound with our lower bound on $\sigij(\ghat)$ proved above to conclude the proof of the theorem.  We will begin by introducing some notation.  For fixed $1 \leq i,j \leq K$ and a fixed $g \in \cG$, we define:
\begin{align*}
    \Rijst(g) = \sum_{t \in \Igij} \norm{\paramhat_i \barbx_t - \by_t}^2 - \norm{\paramstar_j \barbx_t - \by_t}^2.
\end{align*}
The main result of this section is as follows:
\begin{lemma}\label{lem:qij_lowerbound_uniform}
    Under Assumptions \ref{ass:dirsmooth}-\ref{ass:boundedness}, with probability at least $1 - \delta$, it holds that
    \begin{align*}
        \Qij(\ghat) &\leq \epsorac + C K^2 B R d \sqrt{T m K \log\left(\frac{T B R K}{\delta}\right)} + C K^3 d \left(4 B^2 R^2 + \nu^2 \log\left(\frac{T}{\delta}\right)\right) \log\left(\frac{B K T}{\delta}\right) \\
            &+ C \nu^2 d^2 K m \log\left(\frac{T R B m d K}{ \delta}\right) + C \nu  d\sqrt{ K m \log\left(\frac{T R B m d K }{ \delta}\right)}.
    \end{align*}
\end{lemma}
To prove this result, we begin by fixing $g \in \cG$ and $1 \leq i,j \leq K$ and bounding $\Qij(g)$ by $\Rijst(g)$.  We will then use the results of \Cref{subsec:disagreement} to make the statement uniform in $\cG$.  We proceed with the case of fixed $g$ and prove the following result:
\begin{lemma}\label{lem:qij_lowerbound_fixed}
    Suppose that Assumptions \ref{ass:dirsmooth}-\ref{ass:boundedness} all hold.  Fix $1 \leq i, j \leq K$ and some $g \in \cG$.  For any $\lambda > 0$, with probability at least $1 - \delta$, it holds that
    \begin{align*}
        \Qij(g) \leq \Rijst(g) + C \nu^2 d m \log\left(\frac{T R B m d}{\lambda \delta}\right) + C \nu \sqrt{\lambda d m \log\left(\frac{T R B m d}{\lambda \delta}\right)} + 16 B R \sum_{t \in \Igij} \norm{\bdelta_t}.
    \end{align*}
\end{lemma}

\begin{proof}
    We begin by expanding
    \begin{align*}
        \Rijst(g) &= \sum_{t \in \Igij} \norm{\paramhat_i \barbx_t - \by_t}^2 - \norm{\paramstar_j \barbx_t - \by_t}^2 \\
        &= \sum_{t \in \Igij} \norm{\left(\paramhat_i - \paramstar_j\right)\barbx_t - \be_t - \bdelta_t}^2 - \norm{\be_t + \bdelta_t}^2 \\
        &= \sum_{t \in \Igij} \norm{\left(\paramhat_i - \paramstar_j\right)\barbx_t}^2 - 2 \sum_{t \in \Igij} \inprod{\be_t + \bdelta_t}{\left(\paramhat_i - \paramstar_j\right)\barbx_t} \\
        &= \Qij(g) - 2 \sum_{t \in \Igij} \inprod{\be_t + \bdelta_t}{\left(\paramhat_i - \paramstar_j\right)\barbx_t}.
    \end{align*}
    Thus, it suffices to bound the second term.  By linearity, and the assumption on $\bdelta_t$, we have
    \begin{align*}
        2 \sum_{t \in \Igij} \inprod{\be_t + \bdelta_t}{\left(\paramhat_i - \paramstar_j\right)\barbx_t} \leq 2 \sum_{t \in \Igij} \inprod{\be_t }{\left(\paramhat_i - \paramstar_j\right)\barbx_t} + 4 B R\sum_{t \in \Igij} \norm{\bdelta_t}.
    \end{align*}
    We use a generalization of the self-normalized martingale inequality, \Cref{lem:self_normalized_scalar} to bound the first term above.  To do this, fix some $\bDelta \in \rr^{m \times (d + 1)}$ and define
    \begin{align*}
        E(\bDelta) = \sum_{t \in \Igij} \inprod{\be_t}{\bDelta \barbx_t} && V(\bDelta) = \sum_{t \in \Igij} \norm{\bDelta \barbx_t}^2.
    \end{align*}
    Noting that $V(\paramhat_i - \paramstar_j) = \Qij(g) $, we see from the above that
    \begin{align}\label{eq:qijfirst}
        \Rijst(g) = V(\paramhat_i - \paramstar_j) - 2 E(\paramhat_i - \paramstar_j)
    \end{align}
    We will now bound $E(\bDelta)$ for some fixed $\bDelta$ and then apply a covering argument to lift the statement to apply to $\paramhat_i - \paramstar_j$.  By \Cref{ass:boundedness}, we may take $\norm{\bDelta} \leq 2 R$. 

    Recall that $\filty_t$ denotes the filtration generated by $\bx_1,\dots,\bx_{t},\by_1,\dots,\by_{t-1}$.  Applying \citet[Lemma E.1]{simchowitz2019learning}, a generalization of \Cref{lem:self_normalized_scalar}, with $e_t \gets \inprod{\be_t}{\bDelta \barbx_t}$, $u_t = \bDelta \barbx_t$, and the filtration $\filty_t$, tells us that for any $\lambda > 0$, it holds with probability at least $1 - \delta$ that
    \begin{equation}\label{eq:abbasiqij}
        E(\bDelta)^2 \leq 2 \nu^2 (\lambda + V(\bDelta)) \log\left(\frac{\sqrt{1 + \frac{V(\bDelta)}{\lambda}}}{\delta}\right).
    \end{equation}
    We note by \Cref{ass:boundedness}, it holds that
    \begin{equation*}
        V(\bDelta) \leq 4 T R^2 B^2
    \end{equation*}
    and thus the subadditivity of the square root tells us that
    \begin{align*}
        E(\bDelta)^2 \leq 2 \nu^2 (\lambda + V(\bDelta)) \log\left(\frac{1 + 2 R B \sqrt{\frac{T}{\lambda}}}{\delta}\right)
    \end{align*}
    with probability at least $1 - \delta$.  Now let $\scrN$ be an $\epsilon$-net of the Frobenius ball in $\rr^{m \times (d + 1)}$ of radius $2 R$.    For small $\epsilon$, we may take $\scrN$ such that
    \begin{equation}\label{eq:qij_coversize}
        \log(\abs{\scrN}) \leq m(d + 1) \log\left(\frac{6 R}{\epsilon}\right).
    \end{equation}
    For any $\bDelta$, denote by $\bDelta'$ its projection into $\scrN$.  Then we compute
    \begin{align*}
        E(\bDelta)^2 &\leq E(\bDelta')^2 + E(\bDelta)^2 - E(\bDelta')^2 \\
        &\leq \left(\sum_{t \in \Igij} \fronorm{\be_t}^2\right) 4 R B^2T \sqrt{\fronorm{\bDelta - \bDelta'}} + 2 \nu^2 \left( \lambda + V(\bDelta') \right) \log\left(\frac{1 + 2 R B \sqrt{\frac T\lambda}}{\delta} \abs{\scrN}\right) \\
        &\leq \left(\sum_{t \in \Igij} \norm{\be_t}^2\right) 4 R B^2T \sqrt{\epsilon} + 2 \nu^2 \left( \lambda + V(\bDelta') \right) \log\left(\frac{1 + 2 R B \sqrt{\frac T\lambda}}{\delta}\abs{\scrN}\right) \\
        &\leq \left(\sum_{t \in \Igij} \norm{\be_t}^2\right) 4 R B^2T \sqrt{\epsilon} + 2 \nu^2 \left( \lambda + V(\bDelta) \right) \log\left(\frac{1 + 2 R B \sqrt{\frac T\lambda}}{\delta}\abs{\scrN}\right) \\
        &\qquad+ (V(\bDelta') - V(\bDelta))  \log\left(\frac{1 + 2 R B \sqrt{\frac T\lambda}}{\delta}\right)
    \end{align*}
    where the second inequality follows from \Cref{lem:qij_covering} along with a union bound over $\scrN$ applied to \eqref{eq:abbasiqij}, the third inequality follows from the definition of an $\epsilon$-net, and the last inequality follows from simply adding and subtracting the same term.  Applying \Cref{lem:qij_covering} once again, we have with probability at least $1 - \delta$,
    \begin{align*}
        E(\bDelta)^2 &\leq \left(\sum_{t \in \Igij} \norm{\be_t}^2\right) 4 R B^2T \sqrt{\epsilon} + 2 \nu^2 \left( \lambda + V(\bDelta) \right) \log\left(\frac{1 + 2 R B \sqrt{\frac T\lambda}}{\delta}\abs{\scrN}\right) \\ 
        &\qquad+ 4 R B^2 T \sqrt{\epsilon}\log\left(\frac{1 + 2 R B \sqrt{\frac T\lambda}}{\delta}\abs{\scrN}\right) \\
        &\leq 2 \nu^2 \left( \lambda + V(\bDelta) \right) \log\left(\frac{1 + 2 R B \sqrt{\frac T\lambda}}{\delta}\abs{\scrN}\right) + 4 R B^2 T^2 m \nu^2 \sqrt{\epsilon}\log\left(\frac{1 + 2 R B \sqrt{\frac T\lambda}}{\delta}\abs{\scrN}\right) 
    \end{align*}
    uniformly for all $\bDelta$ in the Frobenius ball of radius $2 R$.  We now choose
    \begin{equation*}
        \epsilon = C \left(\frac{T R B^2 m d}{\lambda}\right)^{-c}
    \end{equation*}
    for some universal constants $C, c$ to ensure that 
    \begin{equation*}
        4 R B^2 T^2 m \nu^2 \sqrt{\epsilon}\log\left(\frac{1 + 2 R B \sqrt{\frac T\lambda}}{\delta}\abs{\scrN}\right) \leq 2 \nu^2 \lambda
    \end{equation*}
    Thus with probability at least $1 - \delta$, it holds for all $\bDelta$ that
    \begin{align*}
        E(\bDelta)^2 \leq 2 \nu^2 (\lambda + V(\bDelta))\left(\log\left(\frac{1 + 2 R B \sqrt{\frac T\lambda}}{\delta}\right) + m(d + 1) \log\left(\frac{T R^2 B^2 m d}{\lambda}\right)  \right)
    \end{align*}
    where we applied \eqref{eq:qij_coversize} to bound the size of $\scrN$.  In particular, this holds for $\bDelta = \paramhat_i - \paramstar_j$ and thus
    \begin{align*}
        E(\paramhat_i - \paramstar_j)^2 \leq 2 \nu^2 (\lambda + \Qij(g))\left(\log\left(\frac{1 + 2 R B \sqrt{\frac T\lambda}}{\delta}\right) + m(d + 1) \log\left(\frac{T R^2 B^2 m d}{\lambda}\right)  \right)
    \end{align*}
    Plugging this back into \eqref{eq:qijfirst}, we have that with probability at least $1 - \delta$, it holds that
    \begin{align*}
        &4 B R |\Igij| \epscor + \Rijst(g) \geq \Qij(g) - \\
        &\qquad- C \nu \sqrt{(\lambda + \Qij(g)) \left(\log\left(\frac{1 + 2 R B \sqrt{\frac T\lambda}}{\delta}\right) + m(d + 1) \log\left(\frac{T R^2 B^2 m d}{\lambda}\right)\right)}.
    \end{align*}
    Rearranging concludes the proof.
    \end{proof}
    We defer proofs of the technical computations, Lemmas \ref{lem:qij_covering} and \ref{lem:sumetbound}, used in the proof of \Cref{lem:qij_lowerbound_fixed} until the end of the section.  For now, we press on to lift our bound from a statement about a fixed $g \in \cG$ to one uniform in $\cG$.  We require one last lemma before concluding this proof, however.  We are aiming to bound $\Qij(\ghat)$ by $\Rijst(\ghat)$, but we need to upper bound $\Rijst(\ghat)$.  This is the content of the following lemma:
    \begin{lemma}\label{lem:rijst}
        With probability at least $1 - \delta$, for all $1 \leq i,j \leq K$, it holds that
        \begin{align*}
            \Rijst(\ghat) &\leq \epsorac +  4 B R K^2\sum_{t \in \Igij} \norm{\bdelta_t} + C K^2 B R \sqrt{T m d \log\left(\frac{T B R K \dn(\epsilon)}{\delta}\right)} \\
            &\quad + K^2 \left(4 B^2 R^2 + \nu^2 \log\left(\frac{T}{\delta}\right)\right) \left(2 T \epsilon + 6 \log\left(\frac{\dn(\epsilon)}{\delta}\right)\right)
        \end{align*}
    \end{lemma}
    \begin{proof}
        We begin by introducing some notation.  For fixed $\param \in \rr^{m \times(d + 1)}$ and $g \in \cG$, let
        \begin{align*}
            \Rijst(g, \param) = \sum_{t \in \Igij} \norm{\param_i \barbx_t - \by_t}^2 - \norm{\be_t + \bdelta_t}^2
        \end{align*}
        and note that $\Rijst(g) = \Rijst(g, \paramhat_i)$.  First, we note that by definition of $\ermoracle$,
        \begin{align*}
            \epsorac &\geq \sum_{t = 1}^T \norm{\paramhat_{\ghat(\barbx_t)} \barbx_t - \by_t}^2 - \inf_{g, \paramhat} \sum_{t = 1}^T \norm{\paramhat_{\ghat(\barbx_t)} \barbx_t - \by_t}^2 \\
            &\geq \sum_{t = 1}^T \norm{\paramhat_{\ghat(\barbx_t)} \barbx_t - \by_t}^2 - \norm{\paramstar_{\gstar(\barbx_t)}\barbx_t - \by_t}^2 \\
            &= \sum_{1 \leq i,j \leq K} \Rijst(\ghat).
        \end{align*}
        Thus, we have
        \begin{align}\label{eq:rijupper}
            \Rijst(\ghat) \leq \epsorac - \sum_{\substack{1 \leq i',j' \leq K \\ (i,j) \neq (i', j')}} \Rijst(\ghat).
        \end{align}
        Thus it will suffice to provide a lower bound on $\Rijst(\ghat)$ that holds with high probability.  To do this, note that for fixed $g \in \cG$ and $\param_0$, we have
        \begin{align*}
            \Rijst(g, \param_0) &= \sum_{t \in \Igij} \norm{\left(\param_0 - \paramstar_j\right) \barbx_t }^2 - 2 \sum_{t \in \Igij} \inprod{\be_t + \bdelta_t}{\left(\param_0 - \paramstar_j\right) \barbx_t - \be_t - \bdelta_t} \\
            &\geq - 2 \sum_{t \in \Igij} \inprod{\be_t + \bdelta_t}{\left(\param_0 - \paramstar_j\right) \barbx_t} \\
            &\geq -2 \sum_{t \in \Igij} \inprod{\be_t}{(\param_0 - \paramstar_j) \barbx_t} - 4 B R \sum_{t \in \Igij} \norm{\bdelta_t}.
        \end{align*}
        Noting that $\be_t$ is subGaussian by \Cref{ass:subgaussian} and by \Cref{ass:boundedness} we have control over $\norm{\left(\param_0 - \paramstar_j\right) \barbx_t}$, we see that with probability at least $1 - \delta$, it holds that
        \begin{align*}
            \Rijst(g, \param_0) &\geq - 2 \sum_{t \in \Igij} \inprod{\be_t}{\left(\param_0 - \paramstar_j\right) \barbx_t} - 4 B R \sum_{t \in \Igij} \norm{\bdelta_t} \\
            & \geq - 8 BR \sqrt{\abs{\Igij} \cdot \log\left(\frac 1\delta \right)} - 4 B R \sum_{t \in \Igij} \norm{\bdelta_t}.
        \end{align*}
        Now note that 
        \begin{align*}
            \Rijst(g, \param_0) - \Rijst(g, \param_0') &= \sum_{t \in \Igij} \norm{\left(\param_0 - \paramstar_j\right) \barbx_t - \be_t - \bdelta_t}^2 - \norm{\left(\param_0' - \paramstar_j\right) \barbx_t - \be_t - \bdelta_t}^2 \\
            &= \sum_{t \in \Igij} \inprod{\left(\param_0 + \param_0' - 2 \paramstar_j\right) \barbx_t - 2 \be_t - 2 \bdelta_t}{\left(\param_0 - \param_0'\right)\barbx_t}
        \end{align*}
        Applying Cauchy Schwarz and noting that with probability at least $1 - \delta$ it holds that 
        \begin{equation*}
            \max_{1 \leq t \leq T} \norm{\be_t} \leq \nu \sqrt{\log\left(\frac T\delta\right)},
        \end{equation*}
        we have that with probability at least $1 - \delta$,
        \begin{align*}
            \Rijst(g, \param_0) - \Rijst(g, \param_0') &\leq 4 R B^2 \abs{\Igij} \norm{\param_0 - \param_0'} + 4 B R \sum_{t \in \Igij} \norm{\bdelta_t}.
        \end{align*}
        Thus if we take a union bound over a $\frac 1{T R B^2}$-net of the Frobenius ball of radius $R$ in $\rr^{m \times(d+1)}$, we see that with probability at least $1 - \delta$, it holds that
        \begin{align*}
            \Rijst(g) = \Rijst(g, \paramhat_i) &\geq - C B R \sqrt{T m(d+1) \log\left(\frac{T B R}{\delta}\right)} - 4 B R \sum_{t \in \Igij} \norm{\bdelta_t}.
        \end{align*}
        We now note that by \Cref{lem:disagreementcovererror}, if $\cD$ is an $\epsilon$-disagreement cover of $\cG$, then
        \begin{align*}
            \sup_{g \in \cG} \min_{g_k \in \cD} \sum_{t = 1}^T \I\left[g(\barbx_t) \neq g_k(\barbx_t)\right] \leq 2 T \epsilon + 6 \log\left(\frac{\dn(\epsilon)}{\delta}\right).
        \end{align*}
        We now apply \Cref{lem:discretizerij} and note that this implies that if $g$ is the projection of $\ghat$ onto the disagreement cover, then
        \begin{align*}
            \Rijst(\ghat) \geq \Rijst(g) - \left(4 B^2 R^2 + \nu^2 \log\left(\frac{T}{\delta}\right)\right) \left(2 T \epsilon + 6 \log\left(\frac{\dn(\epsilon)}{\delta}\right)\right)
        \end{align*}
        Taking a union bound over $g \in \cD$ and all pairs $(i,j)$ tells us that with probability at least $1 - \delta$,
        \begin{align*}
            \Rijst(\ghat) \geq - C B R \sqrt{T m d \log\left(\frac{T B R K \dn(\epsilon)}{\delta}\right)}- 4 B R \sum_{t \in \Igij} \norm{\bdelta_t} - \left(4 B^2 R^2 + \nu^2 \log\left(\frac{T}{\delta}\right)\right) \left(2 T \epsilon + 6 \log\left(\frac{\dn(\epsilon)}{\delta}\right)\right).
        \end{align*}
        Combining this with \eqref{eq:rijupper} concludes the proof.
    \end{proof}

    We are now finally ready to prove the main result of this section:
    \begin{proof}[Proof of \Cref{lem:qij_lowerbound_uniform}]
        We begin by noting that \Cref{lem:disagreementcovererror} tells us that with probability at least $1 - \delta$, it holds that
        \begin{equation*}
            \sup_{g \in \cG} \min_{g_k \in \cD} \sum_{t = 1}^T \I[g(\barbx_t) \neq g_k(\barbx_t)] \leq 2 T \epsilon + 6 \log\left(\frac{\dn(\epsilon)}{\delta}\right)
        \end{equation*}
        We observe that for $g, g' \in \cG$, we have
        \begin{align*}
            \Qij(g) - \Qij(g') &= \sum_{t \in \Igij} \norm{(\paramhat_i - \paramstar_j)\barbx_t}^2 - \sum_{t \in I_{ij}(g')} \norm{(\paramhat_i - \paramstar_j)\barbx_t}^2 \\
            &\leq 4R^2 B^2 \sum_{t} \abs{\Ztij(g) - \Ztij(g')} \\
            &\leq 4 R^2 B^2 \sum_{t = 1}^T \I\left[g(\barbx_t) \neq g'(\barbx_t)\right].
        \end{align*}
        By \Cref{lem:discretizerij}, it holds that with probability at least $1 - \delta$, for all $g, g' \in \cG$, we have
        \begin{align*}
            \Rijst(g) - \Rijst(g') \leq \left(4 B^2 R^2 + \nu^2 \log\left(\frac T\delta\right)\right) \sum_{t =1}^T \I\left[g(\barbx_t) \neq g'(\barbx_t)\right]
        \end{align*}
        Now, taking a union bound over a minimal disagreement cover at scale $\epsilon$ as well as all pairs $1 \leq i,j \leq K$, we see that \Cref{lem:qij_lowerbound_fixed} implies that with probability at least $1 - \delta$, it holds that for all $g_k \in \cD$ the disagreement cover,
        \begin{align*}
            \Qij(g_k) \leq \Rijst(g_k) + C \nu^2 d m \log\left(\frac{T R B m d K^2 \dn(\epsilon)}{\lambda \delta}\right) + C \nu \sqrt{\lambda d m \log\left(\frac{T R B m d K^2 \dn(\epsilon)}{\lambda \delta}\right)}.
        \end{align*}
        Letting $g$ denote the projection of $\ghat$ into the disagreement cover, we compute:
        \begin{align*}
            \Qij(\ghat) &= \Qij(\ghat) - \Qij(g) + \Qij(g) \\
            &\leq 4 R^2 B^2 \sum_{t = 1}^T \I\left[\ghat(\barbx_t) \neq g(\barbx_t)\right] + \Qij(g) \\
            &\leq 4 R^2 B^2 \sum_{t = 1}^T \I\left[\ghat(\barbx_t) \neq g(\barbx_t)\right] + \Rijst(g) + C \nu^2 d m \log\left(\frac{T R B m d K^2 \dn(\epsilon)}{\lambda \delta}\right) + C \nu \sqrt{\lambda d m \log\left(\frac{T R B m d K^2 \dn(\epsilon)}{\lambda \delta}\right)} \\
            &\leq \left(8 B^2 R^2 + \nu^2 \log\left(\frac{T}{\delta}\right)\right)\left(2 T \epsilon + 6 \log\left(\frac{\dn(\epsilon)}{\delta}\right)\right) + \Rijst(\ghat) \\
            &+C \nu^2 d m \log\left(\frac{T R B m d K^2 \dn(\epsilon)}{\lambda \delta}\right) + C \nu \sqrt{\lambda d m \log\left(\frac{T R B m d K^2 \dn(\epsilon)}{\lambda \delta}\right)}
        \end{align*}
        Setting $\lambda = 1$ and applying \Cref{lem:rijst} to bound $\Rijst(\ghat)$ then tells us that
        \begin{align*}
            \Qij(\ghat) &\leq \epsorac + C K^2 B R \sqrt{T m d \log\left(\frac{T B R K \dn(\epsilon)}{\delta}\right)} + K^2 \left(4 B^2 R^2 + \nu^2 \log\left(\frac{T}{\delta}\right)\right) \left(2 T \epsilon + 6 \log\left(\frac{\dn(\epsilon)}{\delta}\right)\right) \\
            &+ C \nu^2 d m \log\left(\frac{T R B m d K^2 \dn(\epsilon)}{ \delta}\right) + C \nu \sqrt{ d m \log\left(\frac{T R B m d K^2 \dn(\epsilon)}{ \delta}\right)} + 4 B R K^2T\epscor.
        \end{align*}
        We can now plug in our bound on $\dn(\epsilon)$ from \Cref{lem:disagreementcoversmall} which says that
        \begin{align*}
            \log\left(\dn(\cG, \epsilon)\right) \leq 2 K d \log\left(\frac{3 B K}{\epsilon}\right)
        \end{align*}
        and set $\epsilon = \frac 1T$ to conclude the proof.
    \end{proof}
    With the proof of \Cref{lem:qij_lowerbound_uniform} concluded, we now prove the technical lemmas:
\begin{claim}\label{lem:qij_covering}
    Let $\bDelta, \bDelta' \in \rr^{m \times (d + 1)}$ have Frobenius norm at most $R$.  Let $E, V$ be defined as in the proof of \Cref{lem:qij_lowerbound_fixed}.  Then it holds that
    \begin{align*}
        E(\bDelta)^2 - E(\bDelta')^2 &\leq \left(\sum_{t \in \Igij} \norm{\be_t}^2\right) 4 R B^2T \sqrt{\fronorm{\bDelta - \bDelta'}} \\
        V(\bDelta') - V(\bDelta) &\leq 4 R B^2 T \sqrt{\bDelta - \bDelta'}
    \end{align*}
\end{claim}
\begin{proof}
    We compute:
    \begin{align*}
        E(\bDelta)^2 - E(\bDelta')^2 &= E(\bDelta + \bDelta') \cdot E(\bDelta-  \bDelta') \\
        &= \left(\sum_{t \in \Igij} \inprod{\be_t}{ (\bDelta + \bDelta')\bx_t}\right) \cdot \left(\sum_{t \in \Igij} \inprod{\be_t}{ (\bDelta - \bDelta')\bx_t}\right) \\
        &\leq \left(\sum_{t \in \Igij} \norm{\be_t}^2\right) \sqrt{V(\bDelta + \bDelta')} \cdot \sqrt{V(\bDelta - \bDelta')} \\
        &\leq \left(\sum_{t \in \Igij} \norm{\be_t}^2\right) \sqrt{16 T R^2 B^2} \cdot \sqrt{V(\bDelta - \bDelta')},
    \end{align*}
    where the first equality follows from linearity, the first inequality follows by Cauchy-Schwartz, and the last follows by a \Cref{ass:boundedness}.  Similarly,
    \begin{align*}
        V(\bDelta) - V(\bDelta') &= \sum_{t \in \Igij} \left(\norm{\bDelta \barbx_t} + \norm{\bDelta' \barbx_t}\right) \left(\norm{\bDelta \barbx_t} - \norm{\bDelta' \barbx_t}\right) \\
        &\leq \sum_{t \in \Igij} (4 R B) \norm{(\bDelta - \bDelta')\barbx_t} \\
        &\leq 4 R B \sqrt{T} \cdot \sqrt{V(\bDelta - \bDelta')}.
    \end{align*}
    We now note that
    \begin{align*}
        V(\bDelta - \bDelta') &= \sum_{t \in \Igij} \norm{(\bDelta - \bDelta') \barbx_t}^2 \\
        &\leq \sum_{t \in \Igij} \fronorm{\bDelta - \bDelta'}^2 \norm{\barbx_t}^2 \\
        &\leq T B^2 \fronorm{\bDelta - \bDelta'}.
    \end{align*}
    The result follows.
\end{proof}

\begin{claim}\label{lem:sumetbound}
    Suppose that \Cref{ass:subgaussian} holds.  Then with probability at least $1 - \delta$, it holds that
    \begin{align*}
        \sum_{t = 1}^T \norm{\be_t}^2 \leq 2 \nu^2 \left(T m \log (8) + \log\left(\frac 1 \delta\right)\right).
    \end{align*}
\end{claim}
\begin{proof}
    Denote by $\bar\be = (\be_1, \dots, \be_T)$ the concatenation of all the $\be_t$ and note that $\bar\be$ is subGaussian by \Cref{ass:subgaussian}.  Let $\bar\bu$ denote a unit vector in $\rr^{T m}$ and note that for any fixed $\bar\bu \in \cS^{Tm  -1}$, it holds with probability at least $1 - \delta$ that $\inprod{\bar\be}{\bar\bu} \leq \sqrt{2 \nu^2 \log\left(\frac 1\delta\right)}$.  Constructing a covering at scale $\frac 14$ of $\cS^{Tm - 1}$ and applying the standard discretization argument, we see that
    \begin{align*}
        \sum_{t = 1}^T \norm{\be_t}^2  = \norm{\bar\be}^2 = \left(\sup_{\bar\bu \in S^{Tm - 1}} \inprod{\bar\be}{\bar\bu}\right)^2 \leq 2 \nu^2 \left(T m \log (8) + \log\left(\frac 1 \delta\right)\right)
    \end{align*}
    with probability at least $ 1 - \delta$.
\end{proof}

\begin{claim}\label{lem:discretizerij}
    Suppose that Assumptions \ref{ass:subgaussian} and \ref{ass:boundedness} hold.  Then, with probability at least $1 - \delta$, it holds for all $g, g'$,
    \begin{align*}
        \Rijst(g) - \Rijst(g') \leq \left(4 B^2 R^2 + \nu^2 \log\left(\frac T\delta\right)\right) \sum_{t = 1}^T \I\left[g(\barbx_t) \neq g'(\barbx_t)\right].
    \end{align*}
\end{claim}
\begin{proof}
    We compute:
    \begin{align*}
        \Rijst(g) - \Rijst(g') &= \sum_{t \in \Igij} \norm{\paramhat_i \barbx_t - \by_t}^2 - \norm{\be_t}^2 - \sum_{t \in I_{ij}(g')} \norm{\paramhat_i \barbx_t - \by_t}^2 - \norm{\be_t}^2.
    \end{align*}
    A union bound and \Cref{ass:subgaussian} tell us that with probability at least $1 - \delta$,
        \begin{equation*}
            \max_{1 \leq t \leq T} \norm{\be_t} \leq \nu \sqrt{\log\left(\frac T\delta\right)}.
        \end{equation*}
        As $\fronorm{\paramhat_i - \paramstar_j} \leq 2 R$, the result follows.
\end{proof}

\subsection{Proof of Theorem \ref{thm:parameterrecovery}}

We are now ready to put everything together and conclude the proof of \Cref{thm:parameterrecovery}.  We first need to prove the following bound by combining results from \Cref{subsec:smallball,subsec:covariance}:
\begin{lemma}\label{lem:qijupperbound}
    For any fixed $\xi < 1$, Define the following notation:
    \begin{align*}
        \Xi_1 &= C\frac{B^8}{\sigdir^8 \xi^8}\left(\log\left(\frac{2 T}{\delta} + \frac d2 \log\left(\frac{B}{\sigdir \xi}\right) + \log\left(K^2 \dn(\epsilon)\right)\right)\right) \\
        \Xi_2 &= C \frac{B^2}{\sigdir^2 \xi^2} \left(2 T \epsilon + 6 \log\left(\frac{\dn(\epsilon)}{\delta}\right)\right) \\
        \Xi_3 &= 4K^2 T \xi + 12 \log\left(\frac 1\delta\right) + \max\left(\Xi_1, \Xi_2\right)
    \end{align*}
    With probability at least $1 - \delta$, it holds that for all $1 \leq i,j \leq K$ such that
    \begin{align}
        \abs{I_{ij}(\ghat)} \geq \Xi_3. \label{eq:minIgij}
    \end{align}
    we have
    \begin{align}
        \fronorm{\paramhat_i - \paramstar_j}^2 \leq C \frac{B^2}{\sigdir^2 \xi^2 \abs{I_{ij}(\ghat)}} \cdot \Qij(\ghat).
    \end{align}
\end{lemma}
\begin{proof}
    We begin by applying a Chernoff bound (\Cref{lem:chernoff}) to bound the number of times $t$ such that $\Ztij(\ghat) = 1$ and $\pp_t(\Ztij(\ghat) = 1) \leq \xi$.  In particular, we have
    \begin{align*}
        \pp&\left(\sum_{t = 1}^T \I\left[\text{there exist } i,j \text{ such that } \Ztij(\ghat) = 1 \text{ and } \pp_t(\Ztij(\ghat)) \leq \xi\right] \geq 2 T K^2 \xi + 6 \log\left(\frac 1\delta\right)\right) \\
        &=\pp\left(\sum_{t = 1}^T \sum_{1 \leq i,j \leq K} \I\left[\Ztij(\ghat) = 1 \text{ and } \pp_t(\Ztij(\ghat)) \leq \xi\right] \geq 2 TK^2 \xi + 6 \log\left(\frac 1\delta\right)\right) \\
        &\leq \delta
    \end{align*} 
    where the last step follows by \Cref{lem:chernoff}.  Now, denote
    \begin{align*}
        \Itilgij &= I_{ij}(\ghat) \cap \left\{ t | \pp_t(\Ztij(\ghat)) \geq \xi \right\} \\
        \sigtilij(g) &= \sum_{t \in \Itilgij} \barbx_t \barbx_t^T
    \end{align*}
    and note that $\sigij(\ghat) \succeq \sigtilij(\ghat)$ by construction.  Furthermore, by the Chernoff bound above, with probability at least $1 - \delta$, if
    \begin{align*}
        \abs{I_{ij}(\ghat)} \geq 4 K^2 T \xi + 12 \log\left(\frac 1\delta\right) + \max\left(\Xi_1, \Xi_2\right),
    \end{align*}
    it holds that
    \begin{align*}
        \abs{\Itilgij} \geq \max\left(\frac 12 \abs{I_{ij}(\ghat)}, \max\left(\Xi_1, \Xi_2\right)\right).
    \end{align*}
    Now, applying \Cref{lem:uniform_smallball,lem:smallball_final}, we see that with probability at least $1 - \delta$,
    \begin{align*}
        \sigij(\ghat) \succeq \sigtilij(\ghat) \succeq c \frac{\sigdir^2 \xi^2}{B^2} \abs{\Itilgij} I \succeq c \frac{\sigdir^2 \xi^2}{B^2} \abs{I_{ij}(\ghat)}.
    \end{align*}
    We now compute:
    \begin{align*}
        \Qij(\ghat) &= \sum_{t \in I_{ij}(\ghat)}  \norm{\left(\paramhat_i - \paramstar_j\right) \barbx_t}^2 \\
        &= \sum_{t \in I_{ij}(\ghat)} \tr\left(\barbx_t^T \left(\paramhat_i - \paramstar_j\right)^T \left(\paramhat_i - \paramstar_j\right) \barbx_t \right) \\
        &= \sum_{t \in I_{ij}(\ghat)} \tr\left(\left(\paramhat_i - \paramstar_j\right)^T \left(\paramhat_i - \paramstar_j\right) \barbx_t \barbx_t^T\right) \\
        &\geq \fronorm{\paramhat_i - \paramstar_j}^2 \cdot \lambda_{min}\left(\sum_{t \in I_{ij}(\ghat)} \barbx_t \barbx_t^T\right) \\
        &= \fronorm{\paramhat_i - \paramstar_j}^2 \cdot \lambda_{min}\left(\sigij(\ghat)\right)
    \end{align*}
    where we have denoted by $\lambda_{min}(\cdot)$ the minimal eigenvalue of a symmetric matrix.  Thus, under our assumptions, with probability at least $1 - \delta$, it holds that
    \begin{align*}
        \fronorm{\paramhat_i - \paramstar_j} \leq C \frac{B^2}{\sigdir^2 \xi^2 \abs{I_{ij}(\ghat)}} \cdot \Qij(\ghat)
    \end{align*}
    for all $i,j$ satisfying \eqref{eq:minIgij}.  The result follows.
\end{proof}
We are now ready to combine \Cref{lem:qijupperbound} with the results of \Cref{subsec:empiricalerrorupper} to conclude the proof.
\begin{proof}[Proof of Theorem \ref{thm:parameterrecovery}]
    By \Cref{lem:qij_lowerbound_uniform}, it holds with probability at least $1 - \delta$ that 
    \begin{align*}
        \Qij(\ghat) &\leq \epsorac + C K^2 B R d \sqrt{T m K \log\left(\frac{T B R K}{\delta}\right)} + C K^3 d \left(4 B^2 R^2 + \nu^2 \log\left(\frac{T}{\delta}\right)\right) \log\left(\frac{B K T}{\delta}\right) \\
            &+ C \nu^2 d^2 K m \log\left(\frac{T R B m d K}{ \delta}\right) + C \nu  d\sqrt{ K m \log\left(\frac{T R B m d K }{ \delta}\right)}+ 4 B R K^2\abs{\Igij}\epscor \\
            &\leq \epsorac + 1 + C K^3 B^2 R^2 d^2 m \nu^2 \sqrt{T} \log\left(\frac{T R B m d K}{\delta}\right) + 4 B R K^2 \abs{\Igij}\epscor.
    \end{align*}
    By \Cref{lem:qijupperbound}, it holds with probability at least $1 - \delta$ that for all $i,j$ satisfying $\abs{I_{ij}(\ghat)} \geq \Xi_3$, we have
    \begin{align*}
        \fronorm{\paramhat_i - \paramstar_j}^2 \leq C \frac{B^2}{\sigdir^2 \xi^2 \abs{I_{ij}(\ghat)}} \Qij(\ghat).
    \end{align*}
    It now suffices to note that, by \Cref{lem:disagreementcoversmall}, it holds that
    \begin{align*}
        \log\left(\dn(\cG, \epsilon)\right) \leq K(d + 1) \log\left(\frac{3 B K}{\epsilon}\right)
    \end{align*}
    and so, taking $\epsilon = \frac 1T$,
    \begin{align*}
        \Xi_3 \leq C K^2 T \xi + C \frac{B^8 K d}{\sigdir^8 \xi^8} \log\left(\frac{B K T}{\sigdir \xi \delta}\right).
    \end{align*}
    The result follows.
\end{proof}


\newcommand{\Term}{\mathrm{Term}}
\newcommand{\bwktil}{\widetilde{\bw}_{1:K}}

\section{Modifying the Classifier and Mode Prediction}\label{sec:onlinelearning}
\begin{algorithm}[!t]
    \begin{algorithmic}[1]
    \State{}\textbf{Initialize } Classifier $\ghat$, new parameters $\left(\paramhat_{1, i}\right)_{1 \leq i \leq K}$, old parameters $\left(\paramhat_{0,i}\right)_{1 \leq i \leq K}$, threshold $A > 0$, gap $\Delsep > 0$
    \For{$i,j = 1, 2, \dots, K$} \qquad\qquad (\algcomment{Combine large clusters with similar $\paramhat_{1,i}$})
        \If{$i = j$ or $\min(I_i(\ghat), I_j(\ghat)) < A$} \qquad\qquad (\algcomment{Continue if cluster is too small})
            \State{} \textbf{Continue}
        \EndIf
        \If{ $\norm{\paramhat_{1,i} - \paramhat_{1,j}} < \Delsep$} \qquad\qquad (\algcomment{Combine Cluster if parameters closer than gap})
            \State{} $\ghat \gets (j \mapsto i) \circ \ghat$
        \EndIf
    \EndFor
    \State{} $I \gets [K]$
    \State{} Empty permutation $\tilde \pi:[K] \to [K]$
    \State{} Reorder $I$ so that if $i < i'$ then $\abs{I_i(\ghat)} > \abs{I_{i'}(\ghat)}$
    \For{$i = 1, 2, \dots, K$}
    \If{$\abs{I_i(\ghat)} > A$} \qquad \qquad (\algcomment{Check if cluster is large enough})
    \State{} $j \gets \argmin_{j' \in I} \norm{\paramhat_{1, i} - \paramhat_{0,j'}}$
    \State{} $\tilde \pi(i) = j$
    \State{} $I \gets I \setminus \{j\}$
    \EndIf
    \EndFor
    \State{} \textbf{Return } $\tilde \pi \circ \ghat$
    \end{algorithmic}
      \caption{Combine and Permute Labels}
      \label{alg:reorder}
\end{algorithm}
In this section, analyze $\ogdupdate$ (\Cref{alg:ogdupdate}), the algorithm that modifies the classifer $\ghat_\tau$ produced by $ermoracle$ after epoch into a stabilized classifer $\gtil_\tau$ suitable for online prediction.  The problem with the former classifer $\ghat_\tau$ is that while it performs well on the past examples by construction, directional smoothness is not strong enough to imply generalization in the sense that $\ghat_\tau$ will continue to perform well on epoch $\tau + 1$.  

\paragraph{Notation.} We begin our analysis of $\ogdupdate$  by defining some notation.  For any $1 \leq t \leq T$, let $\tau(t)$ denote the epoch containing $t$, i.e., $\tau(t) = \max \left\{ \tau' | \tau' E \leq t \right\}$. Further, recall the concatenated parameter notation  $\bwk = (\bw_1, \dots, \bw_K)$ for $\bw_i \in \cB^{d-1}$.

For a given epoch $\tau$, we let $\left\{ \paramhat_{\tau,i} | i \in [K] \right\}, \ghat_{\tau}$ denote the output of $\ermoracle(\barbx_{1:\tau E}, \by_{1: \tau E})$.  For any $g \in \G$ and $i, j \in [K]$, we denote
\begin{equation}
    I_{ij;\tau}(g) = \left\{ 1 \leq t \leq \tau E | g(\barbx_t) = i \text{ and } \gstar(\barbx_t) = j \right\}.
\end{equation}
Finally, for a fixed epoch $\tau$, if $\ermoracle$ has returned parameters $\paramhat_{\tau,i}$, define
\begin{equation}\label{eq:pitau}
    \pi_\tau(i) = \argmin_{1 \leq j \leq K} \fronorm{\paramhat_{\tau,i} - \paramstar_j}
\end{equation}
the function that takes a label $i$ according to $\ghat_\tau$ and sends it to the closest label according to $\gstar$ as measured by difference in parameter matrices. The notation ``$\pi$'' signifies that $\pi_{\tau}$ represents a permutation when all the estimates $\paramhat_{\tau,i}$ are sufficiently accurate. 

We conduct the analysis under the following condition, which informally states that for all sufficiently large clusters $I_{ij}$ considered in \Cref{alg:reorder} are sufficiently large, the associate parameters, \Cref{ass:gap}
\begin{condition}\label{condition:large_clusters}
    We say that $\Xi \ge 0 $ is a $\delta$-valid clusterability bound if it satisfies the following property. If \Cref{alg:reorder} is run with $A = 2 K \Xi$ then with probability at least $1 - \delta$, for all $1 \leq i \leq K$ such that
    \begin{equation}
       I_{i;\tau}(\ghat_\tau) = \sum_{1 \leq j \leq K} \abs{I_{ij;\tau}(\ghat_\tau)} > 2 K \Xi
    \end{equation}
    the following hold:
    \begin{enumerate}
        \item There exists a unique $1 \leq j \leq K$ such that $\abs{I_{ij;\tau}(\ghat_\tau)} > \Xi$ and for that $j$, $\abs{I_{ij;\tau}(\ghat_\tau)} > K\Xi$.
        \item For the $j$ given by the previous statement, $\norm{\paramhat_i^\tau - \paramstar_j} \leq \frac \Delsep 4$.
        \item For the $j$ in the previous statements, it holds that in epoch $\tau+1$, the classifer $\ghat_{\tau+1}$ after the reordering step in \Cref{alg:ogdupdate}, the estimated parameter satisfies  $\abs{I_{ij;\tau}(\ghat_{\tau+1})} > \Xi$ and in particular, $\norm{\paramhat_i^{\tau + 1} - \paramstar_j} \leq \frac \Delsep 4$.
    \end{enumerate}
\end{condition}

Recall that the the pseudocode for $\ogdupdate$ is in \Cref{alg:ogdupdate}.  In words, the algorithm runs projected online gradient descent on $\elltilgamma$.  We have the following result.
\begin{theorem}\label{thm:onlinelearning}
    Suppose that we run \Cref{alg:master} in the setting described in \Cref{sec:setting}. If we set $A = 2 K \Xi$, then with probability at least $1 - \delta$, it holds that
    \begin{align}
        \sum_{t = 1}^T \I\left[\pi_\tau (\gtil_{\tau(t)}(\barbx_t)) \neq \gstar(\barbx)\right] &\leq  \frac{2 B E T \eta}{\gamma} + 3\left(1 + \frac 1\gamma\right) (K E + 2 K^2 \Xi) + \frac{4 K}{\eta} + \frac{\eta T}{\gamma^2} + \frac{T \gamma}{\sigdir} + \sqrt{T \log\left(\frac 1\delta\right)}
    \end{align}
    where $\Xi$ is a parameter depending on the gap and the problem, defined in \Cref{lem:largeclustersremain}.
\end{theorem}

\paragraph{Proof Strategy.} One challenge is that $\gtil_{\tau(t)}$ is that it is updated at the start of epoch $\tau+1$, and is trained using labels corresponding to the permutation $\pi_{\tau}$. Therefore, we decompose to the error indicator into the case where $\pi_{\tau+1}(\gtil_{\tau(t)}(\barbx_t)) = \pi_{\tau}(\gtil_{\tau(t)}(\barbx_t))$, so this difference is immaterial, and into the cases where $\pi_{\tau+1}$ and $\pi_{\tau}$ differ.
\begin{equation}\label{eq:onlinelearning_breakdown}
    \I\left[\pi_\tau (\gtil_{\tau(t)}(\barbx_t)) \neq \gstar(\barbx)\right] \leq \underbrace{\I\left[\pi_{\tau + 1} (\gtil_{\tau(t)}(\barbx_t)) \neq \gstar(\barbx_t)\right]}_{\text{(look-ahead classification error)}} + \underbrace{\I\left[\pi_\tau (\gtil_{\tau(t)}(\barbx_t)) \neq \pi_{\tau+ 1}(\gtil_{\tau(t)}(\barbx_t))\right]}_{\text{(permutation disagreement error)}}
\end{equation}
We call the former term the ``look-ahead classification error'', because it applies the permutation from the subsequent epoch $\tau+1$; the name for the second term is self-explanatory. Our online update \ogdupdate (\Cref{alg:ogdupdate}) controls the cumulative sum of the look-ahead classification error ( see \Cref{thm:classificationbound} in the following section), while our labeling protocol (\Cref{alg:reorder}) bounds the permutation disagreement error (see \Cref{lem:taucontinuity} in the \Cref{sec:permutation_term}).

\subsection{Look-ahead Classification Error}\label{app:lookahead}
In this section we prove that the first term of \eqref{eq:onlinelearning_breakdown} is small with high probability, which is the content of the following result:
\begin{theorem}\label{thm:classificationbound} With probability at least $1 - \delta$, the look-ahead classification error is at most
    \begin{align}
        \sum_{t = 1}^T \I\left[\pi_{\tau(t)  + 1}\left(\gtil_\tau(\barbx_t)\right) \neq g^\star(\barbx_t)\right] &\leq  \frac{2 B E T \eta}{\gamma} + 2\left(1 + \gamma^{-1}\right) (K E + 2 K^2 \Xi) + \frac{4 K}{\eta} + \frac{\eta T}{\gamma^2} + \frac{K^2 T \gamma}{\sigdir} + \sqrt{T \log\left(\tfrac 1\delta\right)}.
    \end{align}
\end{theorem}
\begin{proof}
    In addition to the notation $\widetilde{\ell}_{\gamma,t,\widehat{g}}$ defined in \Cref{eq:ltilgamma}, we introduce the following notation:
    \begin{align}
        i_t^\star = g^\star(\barbx_t) && \widehat{i}_t = \argmax_{1 \leq i \leq K} \inprod{\bw_i^{\tau(t)}}{\barbx_t} = \gtil_{\tau(t)}(\barbx_t) && \bar i_t = \widehat{g}_{\tau(t)+1}(\barbx_t)
    \end{align}
    or, in words, $i_t^\star$ is the groundtruth correct label, $\widehat{i}_t$ is the class predicted by the current epoch's linear predictors, and $\bar i_t$ is the class predicted by the ERM trained with the current epoch's data included.  Additionally, let 
    \begin{equation}
        \bwktil^t = \Pi_{({\cB}^{d-1})^{\times K}}\left(\bwktil^{t-1} - \eta \nabla \widetilde{\ell}_{\gamma, t, \widehat{g}_{k(t)+1}}(\bwktil^{t-1}) \right)
    \end{equation}
    i.e., the predicted weight if we were able to apply gradient descent with the labels predicted by the ERM trained on the current epoch's data and
    \begin{equation}
        \widetilde{i}_t = \argmax_{1 \leq i \leq k} \inprod{\widetilde\bw_i^t}{\barbx_t},
    \end{equation}
    the class predicted by $\widetilde{\bwk}$.  We also consider the following ``losses'' that we will use in the analysis:
    \begin{align}
        \widehat{\ell}_t(\bwk) = \I\left[\widehat{i}_t \neq \bar i_t \right] && \ell_t^\star(\bwk) = \I\left[\pi_{\tau(t) + 1}(\widehat{i}_t) \neq i_t^\star\right]
    \end{align}
    or, in words, $\widehat{\ell}_t$ is the event that our prediction is not equal to that of $\widehat{g}_{k(t) + 1}$ and $\ell_t^\star$ is the event that our prediction is not equal to the groundtruth.  

    \paragraph{Mistake Decomposition.} We now compute:
    \begin{align}
        \sum_{t = 1}^T \ell_t^\star(\bwk^{\tau(t)}) &\leq \underbrace{\left(\sum_{t = 1}^T \I\left[\pi_{\tau + 1}(\ghat_{\tau + 1}(\barbx_t)) \neq \gstar(\barbx_t)\right]\right)}_{\Term_1} + \left(\sum_{t = 1}^T \I\left[\gtil_\tau(\barbx_t) \neq \ghat_{\tau + 1}(\barbx_t)\right]\right)
    \end{align}
    \Cref{lem:wronglabelsbound} in \Cref{sec:permutation_term} controls $\Term_1$.  To control the other term, we note that for any $t,\gamma,g,\bwk$,
    \begin{align*}
    \elltilgamma[g](\bwk) &= \max\left(0, \max_{j \neq g(\barbx_t)} \left(1 - \frac{\inprod{\bw_{g(\barbx_t)} - \bw_j}{\barbx_t}}{\gamma} \right)\right) \ge \I[g_{\bwk}(\barbx_t) \ne g(\barbx_t)],
    \end{align*}
    where $g_{\bwk}(\barbx_t)$ is the classifier induced by the paramters $\bwk$. As $\gtil_{\tau}(\barbx_t)$ is the classifier induced by $\bwk^\tau$, we have
    \begin{align*}
        \sum_{t = 1}^T \I\left[\gtil_\tau(\barbx_t) \neq \ghat_{\tau + 1}(\barbx_t)\right] &\leq \sum_{t =1}^T \widetilde{\ell}_{\gamma, t, \ghat_{\tau + 1}}(\bwk^{\tau}) \\
        &\leq \underbrace{\left(\sum_{t = 1}^T \widetilde{\ell}_{\gamma, t, \ghat_{\tau + 1}}(\bwk^{\tau}) - \widetilde{\ell}_{\gamma, t, \ghat_{\tau + 1}}(\bwktil^{t})\right)}_{\Term_2} + \left(\sum_{t= 1}^T \widetilde{\ell}_{\gamma, t, \ghat_{\tau + 1}}(\bwktil^{t})\right)
    \end{align*}
    Contining, we have
    \begin{align*}
        \sum_{t = 1}^T \ell_t^\star(\bwk^{\tau(t)}) &\leq  \Term_1 + \Term_2 + \sum_{t= 1}^T \widetilde{\ell}_{\gamma, t, \ghat_{\tau + 1}}(\bwktil^{t}) \\
        &\leq \Term_1 + \Term_2 + \underbrace{\sup_{\bwk}\left(\sum_{t= 1}^T \widetilde{\ell}_{\gamma, t, \ghat_{\tau + 1}}(\bwktil^{t}) - \widetilde{\ell}_{\gamma, t, \ghat_{\tau + 1}}(\bwk)\right)}_{\Term_3} \\
        &+ \underbrace{\inf_{\bwk} \sum_{t = 1}^ T \widetilde{\ell}_{\gamma, t, \ghat_{\tau + 1}}(\bwk)}_{\Term_4}
    \end{align*}

    \paragraph{Bounding the ``delay'' penalty: $\Term_2$.} $\Term_2$ corresponds to the error we may suffer from using delayed gradient updates. We now observe that
    \begin{align}\label{eq:hingelosslip}
        \norm{\nabla \widetilde{\ell}_{\gamma, t, \ghat_{\tau(t) + 1}}} \leq \frac{\norm{\barbx_t}}{\gamma} \leq \frac{B}{\gamma}
    \end{align}
    and that projection onto a convex body is a contraction.  Furthermore, the gradient update in \Cref{alg:ogdupdate}, Line \ref{line:update} only affects at most two distinct $i,j$ in the coordinates of $\bwk$ per update.  Thus it holds that, for all $t$,
    \begin{align*}
        \norm{{\bwktil}^t - \bwk^{\tau(t)}} \leq 2 E \eta.
    \end{align*}
    Applying \eqref{eq:hingelosslip} tells us that
    \begin{equation}
        \Term_2 := \sum_{t = 1}^T \widetilde{\ell}_{\gamma, t, \ghat_{\tau + 1}}(\bwk^{\tau}) - \widetilde{\ell}_{\gamma, t, \ghat_{\tau + 1}}(\bwktil^{t}) \leq \frac{2 B ET \eta}{\gamma}.
    \end{equation}
    \paragraph{Bounding the regret term: $\Term_3$. } We see that $\Term_3$  is just the regret for Online Gradient Descent for losses with gradients bounded in norm by $\frac 1\gamma$ on the $K$-fold product of unit balls, having diameter $2 \sqrt{K}$.  Thus it holds by classical results (c.f. \citet[Theorem 3.1]{hazan2016introduction}) that
    \begin{equation}
        \Term_3  =  \sum_{t = 1}^T \widetilde{\ell}_{\gamma, t, \widehat{g}_{k(t)}}(\widetilde{\bw}_t) - \inf_{\bwk} \widetilde{\ell}_{\gamma, t, \widehat{g}_{k(t)}}(\bwk) \leq \frac{4 K}{\eta} + \frac{\eta T}{\gamma^2}.
    \end{equation}

    \paragraph{Bounding the comparator: $\Term_4$.} To bound $\Term_4$, we aim to move from the losses with margin $\widetilde{\ell}_{\gamma, t, \ghat_{\tau + 1}}$ back to sum measure of comulative classification error. The central object in this analysis is the following ``ambiguous set'', defined for each parameter $\bwk$: 
    \begin{equation}
        D_\gamma(\bwk) = \left\{ \barbx : \abs{\inprod{\bw_{i,1:d} - \bw_{j,1:d}}{\bx} - w_{i,d+1} + w_{j,d+1}} \leq \gamma \text{ for some } 1 \leq i \leq K\right\},
    \end{equation}
    where we denote the first $d$ coordinates of $\bw_i$ by $\bw_{i,1:d}$ and its last coordinate by the scalar $w_{i,d+1}$.Then, by directional smoothness and a union bound, the following is true for any fixed $\bwk$:
    \begin{align*}
    \Pr(\barbx_t  \in D_\gamma(\bwk) \mid \filt_t) \le K^2 \sup_{i,j}\Pr(\abs{\inprod{\bw_{i,1:d} - \bw_{j,1:d}}{\bx_t} - w_{i,d+1} + w_{j,d+1}} \leq \gamma  \mid \filt_t) \le \frac{K^2 \gamma}{\sigdir}. 
    \end{align*}   
    Let $\bwk^\star$ be the parameter associated with $g^\star$. Then a consequence of the above display and Azuma's inequality is that most $\barbx_t$'s fall outside of $D_{\gamma}(\bwk^\star)$:
    \begin{align}
        \sum_{t = 1}^T \I\left[\barbx_t \in D_\gamma(\bwk^\star)\right] &\leq \sum_{t = 1}^T \pp_t\left(\barbx_t \in D_\gamma(\bwk^\star)\right) + \sqrt{T \log\left(\frac 1\delta\right)} \nonumber\\
        &\leq \frac{K^2 T\gamma}{\sigdir} + \sqrt{T \log\left(\frac 1\delta\right)}. \label{eq:Azuma_Ambiguity_Set}
    \end{align}

    We compute:
    \begin{align*}
        \inf_{\bwk} \sum_{t = 1}^ T \widetilde{\ell}_{\gamma, t, \ghat_{\tau + 1}}(\bwk) &\leq \inf_{\bwk} \sum_{t = 1}^T \I\left[\barbx_t \in D_\gamma(\bwk)\right] + \left(1 + \frac 2\gamma\right) \sum_{t = 1}^T \widehat{\ell}_t(\bwk) \\
        &\leq \inf_{\bwk} \sum_{t = 1}^T \I\left[\barbx_t \in D_\gamma(\bwk)\right] + \left(1 + \frac 2\gamma\right) \sum_{t = 1}^T \ell_t^\star(\bwk)  + \I\left[\pi_{\tau(t) + 1}(\bar i_t) \neq i_t^\star\right]\\
        &\leq \left(1 + \frac 2\gamma\right)\sum_{t = 1}^T \I\left[\pi_{\tau(t) + 1}(\bar i_t) \neq i_t^\star\right] + \inf_{\bwk} \sum_{t = 1}^T \I\left[\barbx_t \in D_\gamma(\bwk)\right] + \left(1 + \frac 2\gamma\right) \ell_t^\star(\bwk) \\
        &\leq \left(1 + \frac{2}{\gamma}\right) \sum_{t = 1}^T \I\left[\pi_{\tau(t) + 1}(\bar i_t) \neq i_t^\star\right] + \sum_{t = 1}^T \I\left[\barbx_t \in D_\gamma(\bwk^\star)\right]\\
        &\leq \left(1 + \frac{2}{\gamma}\right) \underbrace{\sum_{t = 1}^T \I\left[\pi_{\tau(t) + 1}(\bar i_t) \neq i_t^\star\right]}_{=\Term_1} + \frac{K^2 T\gamma}{\sigdir} + \sqrt{T \log\left(\frac 1\delta\right)}.
    \end{align*}
    where the first inequality follows by \Cref{lem:indicatorsoftmargin}, the second inequality is trivial, the third inequality follows because the final term does not depend on $\bwk$, the fourth inequality follows from the realizability assumption, i.e., that $\ell_t^\star(\bwk^\star) = 0$ for all $t$, and the last from \Cref{eq:Azuma_Ambiguity_Set} (recalling the definition of $\Term_1$) .

   Thus it holds with probability at least $1- \delta$ that
    \begin{align*}
    \Term_4 \le \left(1 + \frac 2\gamma\right) \Term_1  + \frac{K^2 T\gamma}{\sigdir} + \sqrt{T \log\left(\frac 1\delta\right)}.
    \end{align*}

    \paragraph{Concluding the proof.}

     To conclude, we see that
    \begin{align}
        \sum_{t = 1}^T \ell_t^\star(\bwk^{\tau(t)}) &\leq \Term_1 + \Term_2 + \Term_3 + \Term_4 \\
        &\le \left(1 + \frac 2\gamma\right)\Term_1 +  \frac{2 B E T \eta}{\gamma}  + \frac{4 K}{\eta} + \frac{\eta T}{\gamma^2} + \frac{K^2 T \gamma}{\sigdir} + \sqrt{T \log\left(\frac 1\delta\right)}.
    \end{align}

     By \Cref{lem:wronglabelsbound} it holds that
    \begin{align}
        \Term_1 = \sum_{t = 1}^T \I\left[\pi_{\tau(t) + 1}(\bar i_t) \neq i_t^\star\right]  := \sum_{t = 1}^T \I\left[\pi_{\tau(t) + 1}(\ghat_{\tau(t) + 1}(\barbx_t)) \neq \gstar(\barbx_t)\right] \leq K E + 2K^2 \Xi,
    \end{align}
    where $\Xi$ is defined in \Cref{lem:largeclustersremain}. This concludes the proof.
\end{proof}

\begin{lemma}[Bound indicator by soft margin] \label{lem:indicatorsoftmargin}
    For any $\gamma, t, \widehat{g}$, it holds that if each component of $\bw$ has unit norm, then
    \begin{equation}
        \widetilde{\ell}_{\gamma,t,\widehat{g}}(\bw) \leq \I\left[\barbx_t \in D_\gamma(\bwk)\right] + \left(1 + \frac 2\gamma \right)\widehat{\ell}_t(\bwk).
    \end{equation}
\end{lemma}
\begin{proof}
    We prove this by casework.  Suppose that $\widehat{\ell}_t(\bwk) = 0$.  Then it holds that
    \begin{equation}
        \widehat{g}(\barbx_t) = \argmax_{1 \leq j \leq K} \inprod{\bw_j}{\barbx_t}
    \end{equation}
    and in particular for all $j \neq \widehat{g}(\barbx_t)$, it holds that $\inprod{\bw_{\widehat{g}(\barbx_t)}- \bw_j}{\barbx_t} \geq 0$.  If $\barbx_t \not\in D_\gamma(\bwk)$ then it holds by construction that
    \begin{equation}
        \max_{j \neq \widehat{g}(\barbx_t)} 1 - \frac{\inprod{\bw_{\widehat{g}(\barbx_t)} - \bw_j}{\barbx_t}}{\gamma} \leq 0
    \end{equation}
    and the conclusion clearly holds.  If it holds that
    \begin{equation}
        \max_{j \neq \widehat{g}(\barbx_t)} 1 - \frac{\inprod{\bw_{\widehat{g}(\bx_t)} - \bw_j}{\barbx_t}}{\gamma} > 0
    \end{equation}
    then either there is some $j$ such that $\inprod{\bw_{\widehat{g}(\bx_t) - \bw_j}}{\barbx_t} < 0$ or there is some $j$ such that $\inprod{\bw_{\widehat{g}(\barbx_t)} - \bw_j}{\barbx_t} < \gamma$ and so $\barbx_t \in D_\gamma(\bwk)$.  We cannot have the former by the assumption that $\widehat{\ell}_t(\bwk) = 0$ so the latter holds and the inequality follows in this case.

    Now suppose that $\widehat{\ell}_t(\bwk) = 1$.  Then we see that as $\norm{\bw_i} \leq 1$ for all $i$, $\widetilde{\ell}_{\gamma, t, \widehat{g}} \leq \left(1 + \frac 2\gamma\right)$ uniformly and the result holds.
\end{proof}

\subsection{Bounding Permutation Disagreement and $\Term_1$ }\label{sec:permutation_term}
In this section, we provide a bound on the permutation disagreement error - the second term of \eqref{eq:onlinelearning_breakdown} - as well as on $\Term_1$ from the section above .  We begin by notion that \Cref{alg:reorder} ensures that large clusters remain large accross epochs, a statement formalized by the following result:
\begin{lemma}\label{lem:largeclustersremain}
    Define the following terms:
    \begin{align}
        \Xi_1 &= C K^2 T \xi + C \frac{B^8 K d}{\sigdir^8 \xi^8} \log\left(\frac{B K T}{\sigdir \xi \delta}\right)\\
        \Xi_2 & = C \frac{B^2}{\sigdir^2 \xi^2 \Delsep^2} \left(\epsorac + 1 + K^3 B^2 R^2 d^2 m \nu^2 \sqrt{T} \log\left(\frac{T R B m d K}{\delta}\right)\right) + \sum_{t = 1}^T \norm{\bdelta_t}.
    \end{align}
    Then $\Xi := \max(\Xi_1, \Xi_2)$ is a $\delta$-valid clusterability parameter  in the sense of \Cref{condition:large_clusters}.
    
\end{lemma}
\begin{proof}
    By the pigeonhole principle, if $I_i(\ghat_\tau) > 2 K \Xi$ there must be at least one $j$ such that $\abs{I_{ij;\tau}(\ghat_\tau)} >  \Xi$ and thus by \Cref{thm:parameterrecovery}, it holds that $\norm{\paramhat_i^\tau - \paramstar_j} \leq \frac{\Delsep}{4}$ and thus the second statement holds.  If there is another $j'$ such that $\abs{I_{ij';\tau}(\ghat_\tau)} > \Xi$ then it would also hold that $\norm{\paramhat_i^\tau - \paramstar_{j'}} \leq \frac{\Delsep}{4}$ and the triangle inequality would then imply that $\norm{\paramstar_j - \paramstar_{j'}} \leq \frac \Delsep 2 < \Delsep$, which ensures that $j = j'$ by \Cref{ass:gap}.  Applying pigeonhole again then shows that
    \begin{equation}
        \abs{I_{ij;\tau}(\ghat_\tau)} = I_{i;\tau}(\ghat_\tau) - \sum_{j' \neq j} \abs{I_{ij';\tau}(\ghat_\tau)} > 2 K \Xi - (K - 1) \Xi = (K  + 1) \Xi.
    \end{equation}
    Thus the first statement holds.  For the last statement, note that as there are at least $K \Xi$ times $t < \tau E$ such that $\gstar(\barbx_t) = j$, there are at least $\Xi$ times $t < (\tau + 1) E$ such that $t \in I_{i'j;\tau}(\ghat_{\tau + 1})$.  This implies, again by \Cref{thm:parameterrecovery}, that $\norm{\paramhat_{i'}^{\tau + 1} - \paramstar_j} \leq \frac{\Delsep}{4}$ and thus $\norm{\paramhat_{i'}^{\tau + 1} - \paramhat_i^\tau} \leq \frac \Delsep 2$.  If $i' \neq i$ after running \Cref{alg:reorder} then there must be an $i''$ such that $I_{i'';\tau}(\ghat_{\tau + 1}) > K\Xi$ and $\norm{\paramhat_{i''}^{\tau + 1} - \paramhat_{i}^{\tau}} \leq \frac{\Delsep}{2}$.  But if this is the case then the triangle inequality implies that $\norm{\paramhat_{i''}^{\tau + 1} - \paramhat_{i'}^{\tau  +1}} < \Delsep$ and as they both are sufficiently large, there were merged by the first half of \Cref{alg:reorder}, implying that $i'' = i'$ and, in turn, that $i' = i$, by the the second half of \Cref{alg:reorder}.  Thus the first half of the third statement holds.  To conclude, simply apply \Cref{thm:parameterrecovery}.
\end{proof}
We are now ready to prove the main bound:
\begin{lemma}\label{lem:taucontinuity}
    With probability at least $1 - \delta$, it holds that
    \begin{align}
        \sum_{t = 1}^T \I\left[\pi_\tau (\gtil_{\tau(t)}(\barbx_t)) \neq \pi_{\tau+ 1}(\gtil_{\tau(t)}(\barbx))\right] \leq K E + 2 K^2 \Xi
    \end{align}
\end{lemma}
\begin{proof}
    We begin by fixing some $1 \leq i \leq K$ and bounding the number of epochs $\tau$ such that $\pi_{\tau}(i) \neq \pi_{\tau + 1}(i)$.  We will restrict our focus to the probability $1 - \delta$ event such that the conclusion of \Cref{lem:largeclustersremain} holds.  Suppose that there is an $i$ such that $I_{i;\tau}(\ghat_\tau) > 2 K \Xi$ and $\pi_{\tau}(i) = j \neq j' =  \pi_{\tau + 1}(i)$.  Then by \Cref{ass:gap} and the triangle inequality, we have that
    \begin{align}
        \norm{\paramhat_i^{\tau} - \paramstar_j} + \frac \Delsep 4 &\leq \norm{\paramhat_i^{\tau + 1} - \paramstar_{j'}} \\
        \norm{\paramhat_i^{\tau + 1} - \paramstar_{j'}} + \frac \Delsep 4 &\leq \norm{\paramhat_i^{\tau } - \paramstar_{j}}
    \end{align}
    Rearranging and again applying \Cref{ass:gap} tells us that
    \begin{equation}
        \frac \Delsep 2 < \norm{\paramhat_i^{\tau} - \paramhat_i^{\tau + 1}}.
    \end{equation}
    Applying the second and third statements of \Cref{lem:largeclustersremain}, however, brings a contradiction.  Thus we have, on the high probability event from \Cref{lem:largeclustersremain},
    \begin{equation}
        \I\left[\pi_\tau (\gtil_{\tau(t)}(\barbx_t)) \neq \pi_{\tau+ 1}(\gtil_{\tau(t)}(\barbx))\right] \leq \I\left[I_{i;\tau}(\ghat_{\tau(t)}) \leq 2 K \Xi \right].
    \end{equation}
    Thus we have
    \begin{align}
        \sum_{t = 1}^T \I\left[\pi_\tau (\gtil_{\tau(t)}(\barbx_t)) \neq \pi_{\tau+ 1}(\gtil_{\tau(t)}(\barbx))\right] &\leq \sum_{i = 1}^K \sum_{t = 1}^T \I\left[\gtil_{\tau(t)}(\barbx_t) = i\right] \I\left[I_{i;\tau}(\ghat_{\tau(t)}) \leq 2 K \Xi\right] \\
        &\leq K \left(E + 2 K \Xi\right).
    \end{align}
    The result follows.
\end{proof}
We may also apply \Cref{lem:largeclustersremain} to show that $\pi_{\tau + 1}(\ghat_{\tau  +1}(\barbx_t)) = \gstar(\barbx_t)$ most of the time.  We formalize this statement in the following result:
\begin{lemma}[Wrong Labels Bound] \label{lem:wronglabelsbound}
    With probability at least $1 - \delta$, it holds that
    \begin{align}
        \sum_{t = 1}^T \I\left[\pi_{\tau  +1}(\ghat_{\tau(t)  +1}(\barbx_t)) \neq \gstar(\barbx_t)\right] \leq K E + 2K^2 \Xi
    \end{align}
\end{lemma}
\begin{proof}
    By \Cref{lem:largeclustersremain}, it holds that
    \begin{align}
        \I\left[\pi_{\tau  +1}(\ghat_{\tau(t)  +1}(\barbx_t)) \neq \gstar(\barbx_t)\right] \leq \sum_{i = 1}^K \I\left[\ghat_{\tau(t) + 1}(\barbx_t) = i\right] \I\left[I_{i;\tau}(\ghat_{\tau(t) + 1}) \leq 2 K \Xi \right].
    \end{align}
    Thus we may argue exactly as in \Cref{lem:taucontinuity} to conclude the proof.
\end{proof}


\section{Proving Theorem \ref{thm:informal}}
In this section we formally state and prove our main result, i.e., that \Cref{alg:master} is an oracle-efficient algorithm for achieving expected regret polynomial in all the parameters and scaling as $T^{1 - \Omega(1)}$ as the horizon tends to infinity.  We have the following formal statement, which we first estalbish under the assumption of a strict separation between parameters (\Cref{ass:gap}).
\begin{theorem}\label{thm:gap}
    Suppose that Assumptions \ref{ass:dirsmooth}-\ref{ass:gap} hold.  Then, for the correct choices of $\gamma, \eta$, and $E$, given in the proof below (found above \eqref{eq:polydependence}), \Cref{alg:master} experiences expected regret:
    \begin{align*}
        \ee\left[\reg_T\right] \leq C B^{10} R^2 m d \nu^2 K^4 \sigdir^{- 4} \Delsep^{-1} T^{\frac{35}{36}} \left(\epsorac + \log(TRBmdK)\right) + C B R K^2 \sum_{t = 1}^T \norm{\bdelta_t}.
    \end{align*}
    Note that the last term is $O\left( B R K^2 T \epscor \right)$ in the worst case.
\end{theorem}
\begin{proof}
    We recall that \Cref{alg:master} proceeds in epochs of length $E$, a parameter to be specified.  At the beginning of epoch $\tau$, at time $(\tau - 1)E + 1$, $\ermoracle$ is called, producing $\paramhat_{\tau,i}$ for $1 \leq i \leq K$ and $\ghat_\tau$, for a total of $\frac{T}{E} \leq T$ calls to $\ermoracle$.  We then use \Cref{alg:ogdupdate} to modify $\ghat_\tau$ to $\gtil_\tau$ and predict $\hat\by_t = \paramhat_{\tau, \gtil_\tau(\barbx_t)} \barbx_t$.  Thus we have
    \begin{align*}
        \reg_T &= \sum_{t= 1}^T \norm{\hat\by_t - \paramstar_{\gstar(\barbx_t)} \barbx_t}^2 \\
        &= \sum_{\tau = 1}^{T / E} \sum_{t = (\tau - 1) E +1}^{\tau E} \norm{\left(\paramhat_{\tau, \gtil_{\tau}(\barbx_t)} - \paramstar_{\gstar(\barbx_t)}\right) \barbx_t}^2 \\
        &= \left(\sum_{\tau = 1}^{T / E} \sum_{t = (\tau - 1) E +1}^{\tau E} \norm{\left(\paramhat_{\tau, \gtil_{\tau}(\barbx_t)} - \paramstar_{\gstar(\barbx_t)}\right) \barbx_t}^2 \I\left[\pi_\tau\left(\gtil_{\tau}(\barbx_t)\right) \neq \gstar(\barbx_t)\right]   \right) \\
        &+ \left(\sum_{\tau = 1}^{T / E} \sum_{t = (\tau - 1) E +1}^{\tau E} \norm{\left(\paramhat_{\tau, \gtil_{\tau}(\barbx_t)} - \paramstar_{\pi_{\tau}\left(\gtil_\tau(\barbx_t)\right)}\right) \barbx_t}^2 \right).
    \end{align*}
    For the first term, we have
    \begin{align*}
        \sum_{\tau = 1}^{T / E} \sum_{t = (\tau - 1) E +1}^{\tau E} \norm{\left(\paramhat_{\tau, \gtil_{\tau}(\barbx_t)} - \paramstar_{\gstar(\barbx_t)}\right) \barbx_t}^2 \I\left[\pi_\tau\left(\gtil_{\tau}(\barbx_t)\right) \neq \gstar(\barbx_t)\right] &\leq \sum_{\tau = 1}^{T / E} \sum_{t = (\tau - 1) E +1}^{\tau E} 4 B^2 R^2 \I\left[\pi_\tau\left(\gtil_{\tau}(\barbx_t)\right) \neq \gstar(\barbx_t)\right] \\
        &\leq 4 B^2 R^2 \sum_{t = 1}^T \I\left[\pi_\tau\left(\gtil_{\tau}(\barbx_t)\right) \neq \gstar(\barbx_t)\right]
    \end{align*}
    and thus, applying \Cref{thm:onlinelearning}, it holds that the above expression is bounded by
    \begin{align}\label{eq:mainprooffirstterm}
        4 B^2 R^2 \left(   \frac{2 B E T \eta}{\gamma} + 3\left(1 + \frac 1\gamma\right) (K E + 2 K^2 \Xi) + \frac{4 K}{\eta} + \frac{\eta T}{\gamma^2} + \frac{T \gamma}{\sigdir} + \sqrt{T \log\left(\frac 1\delta\right)} \right).
    \end{align}
    For the second term, we have 
    \begin{align*}
        \sum_{\tau = 1}^{T / E} \sum_{t = (\tau - 1) E +1}^{\tau E} \norm{\left(\paramhat_{\tau, \gtil_{\tau}(\barbx_t)} - \paramstar_{\pi_{\tau}\left(\gtil_\tau(\barbx_t)\right)}\right) \barbx_t}^2 &\leq B^2 \sum_{\tau = 1}^{T / E} \sum_{t = (\tau - 1) E +1}^{\tau E} \fronorm{\paramhat_{\tau, \gtil_{\tau}(\barbx_t)} - \paramstar_{\pi_{\tau}\left(\gtil_\tau(\barbx_t)\right)}}^2 \\
        &= B^2 \sum_{i = 1}^K \sum_{\tau = 1}^{T / E} \sum_{t = (\tau - 1) E +1}^{\tau E} \fronorm{\paramhat_{\tau, i} - \paramstar_{\pi_{\tau}\left(i\right)}}^2 \I\left[\gtil_\tau(\barbx_t) = i\right].
    \end{align*}
    Now, let
    \begin{align}
        \Xi_1 &= C K^2 T \xi + C \frac{B^8 K d}{\sigdir^8 \xi^8} \log\left(\frac{B K T}{\sigdir \xi \delta}\right)\\
        \Xi_2' & = C \frac{B^2}{\sigdir^2 \xi^2 \alpha } \left(\epsorac + 1 + K^3 B^2 R^2 d^2 m \nu^2 \sqrt{T} \log\left(\frac{T R B m d K}{\delta}\right)\right) + \sum_{t = 1}^T \norm{\bdelta_t} \\
        \Xi' &= \max(\Xi_1, \Xi_2) \label{eq:mainproofxi}.
    \end{align}
    Note that as long as $\alpha < \Delsep^2$ then \Cref{lem:largeclustersremain} tells us that if we run \Cref{alg:master} with $A \geq 2 K \Xi'$, then with probability at least $1 - \delta$ it holds for all $1 \leq i \leq K$ and all $\tau \leq \tau'$ that if
    \begin{align*}
        I_i(\ghat_\tau) > 2 K \Xi'
    \end{align*}
    then
    \begin{align*}
        \abs{I_{i\pi_\tau'(i)}(\ghat_{\tau'})} > \Xi'
    \end{align*}
    and
    \begin{align*}
        \fronorm{\paramhat_i^{\tau'} - \paramstar_{\pi_{\tau'(i)}}}^2 \leq \alpha.
    \end{align*}
    Furthermore, under this condition, \Cref{lem:largeclustersremain} tells us that
    \begin{align*}
        I_i(\ghat_{\tau'}) > 2 K \Xi.
    \end{align*}
    Fixing an $i$, then, we see that
    \begin{align*}
        \sum_{\tau = 1}^{T / E} \sum_{t = (\tau - 1) E +1}^{\tau E} \fronorm{\paramhat_{\tau, i} - \paramstar_{\pi_{\tau}\left(i\right)}}^2 \I\left[\gtil_\tau(\barbx_t) = i\right] &= \sum_{\tau = 1}^{T / E} \sum_{t = (\tau - 1) E +1}^{\tau E} \fronorm{\paramhat_{\tau, i} - \paramstar_{\pi_{\tau}\left(i\right)}}^2 \I\left[\gtil_\tau(\barbx_t) = i \wedge I_i(\ghat_\tau) \leq 2 K \Xi' \right]  \\
        &+ \sum_{\tau = 1}^{T / E} \sum_{t = (\tau - 1) E +1}^{\tau E} \fronorm{\paramhat_{\tau, i} - \paramstar_{\pi_{\tau}\left(i\right)}}^2 \I\left[\gtil_\tau(\barbx_t) = i \wedge I_{i(\ghat_\tau)} > 2 K \Xi'\right].
    \end{align*}
    For the first term, we see that
    \begin{align*}
        \sum_{\tau = 1}^{T / E} \sum_{t = (\tau - 1) E +1}^{\tau E} \fronorm{\paramhat_{\tau, i} - \paramstar_{\pi_{\tau}\left(i\right)}}^2 \I\left[\gtil_\tau(\barbx_t) = i \wedge I_i(\ghat_\tau) \leq 2 K \Xi' \right] &\leq  4 R^2\sum_{\tau = 1}^{T / E} \sum_{t = (\tau - 1) E +1}^{\tau E}  \I\left[\gtil_\tau(\barbx_t) = i \wedge I_i(\ghat_\tau) \leq 2 K \Xi' \right] \\
        &\leq 4 R^2 \sum_{\tau = 1}^{T / E} \sum_{t = (\tau - 1) E +1}^{\tau E} 2 K \Xi' \\
        &\leq \frac{4 R^2 T K \Xi'}{E}.
    \end{align*}
    For the second term, we have with high probability that
    \begin{align*}
        \sum_{\tau = 1}^{T / E} \sum_{t = (\tau - 1) E +1}^{\tau E} \fronorm{\paramhat_{\tau, i} - \paramstar_{\pi_{\tau}\left(i\right)}}^2 \I\left[\gtil_\tau(\barbx_t) = i \wedge I_{i(\ghat_\tau)} > 2 K \Xi'\right] &\leq \sum_{\tau = 1}^{T / E} \sum_{t = (\tau - 1) E +1}^{\tau E} \alpha \\
        &\leq \alpha T.
    \end{align*}
    Summing over all $K$, we see that
    \begin{align}\label{eq:mainproofsecondterm}
        \sum_{i = 1}^K \sum_{\tau = 1}^{T / E} \sum_{t = (\tau - 1) E +1}^{\tau E} \fronorm{\paramhat_{\tau, i} - \paramstar_{\pi_{\tau}\left(i\right)}}^2 \I\left[\gtil_\tau(\barbx_t) = i\right] &\leq B^2 K \left(4 R^2\frac{T}{E} K \Xi' + \alpha T \right).
    \end{align}
    Thus, combining \eqref{eq:mainprooffirstterm} and \eqref{eq:mainproofsecondterm} tells us that with probability at least $1 - \delta$ it holds that $\reg_T$ is upper bounded by
    \begin{align*}
        4 &B^2 R^2 \left(   \frac{2 B E T \eta}{\gamma} + 3\left(1 + \frac 1\gamma\right) (K E + 2 K^2 \Xi) + \frac{4 K}{\eta} + \frac{\eta T}{\gamma^2} + \frac{T \gamma}{\sigdir} + \sqrt{T \log\left(\frac 1\delta\right)} \right) \\
        &+ B^2 K \left(4 R^2\frac{T}{E} K \Xi' + \alpha T \right)
    \end{align*}
    where $\Xi'$ is given in \eqref{eq:mainproofxi} and $\Xi$ is given in \Cref{lem:largeclustersremain}.  Now, setting
    \begin{align}
        E = T^{\frac{17}{18}}, && \gamma = T^{- \frac 1{36}}, && \xi = T^{- \frac{1}{9}}, && \alpha = T^{- \frac{1}{6}}, && \eta = T^{- \frac{19}{36}}, \label{eq:parameters}
    \end{align}
    gives, with probability at least $1 - \delta$,
    \begin{align}
        \reg_T \leq C B^{10} R^2 m d \nu^2 K^4 \sigdir^{- 4} \Delsep^{-1} T^{\frac{35}{36}} \left(\epsorac + \log(TRBmdK)\right) + C B R K^2 \sum_{t=  1}^T \norm{\bdelta_t}. \label{eq:polydependence}
    \end{align}
    This concludes the proof.
\end{proof}
\section{Application to PWA Dynamical Systems}\label{app:pwa_sys}
In this section, we apply our main results to prediction in PWA systems with user-provided controls.

\iftoggle{icml}{
\subsection{Regret for One-Step Prediction in PWA Systems}\label{app:lem:pwa_sys_one_step}
A direct application of our main result is online prediction in piecewise affine systems. Consider the following dynamical system with state $\bz_t \in \R^{d_z}$ and input $\bu_t \in \R^{d_u}$:
\begin{equation}
    \bz_{t+1} = \bA^\star_{i_t}  \bz_t + \bB^\star_{i_t} \bu_t + \bmean^\star_{i_t} + \be_t, \quad i_t = \gst(\bz_t,\bu_t) \label{eq:pwa_system1}.
\end{equation}
Substitute $\paramstar_i := [\bA_i \mid \bB_i \mid \bmean_i]$ and define the concatenations $\bx_t = [\bz_t \mid \bu_t]$ and  $\barbx_t = [\bx_t \mid 1]$. The following lemma, proven in \Cref{app:lem:pwa_sys_smoothness}, gives sufficient conditions on the system noise $\be_t$ and structure of the control inputs $\bu_t$ under which $\bx_t$ is directionally smooth.
\begin{lemma}\label{lem:pwa_sys_smoothness} Let $\cF_t$ denote the filtration generated by $\bx_t = [\bz_t \mid \bu_t]$. Suppose that $\be_{t-1} \mid \cF_{t-1}$ is $\sigdir$ smooth, and that in addition, $\bu_{t} = \bar\bK_t \bx_t + \bar{\bu}_t + \bar{\be}_t$, where $\bar \bK_t$  and $\bar{\bu}_t$ are $\cF_{t-1}$-measurable\footnote{This permits, for example, that $\bar \bK_t$ is chosen based on the previous mode $i_{t-1}$, or any estimate thereof that does not use $\bx_t$.} , and $\bar{\be}_t \mid \cF_{t-1},\be_t$ is $\sigdir$-directionally smooth. Then, $\bx_t \mid \cF_{t-1}$ is $\sigdir/\sqrt{(1+\|\bK_t\|_{\op})^2 +1}$- directionally smooth.
\end{lemma}
Directionally smooth noise distributions, such as Gaussians, are common in the study of online control \citep{dean2017sample,simchowitz2020improper}, and the smoothing condition on the input can be achieved by adding fractionally small noise, as is common in many reinforcement learning domains, such as to compute gradients in policy learning \citep{sutton1999policy} or for Model Predictive Path Integral (MPPI) Control \citep{williams2015model}.

Throughout, we keep the notation for compactly representing our parameters by letting $\paramstar_{i} = [\bstA_i \mid \bstB_i \mid \bmean^\star_i]$, the estimate at time $t$ be $\paramhat_{t,i} = [\bhatA_{t,i} \mid \bhatB_{t,i} \mid \hat\bmean_{t,i}]$, and covariates $\bx_t = (\bz_t,\bu_t)$. We let $\filt_{t}$ denote the filtration generated by $\bx_{1:t}$, and note that $\be_t$ is $\filt_{t+1}$-measurable.

\begin{assumption}[Boundedness]\label{asm:pwa_bound} The covariates and parameters, as defined above, satisfy \Cref{ass:boundedness}. 
\end{assumption}
\begin{assumption}[SubGaussianity and Smoothness so as to satisfy \Cref{lem:pwa_sys_smoothness}]\label{ass:pwa_smooth} We assume that \iftoggle{icml}{$\be_t \mid \filt_t$ is $\sigdir$-directionally smooth and $\bu_{t} = \bar\bK_t \bx_t + \bar{\bu}_t + \bar{\be}_t$, where $\bar \bK_t$  and $\bar{\bu}_t$ are $\cF_{t-1}$-measurable and $\bar{\be}_t \mid \cF_{t-1},\be_t$ is $\sigdir$-directionally smooth.}{
\begin{itemize}
    \item $\be_t \mid \filt_t$ is $\sigdir$-directionally smooth
    \item $\bu_{t} = \bar\bK_t \bx_t + \bar{\bu}_t + \bar{\be}_t$, where $\bar \bK_t$  and $\bar{\bu}_t$ are $\cF_{t-1}$-measurable and $\bar{\be}_t \mid \cF_{t-1},\be_t$ is $\sigdir$-directionally smooth.
\end{itemize}
}
Further, we assume that $\be_t \mid \filt_t$ is $\nu^2$-subGaussian.
\end{assumption}
Under these assumptions, we can apply our main result, \Cref{thm:informal}, to bound the one-step prediction error in PWA systems.  In particular, we have the following:
\begin{theorem}[One-Step Regret in PWA Systems]\label{thm:onestepregret}
    Suppose that $\bz_t, \bu_t$ evolve as in \eqref{eq:pwa_system1} with the attendant notation defined therein.  Suppose further that Assumptions \ref{ass:gap}, \ref{asm:pwa_bound}, and \ref{ass:pwa_smooth} hold and that at each time $t$, the learner predicts $\bzhat_{t+1}$ with the aim of minimizing the cumulative square loss with respect to the correct $\bz_{t+1}$.  If the learner applies \Cref{alg:master} to this setting, then with probability at least $1-  \delta$, the learner experiences regret at most
    \iftoggle{icml}
    { $\sum_{t = 0}^{T-1} \norm{\bz_{t+1} - \bzhat_{t+1}}^2 \leq T \nu^2 + \mathsf{poly}\left(\maxparam, \max_{1 \leq t \leq T} \norm{\bK_t}_{\op}, \log(1/\delta)\right) \cdot T^{1 - \Omega(1)}$}
    {
    \begin{align*}
        \sum_{t = 0}^{T-1} \norm{\bz_{t+1} - \bzhat_{t+1}}^2 \leq T \nu^2 + \mathsf{poly}\left(d, \frac {\sqrt{1 + (1 + \max_{1 \leq t \leq T} \norm{\bK_t}_{\op})^2}}\sigdir, m, B, R, K, \nu, \log\left( \frac 1\delta \right)\right) \cdot T^{1 - \Omega(1)}
    \end{align*}
    }
    where the exact polynomial dependence is given in \Cref{thm:gap}.
\end{theorem}
\begin{proof}
    The result follows from applying Lemma \ref{lem:pwa_sys_smoothness} to demonstrate directional smoothness and using Assumptions \ref{ass:pwa_smooth} and \ref{asm:pwa_bound} and smoothness to apply Theorem \ref{thm:informal}.
\end{proof}
}{}

\subsection{Proof of Lemma \ref{lem:pwa_sys_smoothness}}\label{app:lem:pwa_sys_smoothness}

We observe that this result follows directly from \Cref{lem:concat_prop2}.  Indeed, note that by assumption, $\bz_t | \filt_{t-1}, \bu_t$ is $\sigdir$-directionally smooth by the assumption that $\be_{t-1} | \filt_{t-1}$ is thus.  Similarly, $\bu_t | \filt_{t-1}, \bz_t$ is $\sigdir$-directionally smooth by the assumption that $\bar{\be}_t | \filt_{t-1}, \be_t$ is smooth.  Therefore, the conclusion follows from \Cref{lem:concat_prop2}.

\subsection{Formal Simulation Regret Setup}\label{app:formal_sim_reg}
In this setting, let $\bar\filt_{t}$ denote the filtration generated by $(\bx_{s,h})_{1 \le s \le t, 1\le h \le H}$. Let $\filt_{t,h}$ denote the filtration generated by $\bar \filt_{t-1}$ and $(\bx_{t,h'})_{1\le h' \le h}$, with the convention $\filt_{t,0} = \bar \filt_{t-1}$. 
We require that our estimates $\paramhat_{t,i}$ be $\bar\filt_{t-1}$-measurable. Further, we assume that our planner returns open-loop stochastic policies.
\begin{assumption}[Smoothed, Open Loop Stochastic Policies] \label{ass:openloop}
	We assume that our policy $\pi_t$ is stochastic and does not depend on state. That is, we assume that for all all $h \in [H]$ and $t \in [T]$.  $\bu_{t,h},\dots,\bu_{t,H} \perp \bx_{t,1:h} \mid \filt_{t-1}$. We further assume that $\bu_{t,h} \mid \bu_{t,1:h-1},\filt_{t-1}$ is $\sigdir$-smooth.
\end{assumption}

\begin{assumption}[Noise Distribution]\label{ass:pwa_noise} We assume that $\|\be_{t,h}\| \le \Bnoise$ for all $t,h$. Further, we assume that there exists a zero-mean, $\sigdir$-directionally smooth distribution $\cD$ over $\R^{d_z}$  such that $\be_{t,h} \mid \filt_{t,h-1} \sim \cD$.\footnote{Results can be extended to $\cD_{t,h}$ changing with the time step as long as these distributions are known to the learner. } Furthermore, we also assume that $\bz_{t,1} \mid \bar \filt_{t-1}$ is $\sigdir$-smooth. 
\end{assumption}

\begin{definition}[Wasserstein Distance]
	Let $(\cX, \mathsf{d})$ be a Polish space with metric $\mathsf{d}(\cdot,\cdot)$, and let $\mu, \nu$ be two Borel measures on $\cX$.  Define the Wasserstein distance as
	\begin{align*}
		\cW_2^2(\mu, \nu) = \inf_{\Gamma} \ee_{(X, Y) \sim \Gamma}\left[ \mathsf{d}(X, Y)^2 \right]
	\end{align*}
	where the infimum is over all couplings $\Gamma$ such that the marginal of $X$ is distributed as $\mu$ and the marginal over $Y$ is distributed as $\nu$.
\end{definition}
We will show that with minor modifications, outlined below, \Cref{alg:master} can be leveraged to produce $H$-step simulated predictions with laws that are close in Wasserstein distance.  In particular, this allows us to control the error of the associated predictions in expected squared norm.

\subsection{Algorithm Modifications}\label{app:algmodifications}
\begin{algorithm}[!t]
    \begin{algorithmic}[1]
    \State{}\textbf{Initialize } Epoch length $E$, Classifiers $\bw^0 = (\mathbf{0}, \dots, \mathbf 0)$, margin parameter $\gamma > 0$, learning rate $\eta > 0$, noise distributions $\cD$
    \For{$\tau =1,2,\dots, T/E$}
    \State{} $\left\{(\paramhat_{\tau,i})_{1 \leq i \leq K}, \, \ghat_\tau \right\} \gets \ermoracle((\barbx_{1:\tau E},\by_{1: \tau E} ))$ \qquad\qquad\qquad\qquad\quad(\algcomment{$\selfdot\algfont{modelUpdate}$})
	\State{} $\paramtil_{\tau_i} \gets \algfont{Project}_{\cC_\bP}(\paramhat_{\tau_i})$\qquad\qquad\qquad\qquad\quad (\algcomment{See \eqref{eq:cone}})
    \State{} $\ghat_\tau \gets \reorder\left(\ghat_{\tau}, \left(\paramtil_{\tau, i}\right)_{1\leq i \leq K}, \left(\paramtil_{\tau - 1,i}\right)_{1 \leq i \leq K} \right)$ \qquad\qquad (\algcomment{see \Cref{alg:reorder}})
    \State{} $\gtil_\tau \gets \ogdupdate(\barbx_{(\tau-1)E : \tau E}, \ghat_\tau, \gamma, \eta) $ \quad\qquad\quad(\algcomment{see \Cref{alg:ogdupdate}})
    \For{$t = \tau E ,\dots, (\tau + 1) E - 1$}
    \State{}\textbf{Receive } $\barbx_t$ and set $\bxhat_{t,0} = \barbx_t$
	\For{$h = 1,2,\dots, H$}
		\State{}\textbf{Sample } $\be_{t,h}' \sim \cD$
		\State{}\textbf{Predict } $\bxhat_{t,h} = \paramtil_{\tau,\gtil_\tau(\bxhat_{t,h-1})} \cdot \bxhat_{t,h-1} + \be_{t,h}'$
	\EndFor
	\State{} \textbf{Receive } $\by_t$
    \EndFor
    \EndFor
    \end{algorithmic}
      \caption{Algorithm for achieving low Simulation Regret}
      \label{alg:simregret}
\end{algorithm}

In order to produce good simulation regret, we run a variant of \Cref{alg:master} with two changes.  First, we need to use \Cref{asm:lyap_matrix} in some way to ensure that our estimated dynamics are contractive.  Second, we need to iterate our predicitions in order to generate a trajectory of length $H$ for each time $t$.   To address the first problem, suppose that we are in Line 3 of \Cref{alg:master}, which applies $\ermoracle$ to the past data, resulting in $\left\{ (\paramhat_{\tau, i})_{1 \leq i \leq K}, \ghat_\tau \right\}$.  For each $i$, we modify $\paramhat_{\tau, i}$ to become $\paramtil_{\tau, i}$ by projecting $\paramhat_{\tau,i}$ onto the convex set
\begin{align*}
	\cC_\bP = \left\{\param | \param^\top \bP \param \preceq \bP  \right\}.
\end{align*}
Formally, we define the projection to be
\begin{align}\label{eq:cone}
	\algfont{Project}_{\cC_\bP}(\param) = \argmin_{\param' \in \cC_\bP} \fronorm{\param - \param'}.
\end{align}
Because $\cC_\bP$ is convex, this step can be accomplished efficiently.  As we shall observe in the sequel, this projection step never hurts our error guarantees due, again, to the convexity of $\cC_\bP$.

Moving on to the second modification, at each epoch $\tau$, we have a fixed set of parameters $\left\{ (\paramtil_{\tau, i})_{1 \leq i \leq K}, \gtil_\tau \right\}$ where the $\gtil_\tau$ are modified from $\ghat_\tau$ exactly as in \Cref{alg:master}.  Thus, at each time, we independently sample $\be_{t,h}' \sim \cD$, the noise distribution, for $1 \leq h \leq H$ and predict $\bxhat_{t,h} = \paramtil_{\tau, \gtil_\tau(\bxhat_{t,h-1})} \cdot \bxhat_{t,h-1} + \be_{t,h}'$.  The entire modified algorithm is given in \Cref{alg:simregret} and we prove in the sequel that we experience low simulation regret.
\subsection{Formal Guarantees for Simulation Regret}\label{sec:sim_reg_guar}

In this section, we prove that \Cref{alg:simregret} attains low simulation regret under our stated assumptions.  Noting that \Cref{alg:simregret} is oracle-efficient, this result provides the first efficient algorithm that provably attains low simulation regret in the PWA setting.  The main result is as follows:
\begin{theorem}[Simulation Regret Bound] \label{thm:simregret}
	Suppose that we are in the setting of \eqref{eq:pwa_system2} and that Assumptions \ref{ass:gap}, \ref{asm:pwa_bound}, \ref{ass:pwa_smooth}, \ref{asm:lyap_matrix}, \ref{ass:openloop}, and \ref{ass:pwa_noise} hold.  Suppose that we run \Cref{alg:simregret}.  Then, with probability at least $1 - \delta$, it holds that
	\begin{align*}
		\simreg_{T, H} \leq \frac{9 B^2 H^8}{\sigdir^2} \cdot \left( C B^{10} R^2 m d \nu^2 K^4 \sigdir^{- 4} \Delsep^{-1} T^{\frac{35}{36}} \left(\epsorac + \log(TRBmdK)\right) + C B R K^2 T \epscor\right).
	\end{align*}
\end{theorem}
Before we prove \Cref{thm:simregret}, we must state the following key lemma, which says that if both the true and estimated dynamics are stable, and we are able to predict $\bx_{t+1}$ with high accuracy, then our simulation regret $H$ steps into the future is also small.
\begin{lemma}\label{lem:simregret}
	Let $\cG: \rr^d \to [K]$ denote the class of multi-class affine classifiers and let $\fhat, \fst$ denote functions satisfying the following properties:
	\begin{enumerate}
		\item The function $\fst$ can be decomposed as $\fst(\bx) = \fst_{\gstar(\bx)}(\bx)$ and similarly for $\fhat$ it holds that $\fhat(\bx) = \fhat_{\ghat(\bx)}(\bx)$ for some $\gstar, \ghat \in \cG$.
		\item For all $i \in [K]$, it holds that both $\fhat_i$ and $\fst_i$ are contractions, i.e., $\norm{\fst_i(\bx) - \fst_i(\bx')} \leq \norm{\bx - \bx'}$.
	\end{enumerate}
	Suppose that for $0 \leq h \leq H - 1$, we let
	\begin{align*}
		\bx_{h+1} = \fst(\bx_h) + \be_h && \bxhat_{h+1} = \fhat(\bxhat_h) + \be_h'
	\end{align*}
	for $\be_h, \be_h'$ identically distributed satisfying \Cref{ass:pwa_noise}.  Suppose that $\bx_0$ is $\sigdir$-smooth and that, almost surely, $\bx_h, \bxhat_h$ have norms bounded uniformly by $B$.  If $\max_{1 \leq h < H} \norm{\fhat(\bx_h) - \fst(\bx_h)} \leq \epsilon$ then it holds that
	\begin{align*}
		\cW_2^2\left( \bx_H, \bxhat_H \right) \leq \frac{9 B^2 H^6}{\sigdir^2} \epsilon^2 \epsilon.
	\end{align*}
\end{lemma}
\begin{proof}
	We begin by introducing intermediate random variables $\bx_H^{(\ell)}$ for $1 \leq \ell \leq H$.  We let $\bx_h^{(\ell)} = \bx_h$ for $h \leq \ell$ and let
	\begin{align*}
		\bx_{h+1}^{(\ell)} = \fhat\left(\bx_h^{(\ell)}\right) + \be_t'
	\end{align*}
	otherwise.  Note that $\bx_H^{(1)} = \bxhat_H$ and that $\bx_H^{(H)} = \bx_H$.  By the triangle inequality applied to Wasserstein distance, we observe that
	\begin{align*}
		\cW_2(\bx_H, \bxhat_H) \leq \sum_{\ell = 1}^H \cW_2\left( \bx_H^{(\ell)}, \bx_H^{(\ell+1)} \right).
	\end{align*}
	Consider a coupling where $\be_h = \be_h'$ for all $h$.  Note that under this coupling, for any $h > \ell$, it holds that
	\begin{align*}
		\norm{\bx_h^{(\ell)} - \bx_h^{(\ell+1)}} &= \norm{\fhat(\bx_{h-1}^{(\ell)}) + \be_t - \fhat(\bx_{h-1}^{(\ell + 1)}) - \be_t' } \\
		&= \norm{\fhat(\bx_{h-1}^{(\ell)}) - \fhat(\bx_{h-1}^{(\ell + 1)}) } \\
		&\leq \norm{\bx_{h-1}^{(\ell)} - \bx_{h-1}^{(\ell+1)}} + B \cdot \I\left[ \ghat(\bx_{h-1}^{(\ell)}) \neq \ghat(\bx_{h-1}^{(\ell + 1)}) \right].
	\end{align*}
	Thus, in particular,
	\begin{align*}
		\norm{\bx_{H}^{(\ell)} -  \bx_H^{(\ell+1)}} &\leq \norm{\bx_{\ell+1}^{(\ell)} - \bx_{\ell+1}^{(\ell+1)}} +   B \sum_{h = \ell}^{H-1} \I\left[\ghat(\bx_{h-1}^{(\ell)}) \neq \ghat(\bx_{h-1}^{(\ell + 1)})  \right]  \\
		&\leq \norm{\fhat(\bx_\ell) - \fst(\bx_\ell)} + B  \sum_{h = \ell}^{H-1} \I\left[\ghat(\bx_{h-1}^{(\ell)}) \neq \ghat(\bx_{h-1}^{(\ell + 1)})  \right]\\
		&\leq \epsilon + B  \sum_{h = \ell}^{H-1} \I\left[\ghat(\bx_{h-1}^{(\ell)}) \neq \ghat(\bx_{h-1}^{(\ell + 1)})  \right]
	\end{align*}
	by the assumption that $\norm{\fhat(\bx_h) - \fst(\bx_h)} \leq \epsilon$.  We now note that a union bound implies:
	\begin{align*}
		\sum_{h = \ell}^{H-1} \I\left[\ghat(\bx_{h-1}^{(\ell)}) \neq \ghat(\bx_{h-1}^{(\ell + 1)})  \right] &= \sum_{h = \ell}^{H-1} \I\left[\ghat(\bx_{h-1}^{(\ell)}) \neq \ghat(\bx_{h-1}^{(\ell + 1)}) \text{ and } \ghat(\bx_{h'}^{(\ell)}) = \gstar(\bx_{h'}^{(\ell+1)}) \text{ for } h' \leq h  \right] \\
		&+ \sum_{h = \ell}^{H-1} \I\left[\ghat(\bx_{h-1}^{(\ell)}) \neq \ghat(\bx_{h-1}^{(\ell + 1)}) \text{ and there exists } h' \leq h \text{ such that }\ghat(\bx_{h'}^{(\ell)}) \neq \gstar(\bx_{h'}^{(\ell+1)})  \right] \\
		&\leq \sum_{h = \ell}^{H-1} \I\left[\ghat(\bx_{h-1}^{(\ell)}) \neq \ghat(\bx_{h-1}^{(\ell + 1)}) \text{ and } \ghat(\bx_{h'}^{(\ell)}) = \gstar(\bx_{h'}^{(\ell+1)}) \text{ for } h' \leq h  \right] \\ \\
		&+ \sum_{h = \ell}^{H-1} (H - 1 - \ell)\I\left[\ghat(\bx_{h-1}^{(\ell)}) \neq \ghat(\bx_{h-1}^{(\ell + 1)}) \text{ and } \ghat(\bx_{h'}^{(\ell)}) = \gstar(\bx_{h'}^{(\ell+1)}) \text{ for } h' \leq h  \right] \\
		&\leq H \sum_{h = \ell}^{H-1} \I\left[\ghat(\bx_{h-1}^{(\ell)}) \neq \ghat(\bx_{h-1}^{(\ell + 1)}) \text{ and } \ghat(\bx_{h'}^{(\ell)}) = \gstar(\bx_{h'}^{(\ell+1)}) \text{ for } h' \leq h  \right].
	\end{align*}
	We now observe that
	\begin{align*}
		\I\left[\ghat(\bx_{h-1}^{(\ell)}) \neq \ghat(\bx_{h-1}^{(\ell + 1)}) \text{ and } \ghat(\bx_{h'}^{(\ell)}) = \gstar(\bx_{h'}^{(\ell+1)}) \text{ for } h' \leq h  \right] &\leq \I\left[\ghat(\bx_{h-1}^{(\ell)}) \neq \ghat(\bx_{h-1}^{(\ell + 1)}) \text{ and } \norm{\bx_{h-1}^{(\ell)} - \bx_{h-1}^{(\ell+1)}} \leq \epsilon  \right]
	\end{align*}
	By the fact that $\cG$ is given by multi-class affine thresholds, we see that
	\begin{align*}
		\I\left[\ghat(\bx_{h-1}^{(\ell)}) \neq \ghat(\bx_{h-1}^{(\ell + 1)}) \text{ and } \norm{\bx_{h-1}^{(\ell)} - \bx_{h-1}^{(\ell+1)}} \leq \epsilon  \right] &\leq \sum_{1 \leq i,j \leq K} \I\left[ \abs{\inprod{\bw_{ij}}{\bx_{h-1}}} \leq 2 \epsilon \right]
	\end{align*}
	where the $\bw_{ij}$ are the unit vectors determining the decision boundaries induced by $\ghat$.  By smoothness, then, we see that
	\begin{align*}
		\cW(\bx_H^{(\ell)}, \bx_H^{(\ell + 1)}) &\leq \epsilon + B H\cdot \pp\left(  \sum_{h = \ell}^{H-1} \I\left[\ghat(\bx_{h-1}^{(\ell)}) \neq \ghat(\bx_{h-1}^{(\ell + 1)}) \text{ and } \ghat(\bx_{h'}^{(\ell)}) = \gstar(\bx_{h'}^{(\ell+1)}) \text{ for } h' \leq h  \right] > 1\right) \\
		&\leq  \epsilon + BH \sum_{h = 1}^H \pp\left( \ghat(\bx_{h-1}^{(\ell)}) \neq \ghat(\bx_{h-1}^{(\ell + 1)}) \text{ and } \ghat(\bx_{h'}^{(\ell)}) = \gstar(\bx_{h'}^{(\ell+1)}) \text{ for } h' \leq h  \right) \\
		&\leq \epsilon + BH \sum_{h = 1}^H \sum_{1 \leq i,j \leq K} \pp\left(   \abs{\inprod{\bw_{ij}}{\bx_{h}}} \leq 2 \epsilon\right) \\
		&\leq \left( 1 + \frac{2 B H^2}{\sigdir} \right) \epsilon.
	\end{align*}
	Plugging this into our bound at the beginning and noting that $1 \leq \frac{B H^2}{\sigdir}$ concludes the proof.
\end{proof}
We are now ready to prove \Cref{thm:simregret}.  In particular, we use \Cref{lem:simregret} to reduce a bound on simulation regret to one of the distance between $\fhat$ and $\fst$ evaluated on the data sequence.  We then apply \Cref{thm:parameterrecovery,thm:onlinelearning} to control this distance.  The details are below:
\begin{proof}[Proof of \Cref{thm:simregret}]
	We begin by applying \Cref{lem:simregret}.  Observe that in our setting, $\cG$ is the class of multi-class affine classifiers and $\fst, \fhat_t$ can be decomposed as required by the theorem, where we let $\tau$ denote the epoch in which $t$ is placed and let
	\begin{align*}
		\fhat_t(\bx) = \paramtil_{\tau, \gtil_\tau(\bx)} \cdot \bx.
	\end{align*}
	By the fact that $\paramtil_{\tau, i}, \paramstar_i \in \cC_\bP$, which follows from the construction of $\paramtil_{\tau,i}$ and \Cref{asm:lyap_matrix}, the contractivity assumption required by \Cref{lem:simregret} holds.  Furthermore, note that because $\cC_\bP$ is convex and $\paramstar_i \in \cC_\bP$, it holds that
	\begin{align}\label{eq:projectiondecreases}
		\fronorm{\paramtil_{\tau, i} - \paramstar_{\pi_{\tau(i)}}} \leq \fronorm{\paramhat_{\tau, i} - \paramstar_{\pi_{\tau(i)}}}.
	\end{align}
	Now, note that \Cref{alg:master} does not update the predicting functions within the epoch and thus for all $t$ such that $\left\{ t, t+1, \dots, t + H \right\}$ does not contain an integral multiple of $E$, it holds that our prediction function $\fhat_t$ is constant on the interval.  Let $\cI$ denote the set of times $t$ such that the previous condition fails, i.e.,
	\begin{align*}
		\cI = \left\{ t \leq T | \tau E - t > H \text{ for all } \tau \in \bbN \right\}.
	\end{align*}
	Then, we may apply \Cref{lem:simregret} to get
	\begin{align*}
		\simreg_{T, H} &= \sum_{t = 1}^T \sum_{h = 1}^H \cW_2^2(\bx_{t + h}, \bxhat_{t + h} | \filt_{t-1}) \\
		&= \sum_{t = 1}^T \I\left[t \in \cI  \right]\sum_{h = 1}^H \cW_2^2(\bx_{t + h}, \bxhat_{t + h} | \filt_{t-1}) + \sum_{t = 1}^T \I\left[t \not\in \cI  \right]\sum_{h = 1}^H \cW_2^2(\bx_{t + h}, \bxhat_{t + h} | \filt_{t-1}) \\
		&\leq \sum_{t = 1}^T\I\left[t \in \cI  \right] \sum_{h = 1}^H \frac{9 B^2 H^6}{\sigdir^2} \max_{1 \leq h \leq H} \norm{\fhat(\bx_{t+h}) - \fst(\bx_{t+h})}^2 + 4H B^2 R^2 \sum_{t = 1}^T \I\left[ t \notin \cI \right]\\
		&\leq \frac{9 B^2 H^7}{\sigdir^2} \sum_{t = 1}^T \I\left[t \in \cI  \right]\sum_{h  =1}^H \norm{\fhat(\bx_{t+h}) - \fst(\bx_{t+h})}^2 + 4 B^2 R^2 H \cdot \frac{H T}{E}\\
		&\leq \frac{9 B^2 H^8}{\sigdir^2} \sum_{t = 1}^T\I\left[t \in \cI  \right]\norm{\fhat(\bx_t) - \fst(\bx_t)}^2 + 4 B^2 R^2 H^2 \cdot \frac{T}{E} \\
		&\leq \frac{9 B^2 H^8}{\sigdir^2} \sum_{t = 1}^T \norm{\byhat_t - \fst(\bx_t)}^2 + 4 B^2 R^2 H^2 \cdot \frac{T}{E},
	\end{align*}
	where $\byhat_t$ is the prediction of \Cref{alg:master}.  We may now apply \Cref{thm:gap} to bound this last quantity by applying \eqref{eq:projectiondecreases}.  The result follows.
\end{proof}

\subsection{Extensions}\label{sec:sim_reg_ext}
For simplicity, we considered a setting where all linear dynamics were stable with a common Lyapunov function. Our results can be extended to the case where there are mode-dependent gains $(\bK_{i})_{i \in K}$, where $\bK_i \in \R^{d_z \times d_u}$, as well, because our proof demonstrates that as we achieve low-regret, our simulation accurately recovers the correct mode sequence. Thus, if we know gains which ensure mutual stability, we can apply these gains as well. More general closed-loop policies can also be accomodated, provided we maintain the requisite smoothness (as, for example, ensured by \Cref{lem:pwa_sys_smoothness}). Lastly, as our regret-guarantees ensure parameter recovery, we can envision settings where the gains are constructed, for example, by certainty-equivalent control synthesis, and analyze their suboptimality via perturbation bounds such as those in \cite{simchowitz2020naive}.

\end{document}